\documentclass{article}

\usepackage[preprint]{neurips_2026}

\usepackage[utf8]{inputenc} 
\usepackage[T1]{fontenc}    
\usepackage{hyperref}       
\usepackage{url}            
\usepackage{booktabs}       
\usepackage{amsfonts}       
\usepackage{nicefrac}       
\usepackage{microtype}      
\usepackage{xcolor}         
\usepackage{multirow}
\usepackage{enumitem}
\usepackage{microtype}
\usepackage{graphicx}
\usepackage{subcaption}
\usepackage{booktabs} 

\usepackage{hyperref}

\usepackage{amsmath}
\usepackage{amssymb}
\usepackage{mathtools}
\usepackage{amsthm}

\usepackage[capitalize,noabbrev]{cleveref}

\theoremstyle{plain}
\newtheorem{theorem}{Theorem}[section]
\newtheorem{proposition}[theorem]{Proposition}
\newtheorem{lemma}[theorem]{Lemma}
\newtheorem{corollary}[theorem]{Corollary}
\theoremstyle{definition}

\newtheorem{assumption}[theorem]{Assumption}
\theoremstyle{remark}
\newtheorem{remark}[theorem]{Remark}

\usepackage[disable,textsize=tiny]{todonotes}

\title{TopoGeoScore: A Self-Supervised Source-Only Geometric Framework for OOD Checkpoint Selection}

\author{%
	Farid Hazratian\thanks{Equal contribution.} \\
	University of Tehran \\
	Tehran, Iran \\
	\texttt{farid.hazratian@ut.ac.ir} \\
	\And
	Ali Zia\footnotemark[1] \\
	La Trobe University \\
	Melbourne, Australia \\
	\texttt{a.zia@latrobe.edu.au} \\
	\AND
	Hien Duy Nguyen\footnotemark[1] \\
	La Trobe University, Australia \quad Kyushu University, Japan \\
	\texttt{H.Nguyen5@latrobe.edu.au}
}
\begin{document}

\maketitle

\begin{abstract}

Out-of-distribution (OOD) robustness is difficult to diagnose when target-domain labels are unavailable. We consider a more restrictive source-only variant of unsupervised accuracy estimation: selecting robust checkpoints using only source-domain representations, with no target samples or target labels. We propose \textbf{TopoGeoScore}, a source-only geometric scorer for label-free OOD checkpoint selection. Given a trained checkpoint, we construct class-conditional mutual $k$-nearest-neighbour graphs from source embeddings and extract three interpretable signals: a torsion-inspired reduced Laplacian log-determinant for global class-manifold complexity, Ollivier--Ricci curvature for local neighbourhood regularity, and higher-order topological summaries for fragmented connectivity, loops, and global--local inconsistency. Instead of fixing their weights by hand, TopoGeoScore learns a non-negative linear score through a self-supervised objective that enforces invariance under approximately geometry-preserving embedding views and separation from structure-breaking views. The score remains interpretable and uses no target-domain samples or labels. Results across CIFAR-based corruption and distribution-shift benchmarks, ImageNet-C, MNLI$\to$HANS transfer, and OGBN-Arxiv suggest that source representations contain measurable global--local--topological evidence of robustness, supporting practical checkpoint selection before deployment under distribution shift.

\end{abstract}


\section{Introduction}
\label{sec:intro}
Deep networks can achieve high in-distribution (ID) validation accuracy while remaining unreliable under distribution shift. This creates a practical model-selection problem, i.e. among several checkpoints, architectures, or training runs with similar ID performance, which one should be deployed when the target distribution is shifted, and target labels are unavailable? Prior work on domain generalisation and robust optimisation has shown that models with comparable ID accuracy can differ substantially in out-of-distribution (OOD) performance, depending on training objectives, optimisation dynamics, and checkpoint choice \cite{arjovsky2019irm,rame2022fishr,cha2021swad}. In many deployment settings, however, neither target labels nor target samples are available at selection time. Robust checkpoint selection is therefore often reduced to heuristics based on ID validation accuracy, training epoch, or manually chosen proxy scores.

A closely related line of work studies unsupervised accuracy estimation under distribution shift, where the goal is to predict target-domain accuracy without target labels. Existing methods commonly rely on unlabeled target samples through confidence thresholds, projection norms, gradient statistics, matrix-norm logit measures, or confidence/dispersity signals \cite{garg2022leveraging,yu2022predicting,deng2024leveraging,xie2024mano,deng2025confidence,deng2021labels}. These methods are valuable when target inputs are available, but they do not address a stricter source-only regime that is selecting a robust checkpoint before observing the shifted distribution at all. We study this more restrictive setting. Given only a checkpoint family and labelled source validation data, our goal is to rank checkpoints by expected OOD robustness without using target samples or target labels.

This source-only setting raises a natural question, i.e. what signal, if any, remains in the source representation that can indicate future robustness under shift? A growing body of evidence suggests that representation geometry and topology are related to generalisation and robustness \cite{rieck2018neural,guss2018characterizing,wang2020understanding,barannikov2022representation}. Standard post-hoc diagnostics, however, often reduce representations to low-order or global summaries, such as feature norms, covariance spectra, similarity measures such as CKA, or sharpness-related quantities \cite{kornblith2019similarity,jiang2019fantastic,foret2020sharpness}. These summaries are useful, but they do not explicitly model class-conditional neighbourhood structure, global connectivity, or higher-order topological defects. As a result, they can miss cases where two checkpoints look similar under coarse statistics but differ in whether their class manifolds are coherent, locally regular, or fragmented.

We argue that source-only robustness diagnosis requires a more structured view of representation geometry. Instead of treating an embedding cloud as an unstructured set of vectors, we lift source embeddings into class-conditional mutual $k$-nearest-neighbour graphs. This graph view exposes three complementary aspects of representation quality. First, a torsion-inspired reduced Laplacian log-determinant summarises global class-manifold complexity, with lower values indicating more coherent and less redundant connectivity. This connects to classical spectral graph theory and the Matrix--Tree theorem \cite{chung1997spectral,vonluxburg2007tutorial,belkin2003laplacian,raysinger1971torsion,lyons2005determinantal}. Second, Ollivier--Ricci curvature measures local neighbourhood regularity by quantifying whether nearby random-walk distributions contract or expand along graph edges \cite{ollivier2007ricci,nguyen2023revisiting}. Third, higher-order topology, through Hodge $L_1$ descriptors and persistent-homology summaries, captures loops, fragmentation, and circulation that cannot be represented by pairwise statistics alone \cite{schaub2021signal,roddenberry2022signal,bubenik2015persistence,adams2017persistence,zia2024topological}. In our formulation, topology is not simply another descriptor. It acts as a mediator that explains when global complexity and local regularity disagree.

Based on this view, we propose \textsc{TopoGeoScore}, a source-only geometric scorer for label-free OOD checkpoint selection. For each checkpoint, we construct class-conditional mutual $k$-NN graphs from source embeddings and compute a compact feature vector containing normalised global complexity, signed local regularity, a topology-defect score, a global--local disagreement term, and a topology-modulated disagreement interaction. A non-negative linear score then ranks checkpoints, where a lower score indicates a more robust representation. The non-negativity constraint keeps the score interpretable: increasing global complexity, topological defect, or global--local disagreement cannot improve the robustness score.

A fixed weighted combination of these signals is possible, but it is brittle because the useful weighting can vary across checkpoint families, architectures, and datasets. We therefore learn the score weights using a source-only self-supervised objective. The key inductive bias is simple: a meaningful robustness score should remain stable under geometry-preserving transformations of the source embedding cloud, and it should worsen under transformations that break class-conditional structure. We instantiate this using positive views such as orthogonal rotation, mild Gaussian noise, bootstrap subsampling, feature dropout, and PCA-preserving projection, together with negative views such as label shuffle, feature shuffle, and degree-preserving graph rewiring. The resulting objective learns the non-negative score weights through invariance and separation losses, without using target-domain samples or target-domain labels. We use ``self-supervised'' in this restricted sense: the score is learned from source-conditioned graph views, while source labels are used only to construct class-conditional graphs. Figure~\ref{fig:pipeline} shows the architecture, and our contributions are as follows:

\begin{itemize}
    \item \textbf{A source-only formulation of OOD checkpoint selection:} We study the stricter setting where robust checkpoints must be selected using only labelled source validation data, with no target samples and no target labels.

    \item \textbf{A global--local--topological view of representation robustness:} We formulate source representation diagnosis through class-conditional mutual $k$-NN graphs, where torsion-inspired spectral complexity captures global structure, Ollivier--Ricci curvature captures local regularity, and higher-order topology mediates global--local disagreement.

    \item \textbf{A self-supervised geometric scorer:} We introduce \textsc{TopoGeoScore}, an interpretable non-negative linear score trained from source-conditioned graph views by enforcing invariance to geometry-preserving transformations and separation from structure-breaking transformations.

\end{itemize}

Across CIFAR-based distribution shifts, ImageNet-C architecture variants, MNLI$\rightarrow$HANS transfer, and OGBN-Arxiv temporal shift, we show that the proposed geometric score tracks OOD robustness and supports near-oracle checkpoint selection in the main evaluated settings.
Control experiments further show that the score is stable under geometry-preserving views and degrades under structure-breaking views, supporting the claim that it captures task-aligned source representation structure rather than superficial statistics.

\begin{figure*}[t]
\centering
\includegraphics[width=1.0\linewidth]{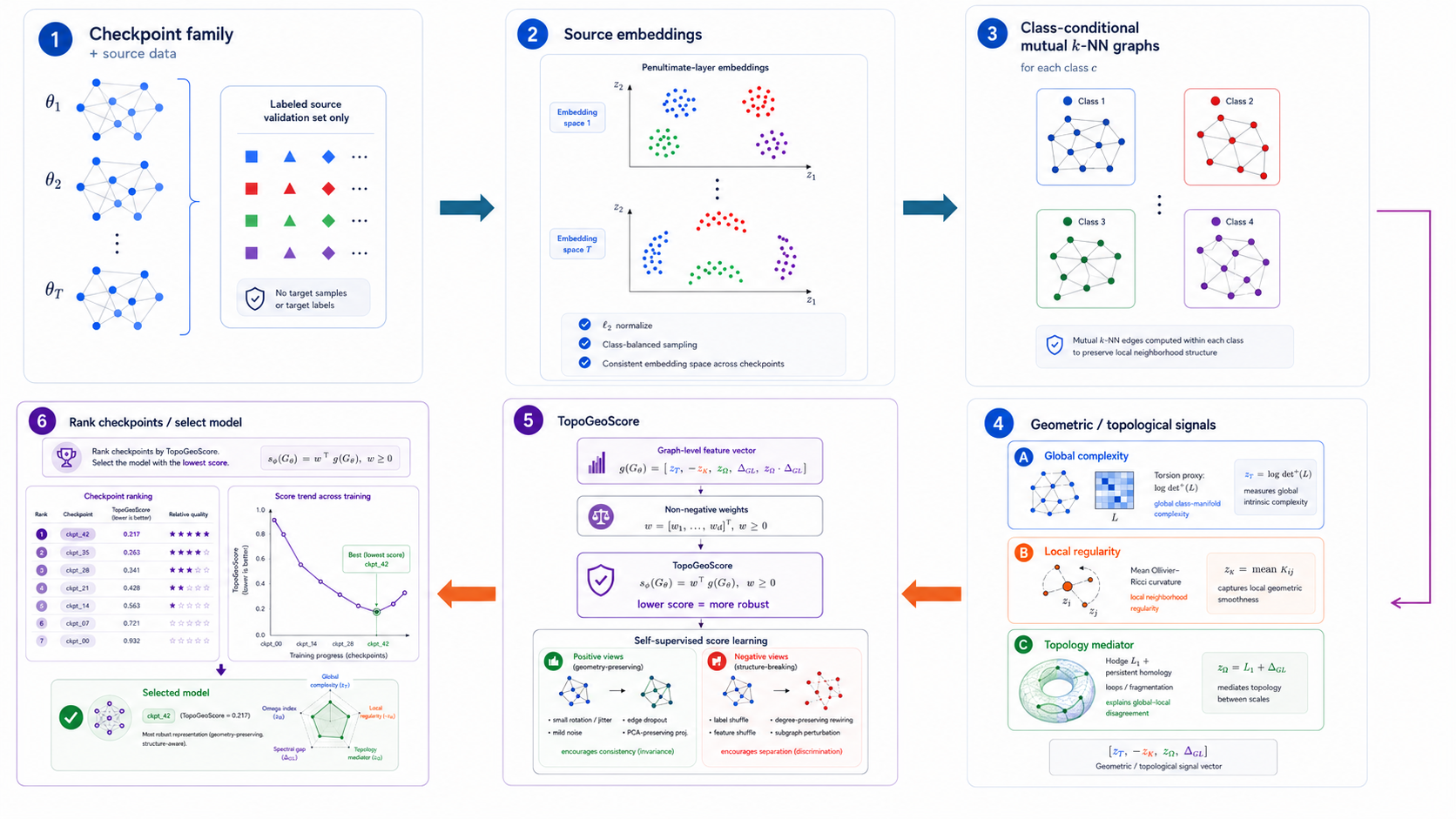}
\caption{
\textsc{TopoGeoScore} ranks checkpoints using only labelled source data.
For each checkpoint, source embeddings are converted into class-conditional mutual $k$-NN graphs, from which global complexity, local regularity, and higher-order topological signals are extracted.
Non-negative feature weights are learned through a source-only invariance--separation objective, yielding a lower-is-better score for expected OOD robustness.
}

\label{fig:pipeline}
\end{figure*}

\section{Related Work}
\label{sec:related}

\paragraph{OOD robustness and source-only checkpoint selection:}
Robustness under distribution shift is commonly addressed at training time through invariant representation learning, domain generalisation, and optimisation strategies that seek flat or stable solutions \cite{arjovsky2019irm,ahuja2020invariant,krueger2021out,cha2021swad,rame2022fishr,izmailov2018averaging}. These methods aim to produce robust models, often requiring multiple environments, modified objectives, or access to specific training conditions. A complementary line of work studies unsupervised accuracy estimation, where target-domain performance is predicted without target labels using confidence thresholds, projection norms, gradient statistics, matrix-norm logit measures, confidence/dispersity signals, or agreement-based criteria \cite{garg2022leveraging,yu2022predicting,deng2024leveraging,xie2024mano,deng2025confidence,deng2021labels,guillory2021predicting,baek2022agreementline}. However, these approaches typically assume access to unlabeled target samples at evaluation time. Our setting is stricter: the practitioner must rank checkpoints using only labelled source validation data before observing the shifted distribution. We therefore study source-only, target-sample-free OOD checkpoint selection rather than training-time robustness optimisation or target-data-based accuracy estimation.

\paragraph{Representation diagnostics and geometric model selection:}
A large body of work analyzes learned representations using similarity measures and global statistics, including CKA and related kernel-based comparisons, covariance spectra, feature norms, gradient-based measures, sharpness/flatness proxies, PAC--Bayes bounds, and margin- or norm-based generalization criteria \cite{kornblith2019similarity,nguyen2021do,jiang2019fantastic,foret2020sharpness,dziugaite2017computing,bartlett2017spectrally}. These diagnostics provide useful global summaries, but they do not explicitly model class-conditional neighbourhood structure, global connectivity, or higher-order topological defects. Our work is closer in spirit to model-selection criteria, but instead of using a hand-designed scalar proxy, we construct class-conditional representation graphs and learn an interpretable non-negative score from source-only graph views.

\paragraph{Spectral, topological, and higher-order analysis of representations:}
Spectral and topological methods provide tools for probing the structure of data and learned representations, including Laplacian descriptors, heat kernels, persistent homology, and representation topology divergence \cite{chung1997spectral,vonluxburg2007tutorial,belkin2003laplacian,smola2003kernels,rieck2018neural,guss2018characterizing,barannikov2022representation}. Persistent homology has been used to study generalization through intrinsic dimension, capacity, persistence evolution during training, and persistence-summary regression of the generalization gap \cite{birdal2021intrinsic,gutierrezfandino2021persistent,ballester2022predicting}. Practical vectorizations include persistence landscapes, persistence images, Betti curves, and persistent entropy \cite{bubenik2015persistence,adams2017persistence,rucco2017characterisation}. Beyond pairwise graph structure, Hodge Laplacians and simplicial methods capture circulation, harmonic structure, and higher-order connectivity \cite{schaub2021signal,roddenberry2022signal,ebli2020simplicial,bodnar2021weisfeiler,huang2024hlhgat,zhou2024facilitating,keros2022dist2cycle}. Discrete Ollivier--Ricci curvature has also been used to analyze contraction, bottlenecks, over-squashing, and over-smoothing in graphs \cite{ollivier2007ricci,topping2022oversquashing,nguyen2023revisiting}. Closely related to our torsion proxy, \citet{shen2025torgnn} use analytic torsion as a local edge-level message-passing weight. In contrast, we use a global reduced log-determinant of class-conditional $k$-NN graph Laplacians as a post-hoc source-representation diagnostic, and combine it with local curvature and higher-order topology for OOD checkpoint ranking.

\paragraph{Self-supervised score learning:}
Self-supervised learning often relies on the principle that useful representations should be stable under label-preserving views and discriminative against incompatible views. Contrastive and non-contrastive methods instantiate this principle through positive-view alignment, negative separation, stop-gradient prediction, or redundancy reduction \cite{chen2020simple,he2020momentum,grill2020byol,chen2021exploring,zbontari2021barlow,bardes2022vicreg}, with theoretical accounts connecting augmentation invariance to representation geometry \cite{wang2020understanding,haochen2021provable}. \textsc{TopoGeoScore} reuses this inductive bias at a different level. Rather than learning a representation, we learn a linear scoring function on fixed geometric and topological graph features. Positive views preserve the source embedding geometry, while negative views break class-conditional structure. This gives a target-sample-free objective for learning checkpoint-ranking weights while preserving interpretability through non-negative parameters.

\paragraph{Positioning:}
To our knowledge, prior work has not formulated source-only OOD checkpoint selection as a global--local--topological consistency problem. \textsc{TopoGeoScore} connects three signals that are usually studied separately: torsion-inspired spectral complexity for global class-manifold structure, Ollivier--Ricci curvature for local regularity, and higher-order topology for loops, fragmentation, and global--local disagreement. The resulting score is learned from source-conditioned graph views and ranks checkpoints without using target samples or target labels.

\section{TopoGeoScore: A Global--Local--Topological Consistency Framework}
\label{sec:method}

\subsection{Problem Setup: Source-Only OOD Checkpoint Selection}
Given a collection of trained model checkpoints $\{\theta\}$ (all trained on a source-domain dataset), our goal is to rank these checkpoints by expected out-of-distribution (OOD) accuracy without using any OOD labels.  In other words, we seek a \emph{source-only} selection score $s(\theta)$ that predicts OOD robustness based solely on the source data and model parameters.  Formally, let the source dataset be $\{(x_i,y_i)\}_{i=1}^N$ and let $f_\theta(\cdot)$ be the penultimate-layer embedding of model $\theta$.  We extract the normalized embeddings $z_i = f_\theta(x_i)/\|f_\theta(x_i)\|_2$ and construct class-conditional graphs (Sec.~\ref{sec:graphs}) from these.  The methodology below describes how to compute geometric and topological summaries of these graphs and combine them into \textsc{TopoGeoScore} to rank checkpoints.

\subsection{Class-Conditional Representation Graphs}
\label{sec:graphs}
For each checkpoint $\theta$, we form a graph for each class $c$.  Let $V_c=\{i:y_i=c\}$ be the indices of source samples in class $c$.  We build a \emph{mutual $k$-nearest-neighbor graph} $G_c=(V_c,E_c,W_c)$ as follows: connect nodes $i,j\in V_c$ iff $i$ is among the $k$ nearest neighbors of $j$ \emph{and} $j$ is among the $k$ nearest neighbors of $i$, measured in embedding space.  We use self-tuning weights to mitigate density effects:
\begin{equation}
\label{eq:self-tuning}
w_{ij} = \exp\!\biggl(-\frac{\|z_i - z_j\|^2}{\sigma_i\,\sigma_j}\biggr), 
\qquad 
\sigma_i = \|z_i - z_{i}^{(k)}\|_2 .
\end{equation}
Here $z_{i}^{(k)}$ is the $k$-th nearest neighbor of $i$.  Let $W_c$ be the weighted adjacency, $D_c$ the degree matrix, and define the normalized Laplacian $\mathcal{L}_c = I - D_c^{-1/2} W_c D_c^{-1/2}$.  Denote its positive eigenvalues by $\{\lambda_{i}^{(c)}>0\}$.  All subsequent geometric and topological quantities are computed class-wise and then aggregated (typically by averaging over classes) to produce global features for checkpoint $\theta$.

\subsection{Geometric and Topological Signals}
\label{sec:signals}

For each class graph $G_c$, we compute three class-wise summaries and aggregate them over classes. 

\paragraph{Global complexity:}
The torsion proxy is
\[
\tau(G_c)=\log\det{}^{*}(\mathcal{L}_c)=\sum_{\lambda_i^{(c)}>0}\log\lambda_i^{(c)} ,
\]
where $\det{}^{*}$ omits zero eigenvalues. By the normalized Matrix--Tree identity (Proposition~\ref{prop:normalized-mtt}), this is a degree-corrected log spanning-tree partition function on each connected component; in the connected case,
\[
\log\det{}^{*}(\mathcal{L}_c)=\log\tau_{\mathrm{tree}}(G_c)+\log\mathrm{vol}(G_c)-\sum_{v\in V_c}\log d_v .
\]
We aggregate class values to obtain $\mathcal{T}(\theta)$, a global class-manifold complexity proxy.

\paragraph{Local regularity:}
For each edge $(u,v)\in E_c$, let $\mu_u,\mu_v$ be the lazy random-walk distributions at $u,v$. The Ollivier--Ricci curvature is
\[
\kappa(u,v)=1-\frac{W_1(\mu_u,\mu_v)}{d(u,v)},
\]
computed approximately via Sinkhorn transport \cite{cuturi2013sinkhorn}. We use the class-averaged curvature
\[
\bar\kappa(\theta)=\frac1C\sum_c\frac{1}{|E_c|}\sum_{(u,v)\in E_c}\kappa(u,v),
\]
with the signed feature $-\bar\kappa$ so that larger values correspond to less regular local geometry.

\paragraph{Topological defect:}
Let $K_c$ be the two-dimensional clique complex built on $G_c$, with boundary matrices $B_1,B_2$ and unweighted Hodge $1$-Laplacian $L_1=B_1^\top B_1+B_2B_2^\top$. Proposition~\ref{prop:hodge-cycles} gives $\ker L_1\cong H_1(K_c;\mathbb{R})$, so $\dim\ker L_1=\beta_1(K_c)$. We extract Hodge summaries such as $\log\det{}^{*}(L_1)$, spectral entropy, and $\beta_1(K_c)$, together with $0$- and $1$-dimensional Vietoris--Rips persistent-homology summaries. Their normalized aggregate is denoted $z_\Omega(\theta)$ and is used as a higher-order defect score, especially in regimes where global complexity and local regularity disagree.

\subsection{TopoGeoScore: Source-Only Scoring Function}

For a checkpoint $\theta$ with class-conditional graph family $G_\theta=\{G_c\}$, let $z_T$, $z_{\bar\kappa}$, and $z_\Omega$ be the $z$-normalised values of $\mathcal{T}(\theta)$, $\bar\kappa(\theta)$, and the topological defect score within the checkpoint family. Define
\[
\Delta_{\mathrm{GL}}(\theta)=|z_T(\theta)+z_{\bar\kappa}(\theta)|,\qquad
g(G_\theta)=[z_T,-z_{\bar\kappa},z_\Omega,\Delta_{\mathrm{GL}},z_\Omega\Delta_{\mathrm{GL}}]^\top ,
\]
and set
\[
s_\phi(G_\theta)=\mathbf{w}^\top g(G_\theta),\qquad \mathbf{w}\in\mathbb{R}_{\ge0}^{5}.
\]
Lower $s_\phi$ indicates a more robust checkpoint. The non-negativity constraint enforces coordinatewise monotonicity: increasing global complexity, topological defect, or global--local disagreement cannot improve the score, while increasing mean curvature decreases the signed local-irregularity feature $-z_{\bar\kappa}$.

\subsection{Self-Supervised Score Learning}
\label{sec:ssl}

The weight vector $\mathbf{w}$ is trained without OOD labels using source-conditioned graph views. Positive views $G_\theta^+\sim\mathcal{T}^+$ are produced by approximately geometry-preserving transformations such as orthogonal rotation, mild noise, bootstrap resampling, feature dropout, and PCA-preserving projection. Negative views $G_\theta^-\sim\mathcal{T}^-$ are produced by structure-breaking transformations such as label shuffling, feature shuffling, and degree-preserving graph rewiring. The desired behavior is $s_\phi(G_\theta^+)\approx s_\phi(G_\theta)$ and $s_\phi(G_\theta^-)>s_\phi(G_\theta)+m$.

We therefore minimize
\begin{align}
\mathcal{L}_{\mathrm{inv}}(\mathbf{w})
&= \mathbb{E}_{\theta,G_\theta^+}\!\left[
\bigl(s_\phi(G_\theta)-s_\phi(G_\theta^{+})\bigr)^2
\right],
\label{eq:linv}\\
\mathcal{L}_{\mathrm{sep}}(\mathbf{w})
&= \mathbb{E}_{\theta,G_\theta^-}\!\left[
\bigl[m+s_\phi(G_\theta)-s_\phi(G_\theta^{-})\bigr]_+
\right],
\label{eq:lsep}\\
\mathcal{L}_{\mathrm{SSL}}(\mathbf{w})
&= \mathcal{L}_{\mathrm{inv}}(\mathbf{w})+\lambda\mathcal{L}_{\mathrm{sep}}(\mathbf{w})
+\mu\|\mathbf{w}\|_2^2,
\label{eq:lssl}
\end{align}
with $\lambda,\mu>0$, optimized by projected gradient descent onto $\mathbb{R}_{\ge0}^{5}$. The theoretical-status review below summarizes the formal stability and optimization results proved in Appendix~\ref{app:theory_proofs}.

\subsection{Score Family and Ablation Variants}
\label{sec:scorefamily}
The non-negative linear form encompasses several ablations. Setting $\mathbf{w}=(1,0,0,0,0)$ yields torsion-only, while $\mathbf{w}=(0,1,0,0,0)$ yields the signed local-irregularity component $-z_{\bar\kappa}$. In the experiments we also report a raw Ricci-only diagnostic using $z_{\bar\kappa}$ with its natural higher-is-better sign. The original fixed \textsc{GeoScore} baseline is recovered by $\mathbf{w}=(\alpha,1-\alpha,0,0,0)$ for fixed $\alpha$, a fixed topology-aware score by $\mathbf{w}=(\alpha,\beta,\gamma,0,0)$, and the learned full model by \textsc{TopoGeoScore}.

\subsection{Theoretical Status and Scope}
Appendix~\ref{app:theory_proofs} separates structural facts about the features from guarantees about learning the score. Proposition~\ref{prop:normalized-mtt} shows that $\log\det{}^{*}(\mathcal{L}_c)$ is not an ad hoc spectral statistic: on each connected component it is a degree-corrected log spanning-tree partition function. Practically, the torsion branch measures global redundancy of class connectivity rather than only a single spectral gap. Proposition~\ref{prop:hodge-cycles} identifies $\ker L_1$ with $H_1(K_c;\mathbb{R})$, so the nullity of the Hodge $1$-Laplacian is anchored to genuine one-dimensional cycles in the class complex.

Theorem~\ref{thm:view-stability} proves that, for fixed-vertex embedding perturbations whose pairwise distance distortion is below the $k$NN margin and relevant spectral-gap thresholds, the mutual $k$NN graph, self-tuning weights, Laplacian log-determinants, unweighted Hodge quantities, and Vietoris--Rips summaries are stable. Corollary~\ref{cor:positive-views} specializes this to exact invariance under orthogonal rotations. These statements justify the positive-view design in the nondegenerate regime; bootstrap resampling changes the vertex set and is therefore treated in the appendix as a sampling-stability heuristic rather than as a direct consequence of Theorem~\ref{thm:view-stability}. On the optimization side, Theorem~\ref{thm:ssl-convex} shows that the empirical invariance--separation objective is strongly convex over $\mathbb{R}_{\ge0}^{5}$ and that non-negative weights enforce coordinatewise monotonicity of the score. Theorem~\ref{thm:ssl-generalization} gives a uniform finite-sample bound comparing the empirical and population view objectives over bounded weights. Together, these results justify the stability, interpretability, and well-posedness of the learned scorer; they do not prove that source-only geometry determines arbitrary OOD accuracy.


\section{Experiments}
\label{sec:experiments}

We evaluate \textsc{TopoGeoScore} on three questions: \textbf{(E1) Predictivity:} does the learned score rank 
checkpoints by OOD accuracy more reliably than torsion alone, 
Ricci alone, or the fixed \textsc{GeoScore}? \textbf{(E2) Consistency:} 
is the score invariant under geometry-preserving views and degraded 
under structure-breaking views, as designed? \textbf{(E3) Selection:} 
when used as an unsupervised selector, how close does it get to 
oracle target-label selection? Settings G--H (ImageNet-C variants, 
MNLI$\to$HANS) and K (OGBN-Arxiv) supply cross-architecture and 
cross-modality replication; full results, scatter grids, severity 
sweeps, and per-metric breakdowns are in 
Appendix~\ref{app:additional}.

\subsection{Setup}
\label{sec:setup}

\paragraph{Settings:}
We use four primary settings and one provisional graph-modality setting 
(Tab.~\ref{tab:settings}). Setting~A is 
the main checkpoint-level study: ResNet-18 trained from scratch on 
CIFAR-10, with checkpoints saved every $10$ epochs ($N\!=\!21$), 
evaluated on CIFAR-10.1 \cite{recht2018cifar} and on the 19 
CIFAR-10-C corruptions across five severities \cite{hendrycks2019benchmarking}. 
Setting~F fine-tunes a pretrained ConvNeXt-T \cite{liu2022convnet} 
on CIFAR-10. Setting~G evaluates five publicly available 
ConvNeXt-T variants on ImageNet$\to$ImageNet-C. Setting~H evaluates 
five MNLI-finetuned language models on MNLI$\to$HANS 
\cite{williams2018broad,mccoy2019right}. A provisional graph-modality 
setting (Setting~K, GCN on OGBN-Arxiv) is reported in 
Appendix~\ref{app:ogbn}.

\begin{table}[h]
\centering
\small
\caption{Empirical scope. $N$ denotes the number of checkpoints or 
models. Settings G--H are reported as supporting evidence for 
generality; the main score comparisons are anchored on A and 
CIFAR-10-C.}
\label{tab:settings}
\begin{tabular}{@{}llrl@{}}
\toprule
ID & Setting / source $\to$ target & $N$ & Type of variation \\
\midrule
A & ResNet-18 ERM, CIFAR-10 $\to$ CIFAR-10.1/10-C   & 21 & Training trajectory \\
F & ConvNeXt-T fine-tune, CIFAR-10 $\to$ CIFAR-10.1/10.2 & 6 & Fine-tuning trajectory \\
G & ConvNeXt-T variants, ImageNet $\to$ ImageNet-C       & 5 & Pretraining variation \\
H & MNLI-finetuned LMs, MNLI $\to$ HANS                  & 5 & NLI transfer \\
K & GCN (2 seeds), OGBN-Arxiv $\to$ OGBN-Arxiv (temporal shift) & $2\times 21$ & Training trajectory \\
\bottomrule
\end{tabular}
\end{table}

\paragraph{What ``target-label-free'' means:}
Source labels enter only through class-conditional graph 
construction; no target-domain labels and no target-domain samples 
(labelled or unlabelled) are used at any stage of diagnosis or 
selection. This is stricter than unsupervised accuracy estimation 
\cite{garg2022leveraging,xie2024mano,deng2024leveraging}, which 
permits unlabelled target data.

\paragraph{Embedding extraction and graph construction:}
We extract $\ell_2$-normalised penultimate-layer features on the 
source-domain validation set, sample $100$ points per class, and 
build mutual $k$-NN graphs ($k\!=\!10$) with the self-tuning kernel of 
Eq.~\eqref{eq:self-tuning}. \textsc{TopoGeoScore} weights are trained 
by projected gradient descent on Eq.~\eqref{eq:lssl} for $1500$ steps 
with $\lambda\!=\!1$, $m\!=\!0.5$, $\mu\!=\!10^{-3}$. Unless otherwise stated, other hyperparameters are held fixed across checkpoints within each setting.

\paragraph{Reporting:}
For each score we report Spearman $\rho$ with OOD accuracy, the 
selected-checkpoint accuracy, the oracle accuracy 
($\max$ OOD accuracy over the family), and the gap to oracle. 
Statistical-significance flags should be read in the context of the 
finite checkpoint family ($N\!=\!21$ on Setting~A; $N\!=\!105$ on 
CIFAR-10-C with $5$ severities $\times$ $21$ checkpoints).

\subsection{Predictivity: fixed scores vs.\ learned \textsc{TopoGeoScore}}
\label{sec:results-main}

\paragraph{Main results (CIFAR-10$\to$CIFAR-10.1 and CIFAR-10-C).}
Table~\ref{tab:main-compare} compares the score family of 
§\ref{sec:scorefamily} on the primary checkpoint trajectory 
(CIFAR-10$\to$CIFAR-10.1, $N\!=\!21$) and its severity-graded extension 
(CIFAR-10-C, $N\!=\!105$). 

Across both settings, the learned \textsc{TopoGeoScore} achieves the 
strongest correlation with OOD accuracy ($\rho\!=\!{-}0.945$ on 
CIFAR-10.1 and $-0.641$ on CIFAR-10-C), consistently outperforming 
torsion-only, topology-aware fixed, and the original fixed 
\textsc{GeoScore}. The gain over fixed \textsc{GeoScore} 
($\Delta\rho\!\approx\!0.36$ on CIFAR-10.1) shows that the learned 
scorer does not simply reproduce the original weighted sum.

All scores other than Ricci-only exhibit the desired negative correlation 
(lower score $\rightarrow$ higher OOD accuracy), while Ricci-only 
correlates positively, reflecting that higher mean curvature already 
indicates robustness under our convention. Importantly, the same 
ordering is preserved on CIFAR-10-C without retraining, demonstrating 
that the learned weights generalise across distribution shifts. While torsion-only is already a strong source-only diagnostic, TopoGeoScore improves the full checkpoint ranking and provides a principled way to combine global, local, and topological signals without hand-tuning. Its main benefit is therefore not replacing torsion in every setting, but stabilising selection when the geometric signals disagree.

\begin{table}[h]
	\centering
	\small
	\caption{\textbf{Predictivity and selection performance.} Spearman $\rho$ 
		vs.\ OOD accuracy and unsupervised selection metrics across two settings. 
		For lower-is-better scores, more negative $\rho$ indicates better alignment; 
		Ricci-only is reported with its raw higher-is-better curvature sign.}
	\label{tab:main-compare}
	\begin{tabular}{@{}llcccc@{}}
		\toprule
		Setting & Score & $\rho$ & Sel. acc & Oracle acc & Gap \\
		\midrule
		\multirow{5}{*}{CIFAR-10$\to$10.1 ($N\!=\!21$)}
		& Torsion-only            & $-0.913$ & 0.8975 & 0.8975 & 0.000 \\
		& Ricci-only              & $+0.527$ & 0.7125 & 0.8975 & 0.185 \\
		& Fixed GeoScore          & $-0.588$ & 0.8625 & 0.8975 & 0.035 \\
		& Topology-aware fixed    & $-0.904$ & 0.8975 & 0.8975 & 0.000 \\
		& \textbf{TopoGeoScore (SSL)} & $\mathbf{-0.945}$ & \textbf{0.8975} & 0.8975 & \textbf{0.000} \\
		\midrule
		\multirow{5}{*}{CIFAR-10-C ($N\!=\!105$)}
		& Torsion-only            & $-0.629$ & 0.92536 & 0.92536 & 0.000 \\
		& Ricci-only              & $+0.265$ & 0.78432 & 0.92536 & 0.141 \\
		& Fixed GeoScore          & $-0.415$ & 0.88842 & 0.92536 & 0.037 \\
		& Topology-aware fixed    & $-0.622$ & 0.92536 & 0.92536 & 0.000 \\
		& \textbf{TopoGeoScore (SSL)} & $\mathbf{-0.641}$ & \textbf{0.92536} & 0.92536 & \textbf{0.000} \\
		\bottomrule
	\end{tabular}
\end{table}

\subsection{Self-supervised ablations}
\label{sec:ablation}

We ablate the components of the self-supervised objective 
(Eq.~\ref{eq:lssl}) on CIFAR-10-C ($N\!=\!105$ anchors). Each variant 
removes one component and retrains $\mathbf{w}$ from scratch.

\begin{table}[h]
	\centering
	\small
	\caption{\textbf{Ablation of the self-supervised objective on CIFAR-10-C} ($N=105$ anchors). Each variant removes one component of Eq.~\ref{eq:lssl} and retrains $\mathbf{w}$. We report Spearman correlation, selection performance, and score variability.}
	\label{tab:ablation}
	\begin{tabular}{@{}lccccc@{}}
		\toprule
		Variant & $\rho$ & Sel. acc & Oracle acc & Gap & Score std \\
		\midrule
		\textbf{Full (\textsc{TopoGeoScore})} 
		& $\mathbf{-0.6408}$ & \textbf{0.92536} & 0.92536 & \textbf{0.0000} & 0.4656 \\
		
		no invariance loss 
		& $-0.6449$ & 0.70546 & 0.92536 & 0.2199 & 1.2989 \\
		
		no separation loss 
		& $-0.6340$ & 0.92536 & 0.92536 & 0.0000 & $2.87\times10^{-6}$ \\
		
		no topology defect ($z_\Omega$) 
		& $-0.5882$ & 0.92536 & 0.92536 & 0.0000 & 0.3511 \\
		
		no interaction term 
		& $-0.6408$ & 0.92536 & 0.92536 & 0.0000 & 0.4597 \\
		
		no monotonic constraint 
		& $-0.5707$ & 0.92536 & 0.92536 & 0.0000 & 0.4890 \\
		
		\bottomrule
	\end{tabular}
\end{table}

Three conclusions emerge. \textbf{(i) Invariance is critical.} 
Removing the invariance loss causes a severe degradation in selection 
performance (gap = 0.22), despite a similar $\rho$, showing that 
rank correlation alone can be misleading without stability under 
geometry-preserving transformations. 

\textbf{(ii) Separation mainly prevents degeneracy.} Removing the 
separation loss retains perfect selection accuracy but collapses the 
learned weights toward trivial solutions, indicating that separation 
acts as a regulariser rather than a primary signal. 

\textbf{(iii) Geometry terms drive predictivity.} Removing the topology 
defect ($z_\Omega$) or the monotonic constraint reduces correlation 
strength, confirming their role in capturing meaningful global--local 
structure. In contrast, the interaction term has negligible effect in 
this setting. 

Overall, these results show that the effectiveness of 
\textsc{TopoGeoScore} is primarily driven by enforcing invariance, 
while topology-aware terms refine the ranking signal and improve robustness prediction.

\subsection{Control consistency: invariance and separation}
\label{sec:controls-results}

By construction, \textsc{TopoGeoScore} should be approximately 
invariant under geometry-preserving views and significantly worse 
under structure-breaking views. We verify this directly on Setting~A 
by scoring every checkpoint under each view family and reporting the 
score change relative to the anchor.

\begin{table}[h]
\centering
\small
\caption{\textbf{Score deltas by view kind.} Mean and standard 
deviation of $s_\phi(G_\theta^{\pm}) - s_\phi(G_\theta)$ across all 
21 checkpoints; positive view families should produce small 
$|\Delta|$, negative view families should produce a large positive 
$\Delta$ (worse score).}
\label{tab:invariance}
\begin{tabular}{@{}llcccc@{}}
\toprule
Kind & View & mean $\Delta$ & |mean| $\Delta$ & std $\Delta$ & $n$ \\
\midrule
Anchor   & --                  & $0.000$ & $0.000$ & $0.000$ & $21$ \\
\addlinespace
Positive (geometry-preserving) & rotate           & $-0.000$ & $0.002$ & $0.002$ & 21 \\
                              & gauss\_noise     & $-0.154$ & $0.165$ & $0.133$ & 21 \\
                              & bootstrap        & $-0.291$ & $0.292$ & $0.122$ & 21 \\
                              & feature\_dropout & $0.035$  & $0.060$ & $0.071$ & 21 \\
                              & pca\_preserving  & $-0.171$ & $0.184$ & $0.158$ & 21 \\
\addlinespace
Negative (structure-breaking) & label\_shuffle    & $0.634$ & $0.634$ & $0.497$ & 21 \\
                              & feature\_shuffle & $2.032$ & $2.032$ & $0.682$ & 21 \\
                              & rewire           & --      & --      & --      & -- \\
\bottomrule
\end{tabular}
\end{table}

The expected pattern holds: anchor deltas are zero; structure-breaking 
views produce a large positive shift in score (consistently worse, 
average shift roughly an order of magnitude larger than the positive 
family); and geometry-preserving views shift the score by a much 
smaller amount. We note that the mean shift on positive views is 
slightly biased away from zero (the score is on average a small 
amount better, not exactly identical, on positive views); this reflects a residual asymmetry between the invariance and separation terms.

\subsection{Checkpoint selection}
\label{sec:selection-results}

Tables~\ref{tab:main-compare} and~\ref{tab:cifar_results} report the selected accuracy, oracle accuracy, and gap to oracle for each score. On Setting~A and CIFAR-10-C, \textsc{TopoGeoScore} matches the strongest fixed components in selection accuracy and improves the full ranking over fixed \textsc{GeoScore}. This is important because selection depends on the top-ranked checkpoint, whereas Spearman correlation measures the reliability of the complete ordering. Per-severity regret is reported in Table~\ref{tab:cifar_severity}.

Splitting the $21$ checkpoints by the median of $\Delta_{GL}$ shows where the learned combination is most useful. In the low-disagreement group, all scores recover oracle or near-oracle selection. In the high-disagreement group, torsion-only and Ricci-only diverge, while \textsc{TopoGeoScore} remains close to oracle. This supports the role of topology as a mediator when global and local geometry disagree; see Fig.~\ref{fig:plot_geoscore_diagnostic}.
Overall, lower torsion is a strong source-only predictor, but topology and self-supervised weighting improve robustness diagnosis when geometric signals conflict, while avoiding hand-tuned score combinations.

\section{Conclusion and Limitations}

We proposed \textsc{TopoGeoScore}, a source-only framework for ranking checkpoints by expected OOD robustness using only labelled source validation data. The method represents each checkpoint through class-conditional mutual $k$-NN graphs and combines global spectral complexity, local Ollivier--Ricci regularity, and higher-order topological defects into an interpretable non-negative score. The weights are learned from source-conditioned graph views by enforcing invariance to geometry-preserving transformations and separation from structure-breaking transformations.
Empirically, \textsc{TopoGeoScore} improves full-ranking consistency over fixed geometric combinations and achieves oracle-level selection in the main CIFAR-based settings, with supporting evidence on ImageNet-C, MNLI$\rightarrow$HANS, and OGBN-Arxiv. These results suggest that source representations can contain useful global--local--topological evidence of robustness before target samples are observed.

The method also has \textbf{limitations}. It is a diagnostic selector rather than a robustness training method, and the source-only setting cannot capture target-specific shifts that leave no measurable trace in the source representation. Some supporting studies use small checkpoint or model families, especially ImageNet-C, MNLI$\rightarrow$HANS, and OGBN-Arxiv, so they should be interpreted as evidence of transferability rather than exhaustive validation. The graph and spectral/topological computations are also more expensive than simple confidence- or norm-based proxies, although they remain practical for post-hoc checkpoint selection.

\bibliographystyle{plainnat}
\bibliography{references}

@inproceedings{bartlett2017spectrally,
  title     = {Spectrally-normalized margin bounds for neural networks},
  author    = {Bartlett, Peter L. and Foster, Dylan J. and Telgarsky, Matus J.},
  booktitle = {Advances in Neural Information Processing Systems},
  volume    = {30},
  year      = {2017}
}

@article{recht2018cifar,
  title   = {Do CIFAR-10 Classifiers Generalize to CIFAR-10?},
  author  = {Recht, Benjamin and Roelofs, Rebecca and Schmidt, Ludwig and Shankar, Vaishaal},
  journal = {arXiv preprint arXiv:1806.00451},
  year    = {2018}
}

@inproceedings{hendrycks2019benchmarking,
  title     = {Benchmarking Neural Network Robustness to Common Corruptions and Perturbations},
  author    = {Hendrycks, Dan and Dietterich, Thomas},
  booktitle = {International Conference on Learning Representations},
  year      = {2019}
}

@inproceedings{cuturi2013sinkhorn,
  title     = {Sinkhorn Distances: Lightspeed Computation of Optimal Transport},
  author    = {Cuturi, Marco},
  booktitle = {Advances in Neural Information Processing Systems},
  volume    = {26},
  pages     = {2292--2300},
  year      = {2013}
}

@article{arjovsky2019irm,
  title={Invariant risk minimization},
  author={Arjovsky, Martin and Bottou, L{\'e}on and Gulrajani, Ishaan and Lopez-Paz, David},
  journal={arXiv preprint arXiv:1907.02893},
  year={2019}
}

@inproceedings{rame2022fishr,
  title={Fishr: Invariant gradient variances for out-of-distribution generalization},
  author={Rame, Alexandre and Dancette, Corentin and Cord, Matthieu},
  booktitle={International Conference on Machine Learning},
  pages={18347--18377},
  year={2022},
  organization={PMLR}
}

@article{cha2021swad,
  title={Swad: Domain generalization by seeking flat minima},
  author={Cha, Junbum and Chun, Sanghyuk and Lee, Kyungjae and Cho, Han-Cheol and Park, Seunghyun and Lee, Yunsung and Park, Sungrae},
  journal={Advances in Neural Information Processing Systems},
  volume={34},
  pages={22405--22418},
  year={2021}
}

@inproceedings{kornblith2019similarity,
  title={Similarity of neural network representations revisited},
  author={Kornblith, Simon and Norouzi, Mohammad and Lee, Honglak and Hinton, Geoffrey},
  booktitle={International conference on machine learning},
  pages={3519--3529},
  year={2019},
  organization={PMlR}
}

@inproceedings{ahuja2020invariant,
  title={Invariant risk minimization games},
  author={Ahuja, Kartik and Shanmugam, Karthikeyan and Varshney, Kush and Dhurandhar, Amit},
  booktitle={International Conference on Machine Learning},
  pages={145--155},
  year={2020},
  organization={PMLR}
}

@InProceedings{krueger2021out,
  title = 	 {Out-of-Distribution Generalization via Risk Extrapolation (REx)},
  author =       {Krueger, David and Caballero, Ethan and Jacobsen, Joern-Henrik and Zhang, Amy and Binas, Jonathan and Zhang, Dinghuai and Priol, Remi Le and Courville, Aaron},
  booktitle = 	 {Proceedings of the 38th International Conference on Machine Learning},
  pages = 	 {5815--5826},
  year = 	 {2021},
  editor = 	 {Meila, Marina and Zhang, Tong},
  volume = 	 {139},
  series = 	 {Proceedings of Machine Learning Research},
  month = 	 {18--24 Jul},
  publisher =    {PMLR}
}

@article{izmailov2018averaging,
  title={Averaging weights leads to wider optima and better generalization},
  author={Izmailov, Pavel and Podoprikhin, Dmitrii and Garipov, Timur and Vetrov, Dmitry and Wilson, Andrew Gordon},
  journal={arXiv preprint arXiv:1803.05407},
  year={2018}
}

@article{nguyen2021do,
  title={Do wide and deep networks learn the same things? uncovering how neural network representations vary with width and depth},
  author={Nguyen, Thao and Raghu, Maithra and Kornblith, Simon},
  journal={arXiv preprint arXiv:2010.15327},
  year={2020}
}

@article{jiang2019fantastic,
  title={Fantastic generalization measures and where to find them},
  author={Jiang, Yiding and Neyshabur, Behnam and Mobahi, Hossein and Krishnan, Dilip and Bengio, Samy},
  journal={arXiv preprint arXiv:1912.02178},
  year={2019}
}

@article{foret2020sharpness,
  title={Sharpness-aware minimization for efficiently improving generalization},
  author={Foret, Pierre and Kleiner, Ariel and Mobahi, Hossein and Neyshabur, Behnam},
  journal={arXiv preprint arXiv:2010.01412},
  year={2020}
}

@inproceedings{wang2020understanding,
  title={Understanding contrastive representation learning through alignment and uniformity on the hypersphere},
  author={Wang, Tongzhou and Isola, Phillip},
  booktitle={International conference on machine learning},
  pages={9929--9939},
  year={2020},
  organization={PMLR}
}

@article{belkin2003laplacian,
  title={Laplacian eigenmaps for dimensionality reduction and data representation},
  author={Belkin, Mikhail and Niyogi, Partha},
  journal={Neural computation},
  volume={15},
  number={6},
  pages={1373--1396},
  year={2003},
  publisher={MIT Press}
}

@article{rieck2018neural,
  title={Neural persistence: A complexity measure for deep neural networks using algebraic topology},
  author={Rieck, Bastian and Togninalli, Matteo and Bock, Christian and Moor, Michael and Horn, Max and Gumbsch, Thomas and Borgwardt, Karsten},
  journal={arXiv preprint arXiv:1812.09764},
  year={2018}
}

@article{ollivier2007ricci,
  title={Ricci curvature of metric spaces},
  author={Ollivier, Yann},
  journal={Comptes Rendus Mathematique},
  volume={345},
  number={11},
  pages={643--646},
  year={2007},
  publisher={Elsevier}
}

@article{guss2018characterizing,
  title={On characterizing the capacity of neural networks using algebraic topology},
  author={Guss, William H and Salakhutdinov, Ruslan},
  journal={arXiv preprint arXiv:1802.04443},
  year={2018}
}

@inproceedings{smola2003kernels,
  title={Kernels and Regularization on Graphs},
  author={Smola, Alexander J. and Kondor, Risi},
  booktitle={Computational Learning Theory and Kernel Machines (COLT/Kernel 2003)},
  series={Lecture Notes in Computer Science},
  volume={2777},
  pages={144--158},
  year={2003},
  publisher={Springer}
}

@inproceedings{topping2022oversquashing,
  title={Understanding over-squashing and bottlenecks on graphs via curvature},
  author={Topping, Jake and Di Giovanni, Francesco and Chamberlain, Benjamin P. and Dong, Xiaowen and Bronstein, Michael M.},
  booktitle={International Conference on Learning Representations (ICLR)},
  year={2022}
}

@inproceedings{barannikov2022representation,
  title={Representation Topology Divergence: A Method for Comparing Neural Network Representations},
  author={Barannikov, Serguei and Trofimov, Ilya and Balabin, Nikita and Burnaev, Evgeny},
  booktitle={Proceedings of the 39th International Conference on Machine Learning (ICML)},
  volume={162},
  pages={1607--1626},
  year={2022},
  publisher={PMLR}
}

@article{vonluxburg2007tutorial,
  title={A tutorial on spectral clustering},
  author={von Luxburg, Ulrike},
  journal={Statistics and Computing},
  year={2007}
}

@article{raysinger1971torsion,
  title={R-torsion and the Laplacian},
  author={Ray, Daniel and Singer, Isadore},
  journal={Advances in Mathematics},
  year={1971}
}

@article{shen2025torgnn,
  author  = {Shen, Cong and Liu, Xiang and Luo, Jiawei and Xia, Kelin},
  title   = {Torsion Graph Neural Networks},
  journal = {IEEE Transactions on Pattern Analysis and Machine Intelligence},
  volume  = {47},
  number  = {4},
  pages   = {2946--2956},
  year    = {2025},
  doi     = {10.1109/TPAMI.2025.3528449}
}

@inproceedings{nguyen2023revisiting,
  author    = {Nguyen, Khang and Nong, Hieu and Nguyen, Vinh and Ho, Nhat and Osher, Stanley and Nguyen, Tan},
  title     = {Revisiting Over-smoothing and Over-squashing Using {Ollivier--Ricci} Curvature},
  booktitle = {Proceedings of the 40th International Conference on Machine Learning (ICML)},
  series    = {PMLR},
  volume    = {202},
  year      = {2023}
}

@article{schaub2021signal,
  author  = {Schaub, Michael T. and Zhu, Yu and Seby, Jean-Baptiste and Roddenberry, T. Mitchell and Segarra, Santiago},
  title   = {Signal Processing on Higher-Order Networks: {Livin'} on the Edge\ldots and Beyond},
  journal = {Signal Processing},
  volume  = {187},
  pages   = {108149},
  year    = {2021},
  doi     = {10.1016/j.sigpro.2021.108149}
}

@inproceedings{bodnar2021weisfeiler,
  author    = {Bodnar, Cristian and Frasca, Fabrizio and Wang, Yu Guang and Otter, Nina and Mont{\'u}far, Guido and Li{\`o}, Pietro and Bronstein, Michael M.},
  title     = {Weisfeiler and {Lehman} Go Topological: Message Passing Simplicial Networks},
  booktitle = {Proceedings of the 38th International Conference on Machine Learning (ICML)},
  series    = {PMLR},
  volume    = {139},
  year      = {2021}
}

@inproceedings{ebli2020simplicial,
  author    = {Ebli, Stefania and Defferrard, Micha{\"e}l and Spreemann, Gard},
  title     = {Simplicial Neural Networks},
  booktitle = {NeurIPS Workshop on Topological Data Analysis and Beyond},
  year      = {2020}
}

@article{huang2024hlhgat,
  author  = {Huang, Jinghan and Chen, Qiufeng and Bian, Yijun and Zhu, Pengli and Chen, Nanguang and Chung, Moo K. and Qiu, Anqi},
  title   = {Advancing Graph Neural Networks with {HL-HGAT}: A {Hodge--Laplacian} and Attention Mechanism Approach for Heterogeneous Graph-Structured Data},
  journal = {arXiv preprint arXiv:2403.06687},
  year    = {2024}
}

@inproceedings{zhou2024facilitating,
  author    = {Zhou, Cai and Wang, Xiyuan and Zhang, Muhan},
  title     = {Facilitating Graph Neural Networks with Random Walk on Simplicial Complexes},
  booktitle = {Advances in Neural Information Processing Systems (NeurIPS)},
  year      = {2024}
}

@inproceedings{keros2022dist2cycle,
  author    = {Keros, Alexandros D. and Nanda, Vidit and Subr, Kartic},
  title     = {{Dist2Cycle}: A Simplicial Neural Network for Homology Localization},
  booktitle = {Proceedings of the AAAI Conference on Artificial Intelligence},
  year      = {2022}
}

@article{bubenik2015persistence,
  author  = {Bubenik, Peter},
  title   = {Statistical Topological Data Analysis Using Persistence Landscapes},
  journal = {Journal of Machine Learning Research},
  volume  = {16},
  pages   = {77--102},
  year    = {2015}
}

@article{adams2017persistence,
  author  = {Adams, Henry and Emerson, Tegan and Kirby, Michael and Neville, Rachel and Peterson, Chris and Shipman, Patrick and Chepushtanova, Sofya and Hanson, Eric and Motta, Francis and Ziegelmeier, Lori},
  title   = {Persistence Images: A Stable Vector Representation of Persistent Homology},
  journal = {Journal of Machine Learning Research},
  volume  = {18},
  number  = {8},
  pages   = {1--35},
  year    = {2017}
}

@inproceedings{rucco2017characterisation,
  author    = {Rucco, Matteo and Castiglione, Filippo and Merelli, Emanuela and Pettini, Marco},
  title     = {Characterisation of the Idiotypic Immune Network through Persistent Entropy},
  booktitle = {Proceedings of the European Conference on Complex Systems (ECCS 2014)},
  pages     = {117--128},
  year      = {2016},
  publisher = {Springer},
  note      = {Persistent entropy summary of persistence diagrams}
}

@inproceedings{birdal2021intrinsic,
  author    = {Birdal, Tolga and Lou, Aaron and Guibas, Leonidas J. and {\c{S}}im{\c{s}}ekli, Umut},
  title     = {Intrinsic Dimension, Persistent Homology and Generalization in Neural Networks},
  booktitle = {Advances in Neural Information Processing Systems (NeurIPS)},
  year      = {2021}
}

@article{gutierrezfandino2021persistent,
  author  = {Guti{\'e}rrez-Fandi{\~n}o, Asier and P{\'e}rez-Fern{\'a}ndez, David and Armengol-Estap{\'e}, Jordi and Villegas, Marta},
  title   = {Persistent Homology Captures the Generalization of Neural Networks Without a Validation Set},
  journal = {arXiv preprint arXiv:2106.00012},
  year    = {2021}
}

@article{ballester2022predicting,
  author  = {Ballester, Rub{\'e}n and Clemente, Xavier Arnal and Casacuberta, Carles and Madadi, Meysam and Corneanu, Ciprian A. and Escalera, Sergio},
  title   = {Predicting the Generalization Gap in Neural Networks Using Topological Data Analysis},
  journal = {arXiv preprint arXiv:2203.12330},
  year    = {2022}
}

@inproceedings{garg2022leveraging,
  author    = {Garg, Saurabh and Balakrishnan, Sivaraman and Lipton, Zachary C. and Neyshabur, Behnam and Sedghi, Hanie},
  title     = {Leveraging Unlabeled Data to Predict Out-of-Distribution Performance},
  booktitle = {International Conference on Learning Representations (ICLR)},
  year      = {2022}
}

@inproceedings{yu2022predicting,
  author    = {Yu, Yaodong and Yang, Zitong and Wei, Alexander and Ma, Yi and Steinhardt, Jacob},
  title     = {Predicting Out-of-Distribution Error with the Projection Norm},
  booktitle = {Proceedings of the 39th International Conference on Machine Learning (ICML)},
  series    = {PMLR},
  volume    = {162},
  year      = {2022}
}

@article{deng2024leveraging,
  author  = {Deng, Renchunzi and Wei, Hongxin and Zhang, Zhi and Cao, Yuzhou and Feng, Lei and An, Bo},
  title   = {Leveraging Gradients for Unsupervised Accuracy Estimation under Distribution Shift},
  journal = {arXiv preprint arXiv:2401.08909},
  year    = {2024}
}

@inproceedings{xie2024mano,
  author    = {Xie, Renchunzi and Wei, Hongxin and Cao, Yuzhou and Feng, Lei and An, Bo},
  title     = {{MaNo}: Exploiting Matrix Norm for Unsupervised Accuracy Estimation under Distribution Shifts},
  booktitle = {Advances in Neural Information Processing Systems (NeurIPS)},
  year      = {2024}
}

@inproceedings{deng2021labels,
  author    = {Deng, Weijian and Zheng, Liang},
  title     = {Are Labels Always Necessary for Classifier Accuracy Evaluation?},
  booktitle = {Proceedings of the IEEE/CVF Conference on Computer Vision and Pattern Recognition (CVPR)},
  year      = {2021}
}

@article{deng2025confidence,
  author  = {Deng, Weijian and others},
  title   = {Confidence and Dispersity as Signals: Unsupervised Model Evaluation and Ranking},
  journal = {arXiv preprint arXiv:2510.02956},
  year    = {2025}
}

@article{roddenberry2022signal,
  author  = {Roddenberry, T. Mitchell and Schaub, Michael T. and Hajij, Mustafa},
  title   = {Signal Processing on Cell Complexes},
  journal = {arXiv preprint arXiv:2110.05614},
  year    = {2022}
}

@inproceedings{liu2022convnet,
  author    = {Liu, Zhuang and Mao, Hanzi and Wu, Chao-Yuan and Feichtenhofer, Christoph and Darrell, Trevor and Xie, Saining},
  title     = {A {ConvNet} for the 2020s},
  booktitle = {Proceedings of the IEEE/CVF Conference on Computer Vision and Pattern Recognition (CVPR)},
  year      = {2022}
}

@inproceedings{williams2018broad,
  author    = {Williams, Adina and Nangia, Nikita and Bowman, Samuel R.},
  title     = {A Broad-Coverage Challenge Corpus for Sentence Understanding through Inference},
  booktitle = {Proceedings of the 2018 Conference of the North American Chapter of the Association for Computational Linguistics: Human Language Technologies (NAACL-HLT)},
  year      = {2018}
}

@inproceedings{mccoy2019right,
  author    = {McCoy, R. Thomas and Pavlick, Ellie and Linzen, Tal},
  title     = {Right for the Wrong Reasons: Diagnosing Syntactic Heuristics in Natural Language Inference},
  booktitle = {Proceedings of the 57th Annual Meeting of the Association for Computational Linguistics (ACL)},
  year      = {2019}
}

@inproceedings{chen2020simple,
  title     = {A Simple Framework for Contrastive Learning of Visual Representations},
  author    = {Chen, Ting and Kornblith, Simon and Norouzi, Mohammad and Hinton, Geoffrey},
  booktitle = {Proceedings of the 37th International Conference on Machine Learning (ICML)},
  pages     = {1597--1607},
  year      = {2020}
}

@inproceedings{he2020momentum,
  title     = {Momentum Contrast for Unsupervised Visual Representation Learning},
  author    = {He, Kaiming and Fan, Haoqi and Wu, Yuxin and Xie, Saining and Girshick, Ross},
  booktitle = {Proceedings of the IEEE/CVF Conference on Computer Vision and Pattern Recognition (CVPR)},
  pages     = {9729--9738},
  year      = {2020}
}

@inproceedings{grill2020byol,
  title     = {Bootstrap Your Own Latent: A New Approach to Self-Supervised Learning},
  author    = {Grill, Jean-Bastien and Strub, Florian and Altch{\'e}, Florent and Tallec, Corentin and Richemond, Pierre H. and Buchatskaya, Elena and Doersch, Carl and Pires, Bernardo Avila and Guo, Zhaohan Daniel and Azar, Mohammad Gheshlaghi and Piot, Bilal and Kavukcuoglu, Koray and Munos, R{\'e}mi and Valko, Michal},
  booktitle = {Advances in Neural Information Processing Systems (NeurIPS)},
  volume    = {33},
  pages     = {21271--21284},
  year      = {2020}
}

@inproceedings{chen2021exploring,
  title     = {Exploring Simple Siamese Representation Learning},
  author    = {Chen, Xinlei and He, Kaiming},
  booktitle = {Proceedings of the IEEE/CVF Conference on Computer Vision and Pattern Recognition (CVPR)},
  pages     = {15750--15758},
  year      = {2021}
}

@inproceedings{zbontari2021barlow,
  title     = {Barlow Twins: Self-Supervised Learning via Redundancy Reduction},
  author    = {Zbontar, Jure and Jing, Li and Misra, Ishan and LeCun, Yann and Deny, St{\'e}phane},
  booktitle = {Proceedings of the 38th International Conference on Machine Learning (ICML)},
  pages     = {12310--12320},
  year      = {2021}
}

@inproceedings{bardes2022vicreg,
  title     = {{VICReg}: Variance-Invariance-Covariance Regularization for Self-Supervised Learning},
  author    = {Bardes, Adrien and Ponce, Jean and LeCun, Yann},
  booktitle = {International Conference on Learning Representations (ICLR)},
  year      = {2022}
}

@inproceedings{haochen2021provable,
  title     = {Provable Guarantees for Self-Supervised Deep Learning with Spectral Contrastive Loss},
  author    = {HaoChen, Jeff Z. and Wei, Colin and Gaidon, Adrien and Ma, Tengyu},
  booktitle = {Advances in Neural Information Processing Systems (NeurIPS)},
  volume    = {34},
  pages     = {5000--5011},
  year      = {2021}
}

@inproceedings{guillory2021predicting,
  title     = {Predicting with Confidence on Unseen Distributions},
  author    = {Guillory, Devin and Shankar, Vaishaal and Ebrahimi, Sayna and Darrell, Trevor and Schmidt, Ludwig},
  booktitle = {Proceedings of the IEEE/CVF International Conference on Computer Vision (ICCV)},
  pages     = {1134--1144},
  year      = {2021}
}

@inproceedings{baek2022agreementline,
  title     = {Agreement-on-the-Line: Predicting the Performance of Neural Networks under Distribution Shift},
  author    = {Baek, Christina and Jiang, Yiding and Raghunathan, Aditi and Kolter, J. Zico},
  booktitle = {Advances in Neural Information Processing Systems (NeurIPS)},
  volume    = {35},
  pages     = {19274--19289},
  year      = {2022}
}

@inproceedings{dziugaite2017computing,
  title     = {Computing Nonvacuous Generalization Bounds for Deep (Stochastic) Neural Networks with Many More Parameters than Training Data},
  author    = {Dziugaite, Gintare Karolina and Roy, Daniel M.},
  booktitle = {Proceedings of the Thirty-Third Conference on Uncertainty in Artificial Intelligence (UAI)},
  year      = {2017}
}

@inproceedings{hu2020ogb,
  title={Open Graph Benchmark: Datasets for Machine Learning on Graphs},
  author={Hu, Weihua and Fey, Matthias and Zitnik, Marinka and Dong, Yuxiao and Ren, Hongyu and Liu, Bowen and Catasta, Michele and Leskovec, Jure},
  booktitle={Advances in Neural Information Processing Systems (NeurIPS)},
  year={2020}
}

@article{zia2024topological,
  title={Topological deep learning: a review of an emerging paradigm},
  author={Zia, Ali and Khamis, Abdelwahed and Nichols, James and Tayab, Usman Bashir and Hayder, Zeeshan and Rolland, Vivien and Stone, Eric and Petersson, Lars},
  journal={Artificial Intelligence Review},
  volume={57},
  pages={77},
  year={2024}
}

@article{bartlett2002rademacher,
  author    = {Peter L. Bartlett and Shahar Mendelson},
  title     = {Rademacher and Gaussian Complexities: Risk Bounds and Structural Results},
  journal   = {Journal of Machine Learning Research},
  volume    = {3},
  pages     = {463--482},
  year      = {2002}
}

@book{bhatia1997matrix,
  author    = {Rajendra Bhatia},
  title     = {Matrix Analysis},
  publisher = {Springer},
  year      = {1997}
}

@inproceedings{chazal2009proximity,
  author    = {Fr{\'e}d{\'e}ric Chazal and Vin de Silva and Marc Glisse and Steve Oudot},
  title     = {Proximity of Persistence Modules and Their Diagrams},
  booktitle = {Proceedings of the 25th Annual Symposium on Computational Geometry},
  pages     = {237--246},
  year      = {2009}
}

@book{chung1997spectral,
  author    = {Fan R. K. Chung},
  title     = {Spectral Graph Theory},
  series    = {CBMS Regional Conference Series in Mathematics},
  volume    = {92},
  publisher = {American Mathematical Society},
  year      = {1997}
}

@article{cohensteiner2007stability,
  author    = {David Cohen-Steiner and Herbert Edelsbrunner and John Harer},
  title     = {Stability of Persistence Diagrams},
  journal   = {Discrete \& Computational Geometry},
  volume    = {37},
  number    = {1},
  pages     = {103--120},
  year      = {2007}
}

@article{eckmann1944harmonische,
  author    = {Beno Eckmann},
  title     = {Harmonische Funktionen und Randwertaufgaben in einem Komplex},
  journal   = {Commentarii Mathematici Helvetici},
  volume    = {17},
  pages     = {240--255},
  year      = {1944}
}

@article{lyons2005determinantal,
  author    = {Russell Lyons},
  title     = {Determinantal Probability Measures},
  journal   = {Publications Math{\'e}matiques de l'IH{\'E}S},
  volume    = {98},
  pages     = {167--212},
  year      = {2003}
}

@book{shalev2014understanding,
  author    = {Shai Shalev-Shwartz and Shai Ben-David},
  title     = {Understanding Machine Learning: From Theory to Algorithms},
  publisher = {Cambridge University Press},
  year      = {2014}
}


\newpage

\appendix
\section{Additional Results}
\label{app:additional}

\begin{figure}[h!]
    \centering
    \includegraphics[width=0.8\linewidth]{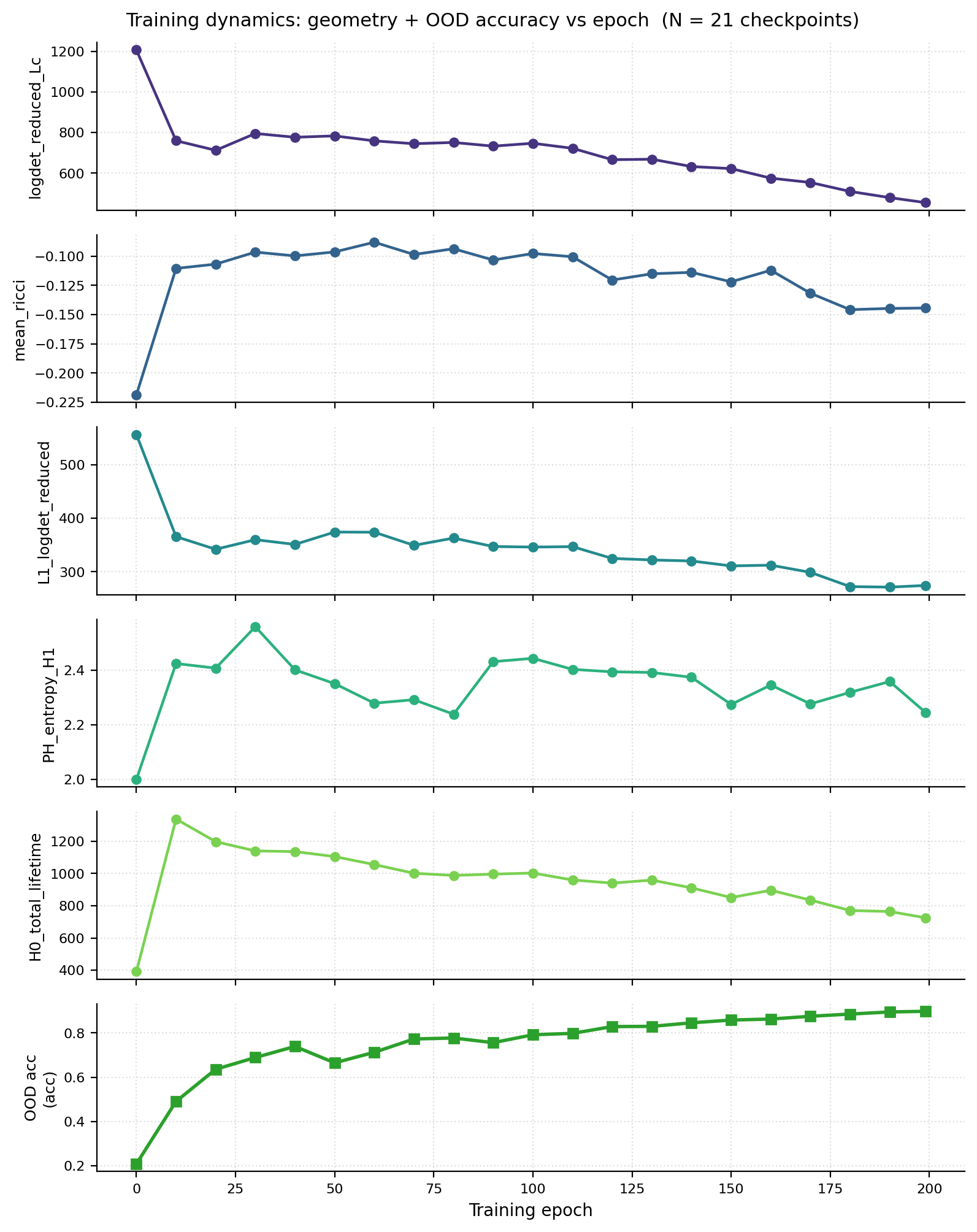}
    \caption{Training dynamics plots.}
    \label{fig:plot_training_dynamics}
\end{figure}

\begin{figure}[h!]
    \centering
    \includegraphics[width=0.8\linewidth]{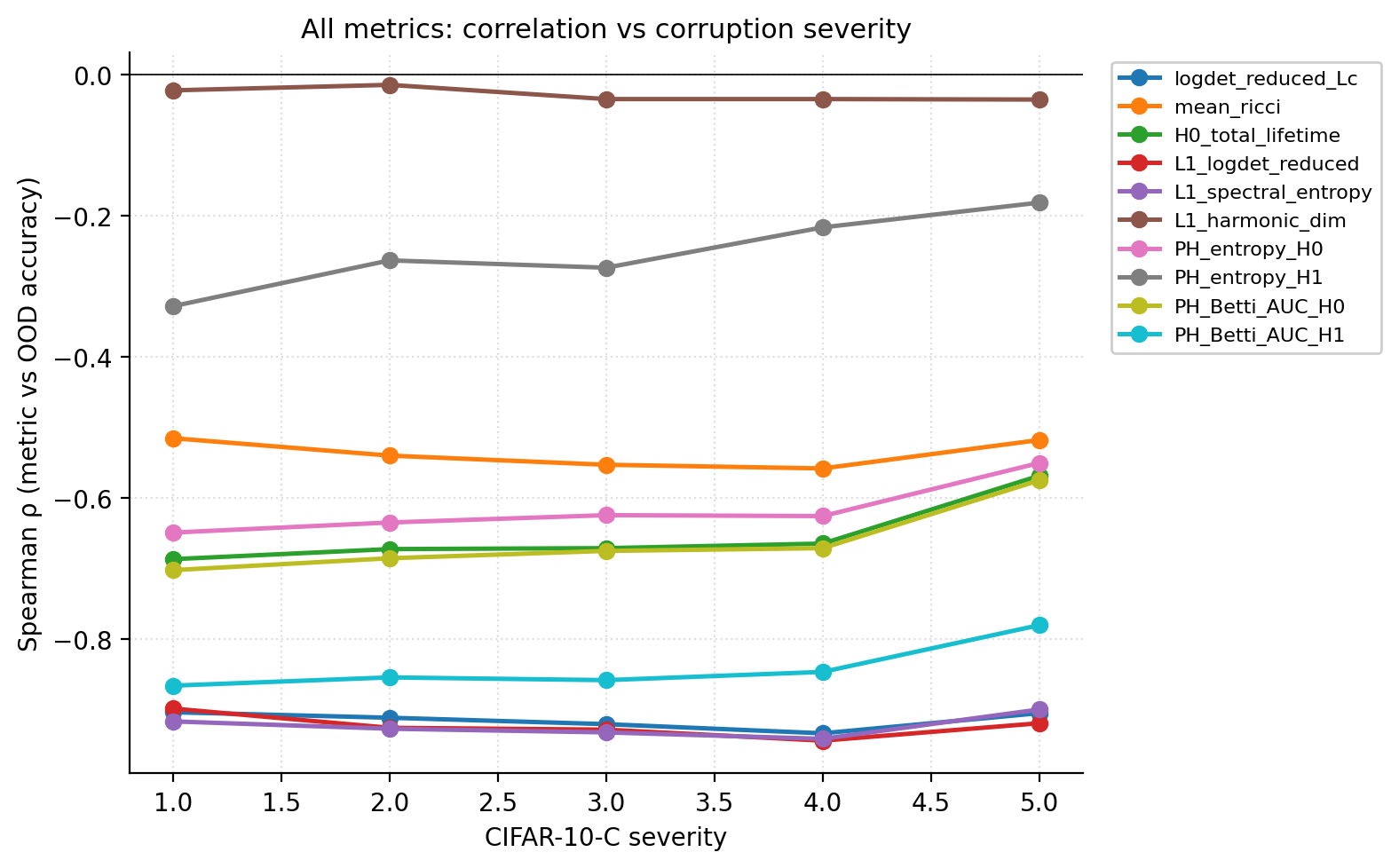}
    \caption{Severity sweep of all metrics on CIFAR-10-C.}
    \label{fig:severity_sweep_all_metrics}
\end{figure}

\begin{figure}[h!]
    \centering
    \includegraphics[width=1\linewidth]{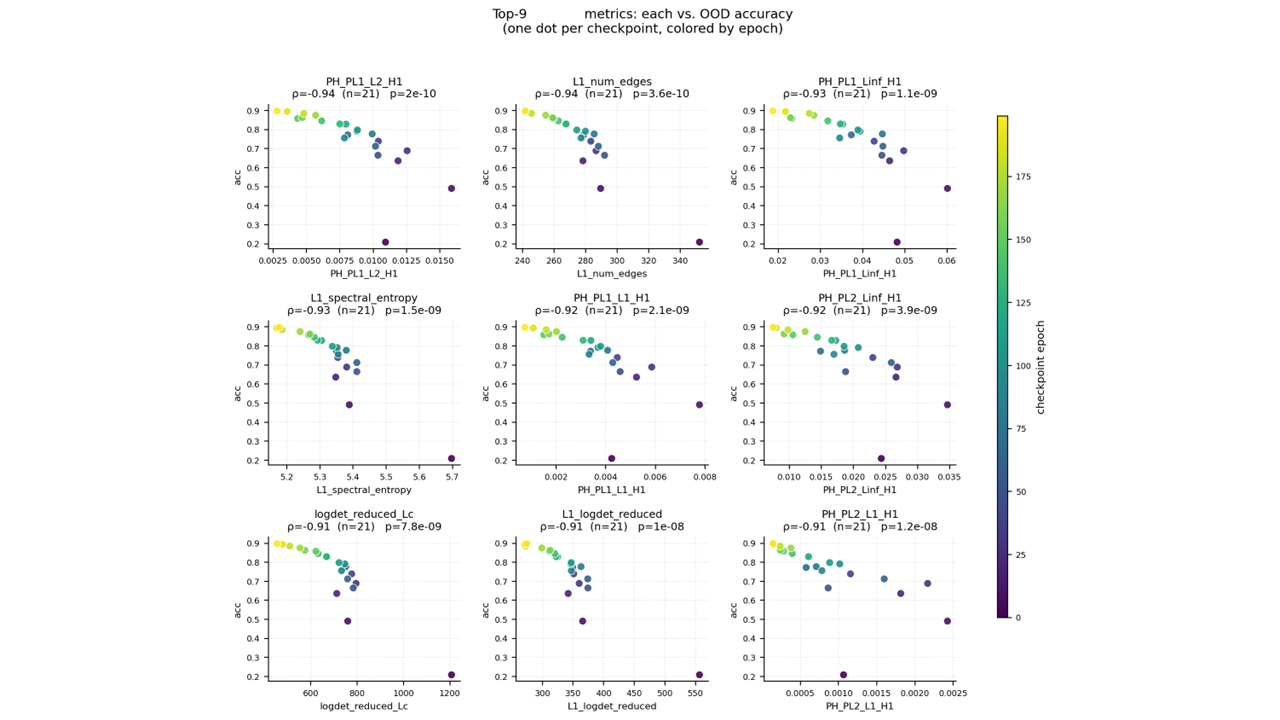}
    \caption{Small-multiples scatter grid: each panel shows one top metric 
versus OOD accuracy; points are checkpoints colored by epoch.}
    \label{fig:per_ckpt_grid}
\end{figure}

\begin{figure}[h!]
    \centering
    \includegraphics[width=1\textwidth]{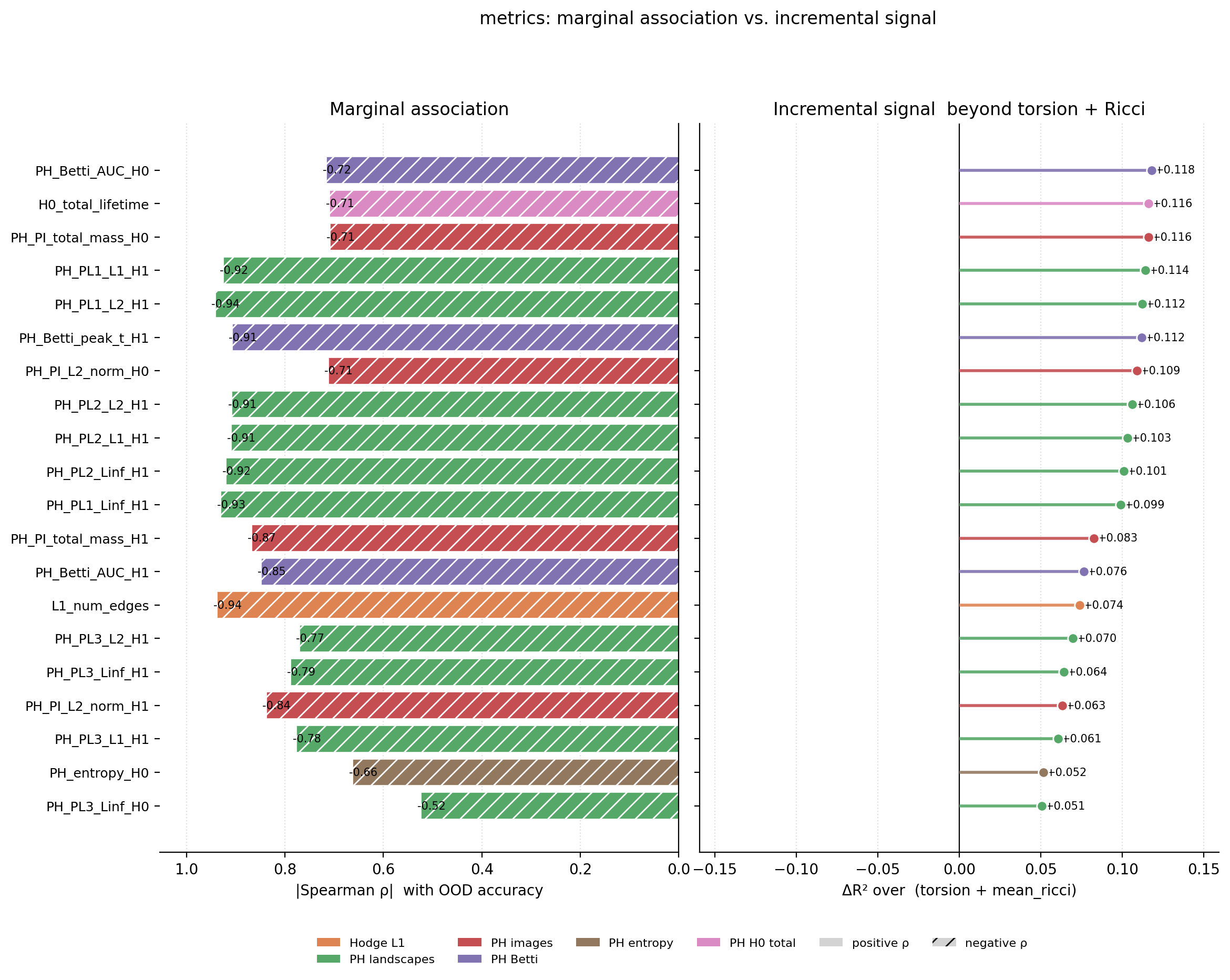}
    \caption{Joint view of $|\rho|$ (marginal predictive power)
and $\Delta R^2$ (incremental power beyond torsion + Ricci). Shows which metrics add
information that torsion/Ricci miss.}
    \label{fig:marginal_vs_dR2}
\end{figure}

Figure~\ref{fig:marginal_vs_dR2} investigates whether each new metric adds 
\emph{incremental} signal beyond torsion + mean Ricci curvature. We report the
partial correlation of the new metric with OOD accuracy while controlling
for the torsion+Ricci linear fit, and the $\Delta R^2$ from adding the new
metric to a linear model using torsion+Ricci as predictors. Specifically, for 
metric $M$, we fit the linear model $y = a\tau + b\kappa + c$ and compare 
$R^2$ with and without $M$.

\begin{figure}[h!]
    \centering
    \includegraphics[width=1\textwidth]{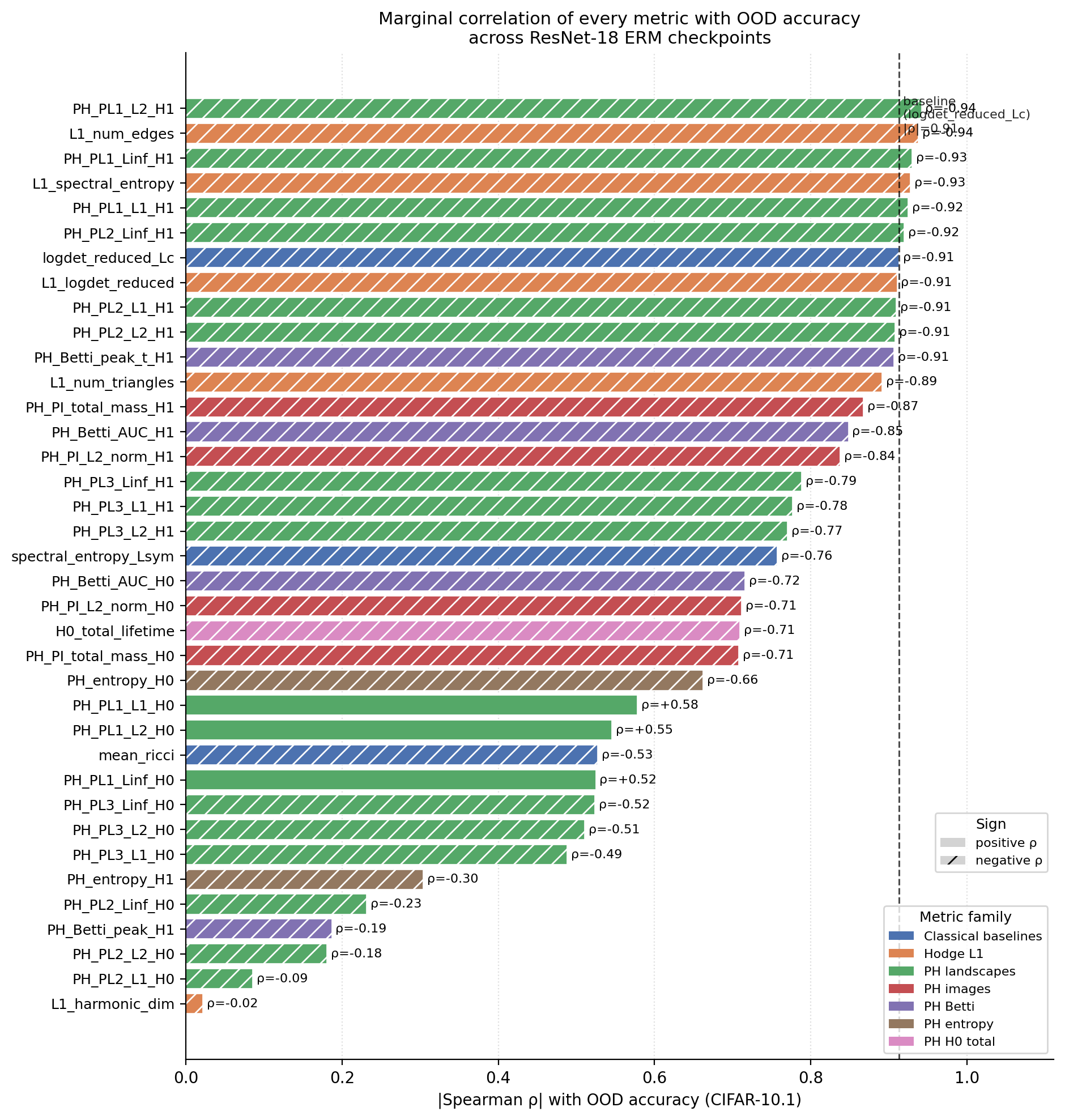}
    \caption{Horizontal bar chart of Spearman $\rho$ for
every metric versus OOD accuracy, with $\mathrm{Logdet\_reduced\_L_c}$ shown as the baseline; metrics are sorted and family-colored.}
    \label{fig:marginal_spearman}
\end{figure}

\begin{figure}[h!]
    \centering
    \includegraphics[width=1\textwidth]{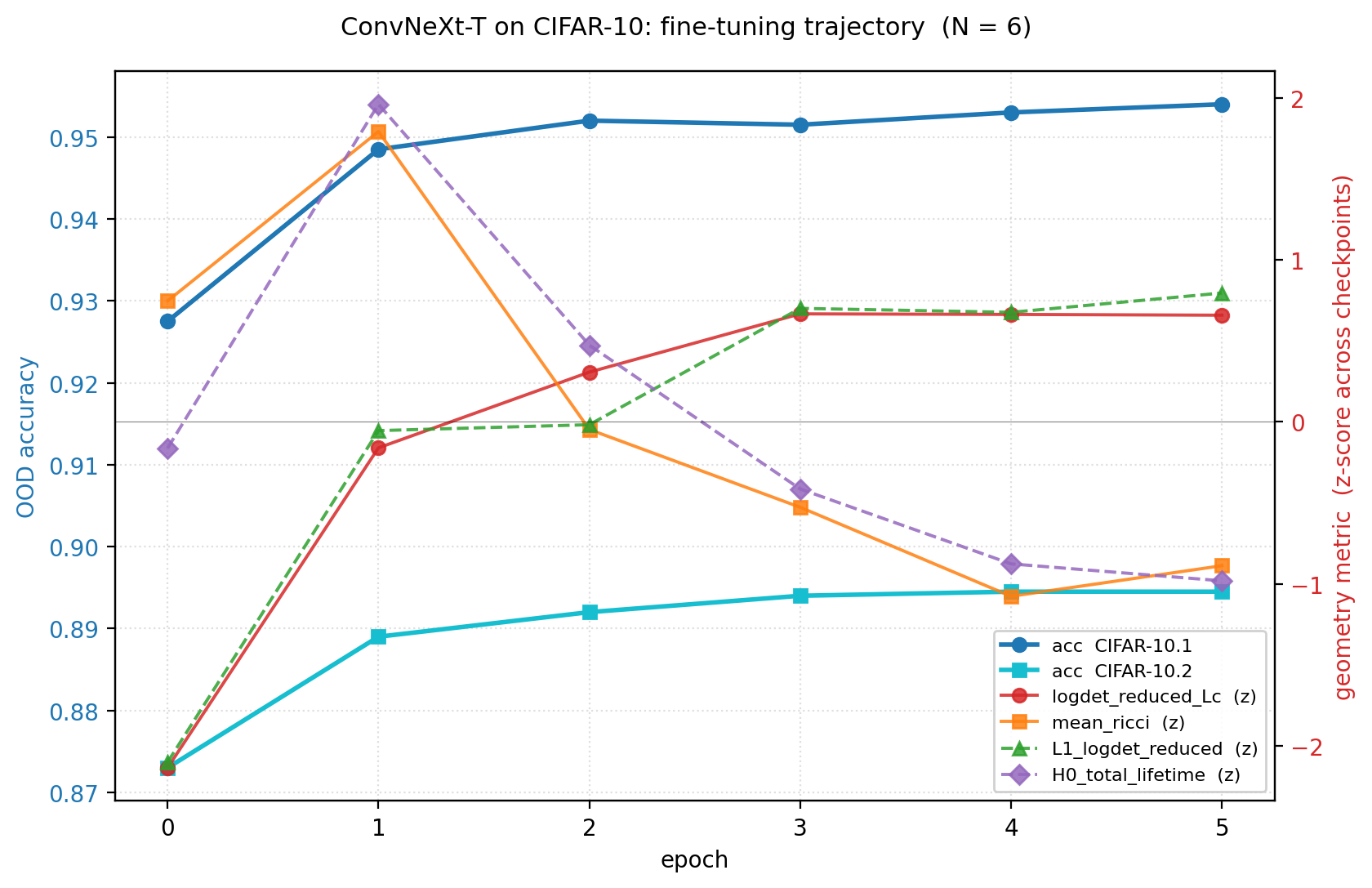}
    \caption{Line plot of accuracy on CIFAR-10.1 and CIFAR-10.2 across
epochs, with $z$-normalised key geometry metrics overlaid.}
    \label{fig:convnext_training_trajectory}
\end{figure}

\begin{figure}[h!]
    \centering
    \includegraphics[width=0.6\textwidth]{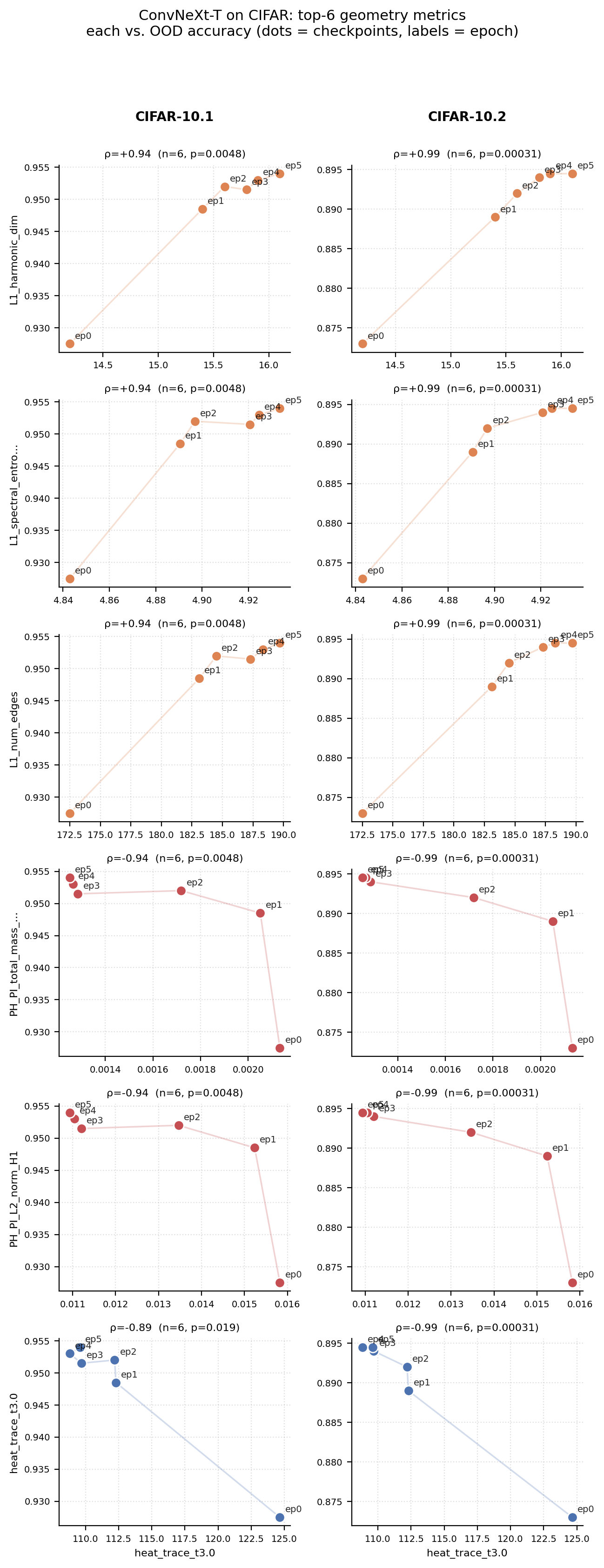}
    \caption{Top-6 metrics $\times$ 2 OOD targets scatter grid,
with points labeled by epoch.}
    \label{fig:convnext_scatter_grid}
\end{figure}

\begin{figure}[h!]
    \centering
    \includegraphics[width=0.7\textwidth]{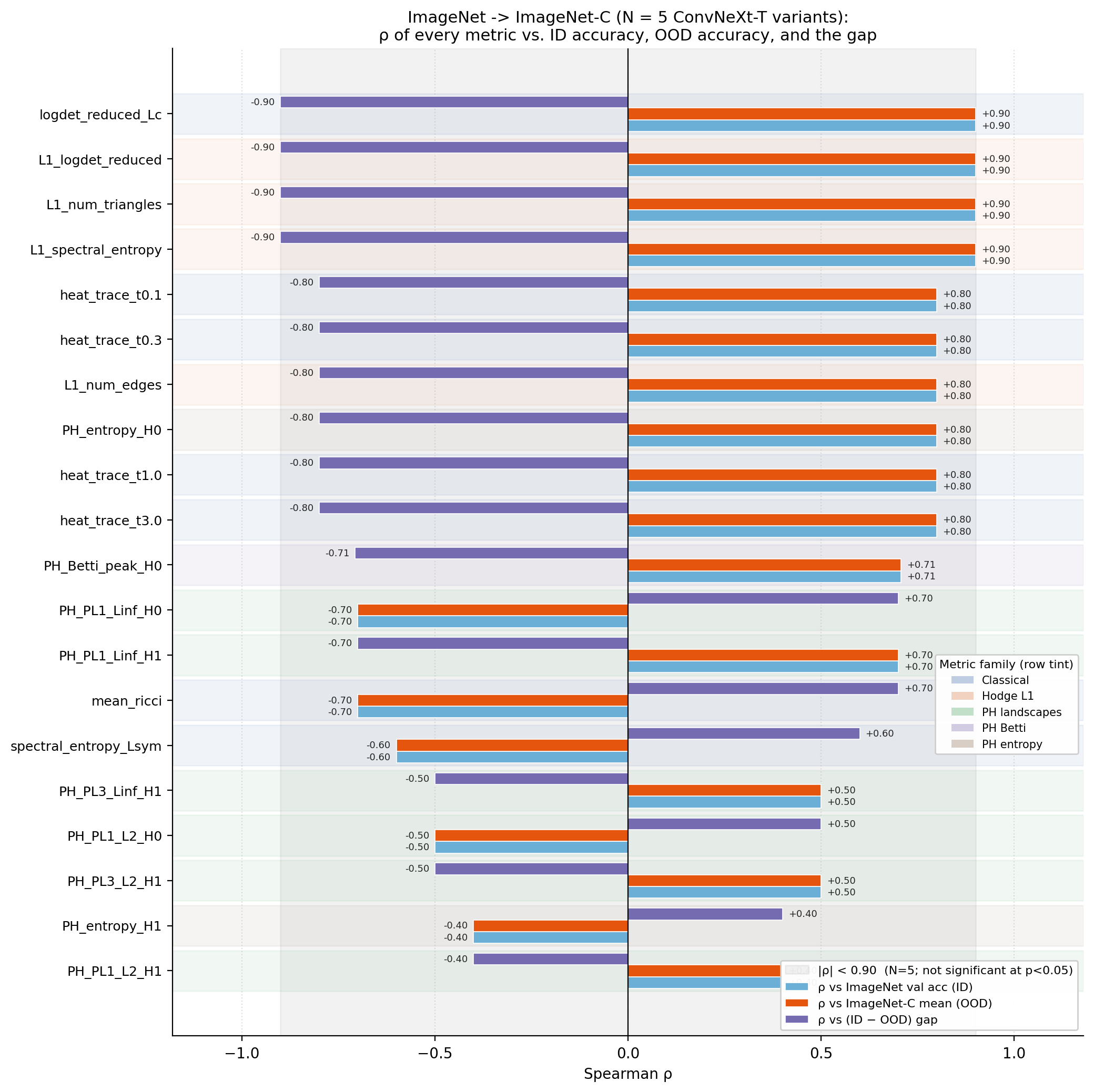}
    \caption{Metric-wise Spearman correlations with ID accuracy, OOD accuracy, and the ID--OOD gap for ImageNet-C variants. Metrics are sorted by $|\rho|$ with respect to the gap; the grey band indicates the non-significant region for $N=5$.}
    \label{fig:imagenetc_rho_three_way}
\end{figure}

Figure~\ref{fig:imagenetc_rho_three_way} shows that, across five pretrained 
ConvNeXt-T variants on ImageNet $\rightarrow$ ImageNet-C, classical Laplacian 
torsion ($\mathrm{Logdet\_reduced\_L_c}$) and three Hodge $L_{1}$ metrics 
($\mathrm{Logdet\_reduced}$, num\_triangles, spectral\_entropy) all achieve 
$\rho = 0.90$ ($p = 0.037$) with both ID and OOD accuracy. This replicates and 
extends the result on a modern ConvNet architecture and a 1000-class natural-image 
OOD benchmark. The Hodge $L_{1}$ family achieves comparable predictive value to 
classical torsion, supporting the hypothesis that higher-order combinatorial 
geometry captures the same robustness-relevant structure.

\begin{figure}[h!]
    \centering
    \includegraphics[width=1\textwidth]{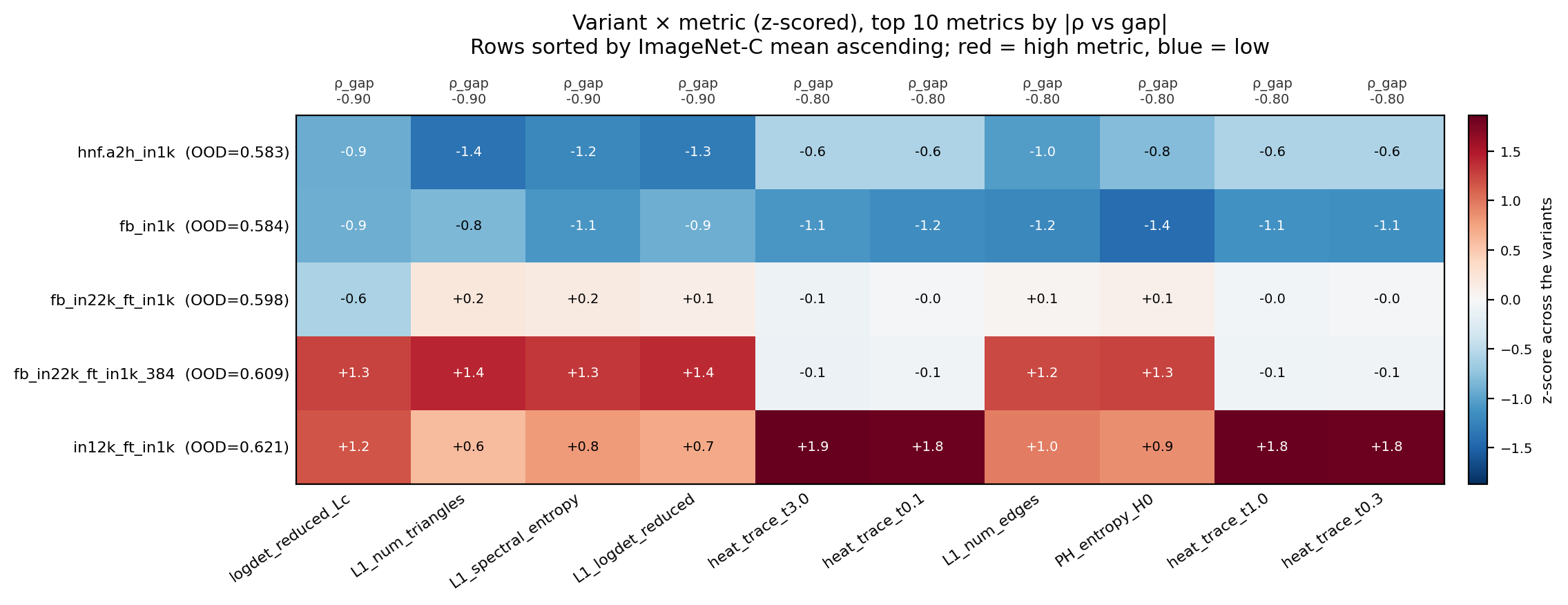}
    \caption{$5$ variants $\times$ top-$10$ metrics, with $z$-scored values and $\rho$ with respect to the ID--OOD gap shown above each column.}
    \label{fig:imagenetc_variant_metric_heatmap}
\end{figure}

\begin{figure}[h!]
    \centering
    \includegraphics[width=0.8\textwidth]{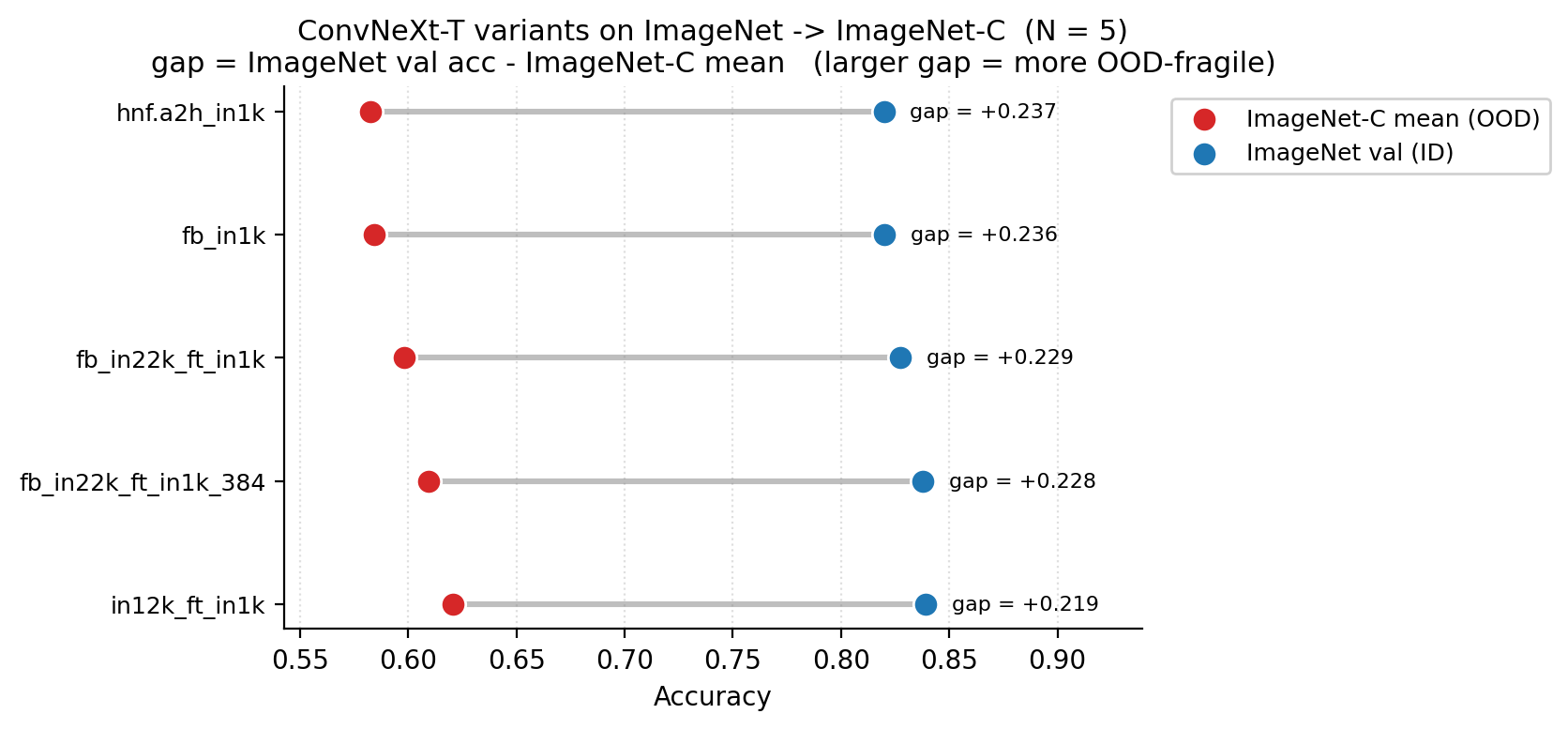}
    \caption{ImageNet validation accuracy and ImageNet-C mean accuracy for each variant, with the ID--OOD gap annotated.}
    \label{fig:imagenetc_variant_summary}
\end{figure}

\begin{figure}[h!]
    \centering
    \includegraphics[width=0.8\textwidth]{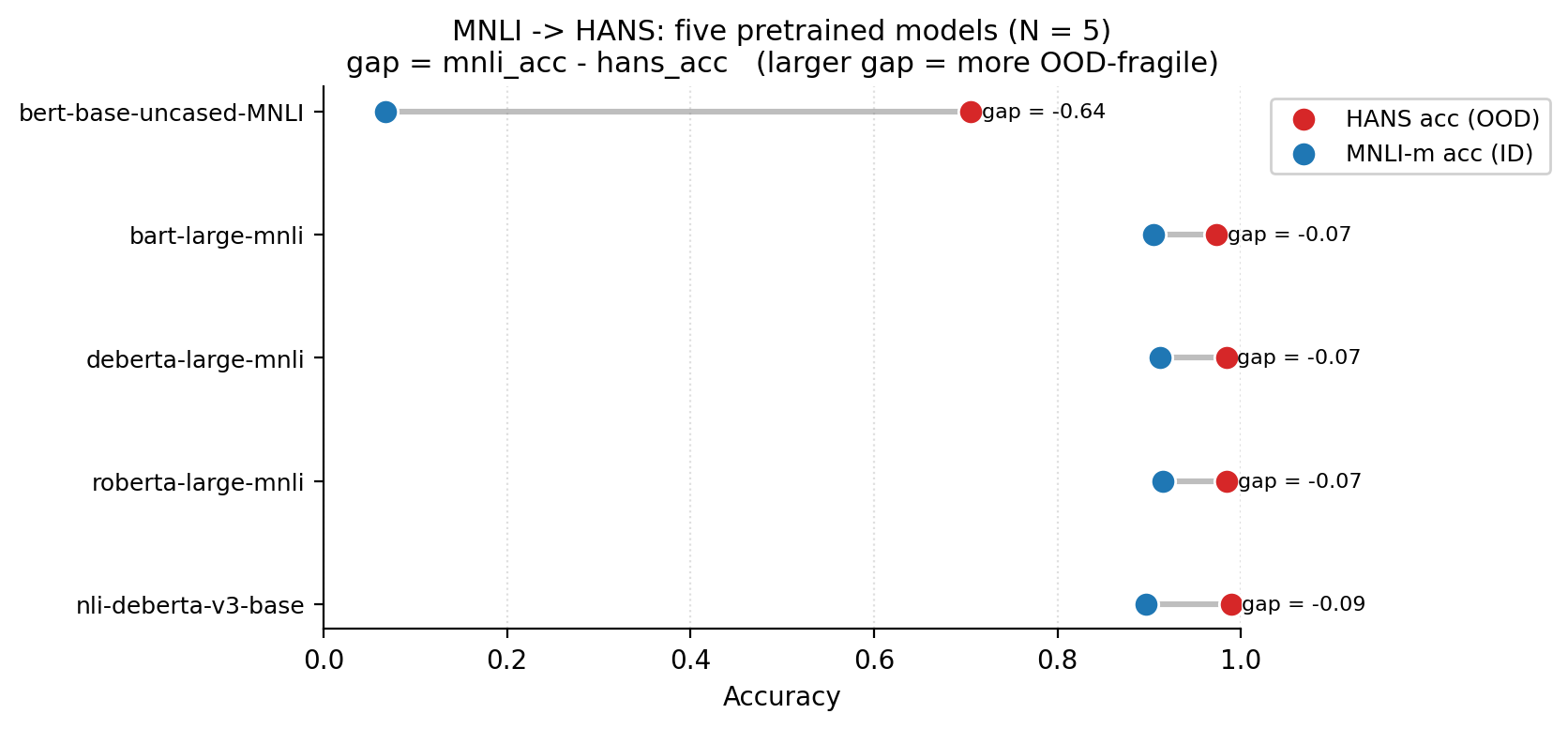}
    \caption{MNLI-m and HANS accuracy for each model, with the ID--OOD gap annotated.}
    \label{fig:nli_model_summary}
\end{figure}

\begin{figure}[h!]
	\centering
	\includegraphics[width=0.8\textwidth]{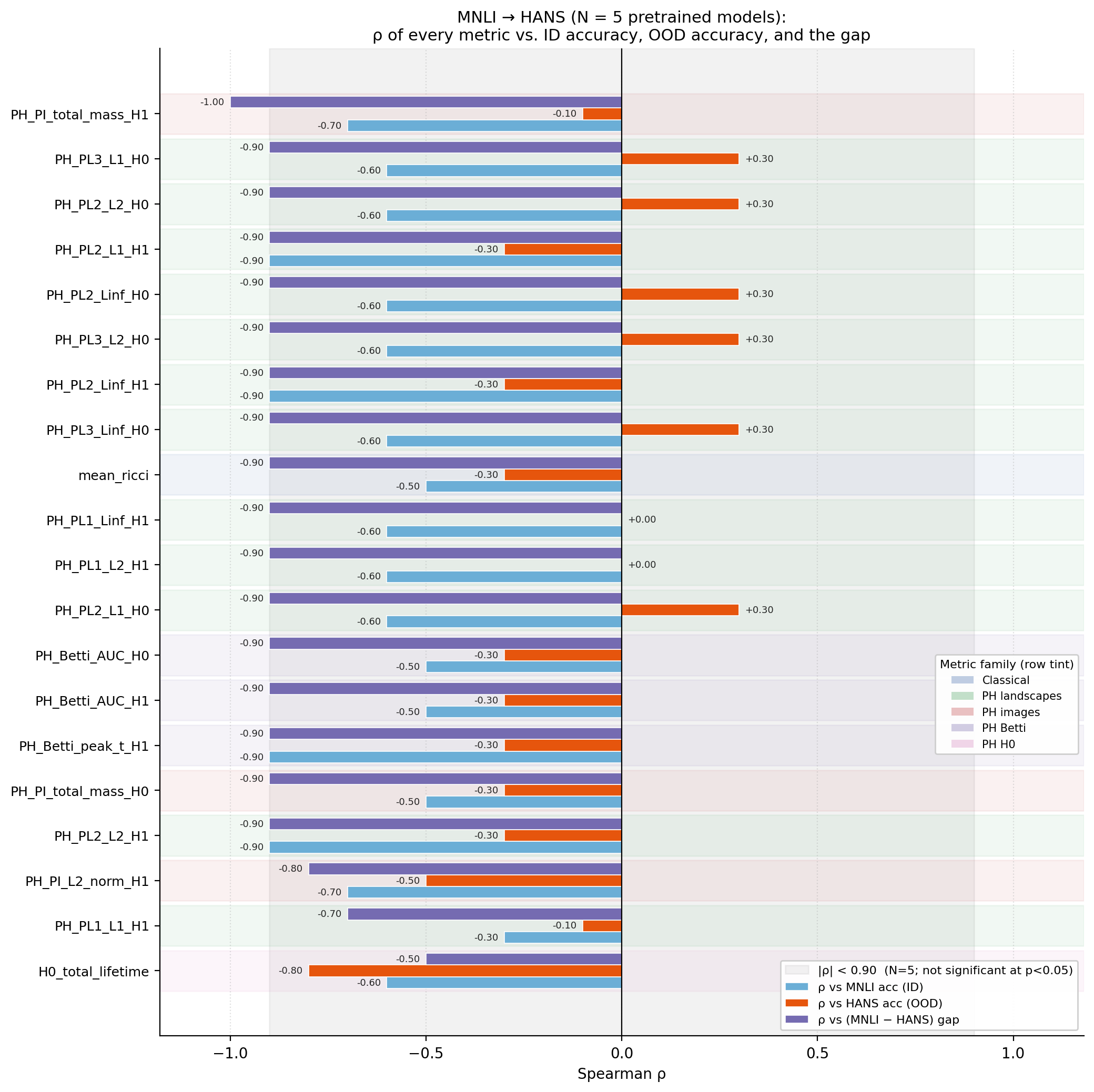}
	\caption{Metric correlations on MNLI$\to$HANS (N=5 models). Bars show Spearman $\rho$ between each metric and in-distribution accuracy (MNLI), out-of-distribution accuracy (HANS), and the generalization gap (MNLI$-$HANS). While many metrics strongly correlate with MNLI accuracy, their alignment with HANS accuracy is weaker and sometimes inconsistent. In contrast, correlations with the generalization gap are consistently strong (typically $\rho \approx -0.9$), indicating that these geometry- and topology-based metrics primarily capture robustness and shortcut-induced failure modes rather than raw accuracy.}
	\label{fig:plot_nli_rho_three_way}
\end{figure}

\begin{figure}[h!]
    \centering
    \includegraphics[width=0.9\textwidth]{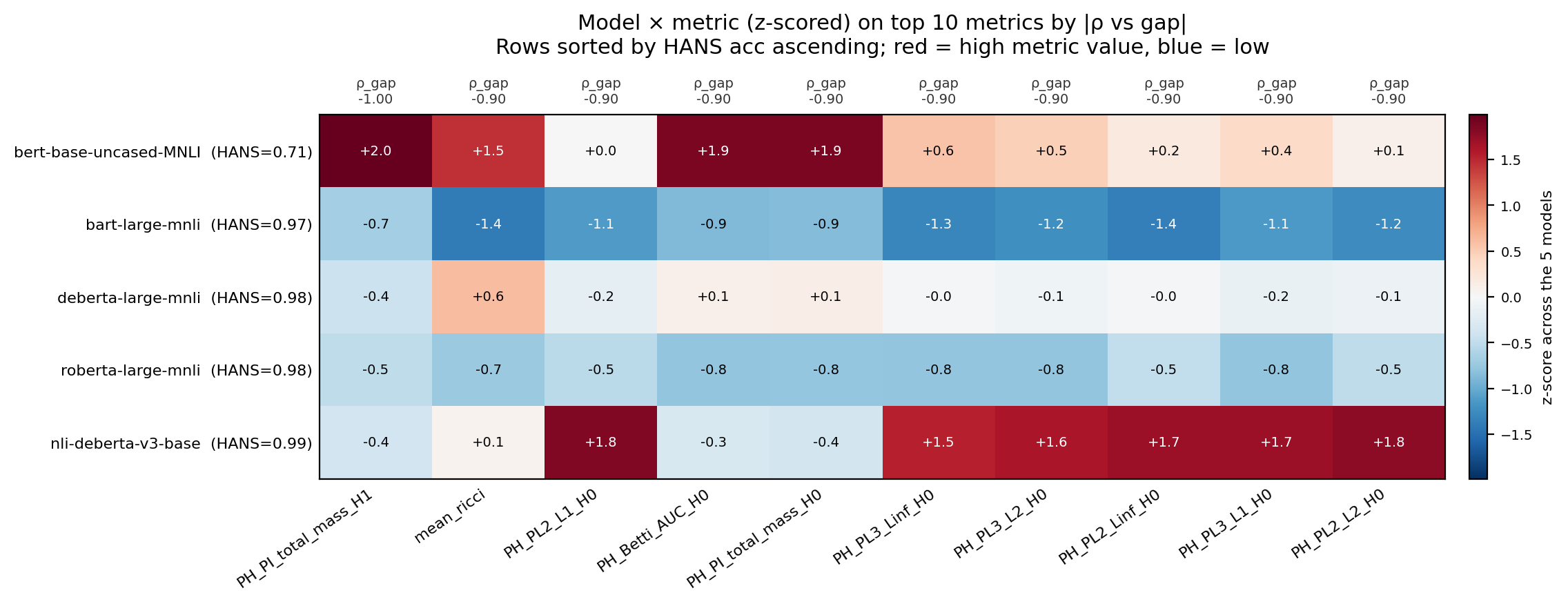}
    \caption{$5$ models $\times$ top-$10$ metrics heatmap, $z$-scored values, $\rho$-vs-gap headers.}
    \label{fig:nli_model_metric_heatmap}
\end{figure}

\begin{figure}[h!]
	\centering
	\includegraphics[width=0.9\textwidth]{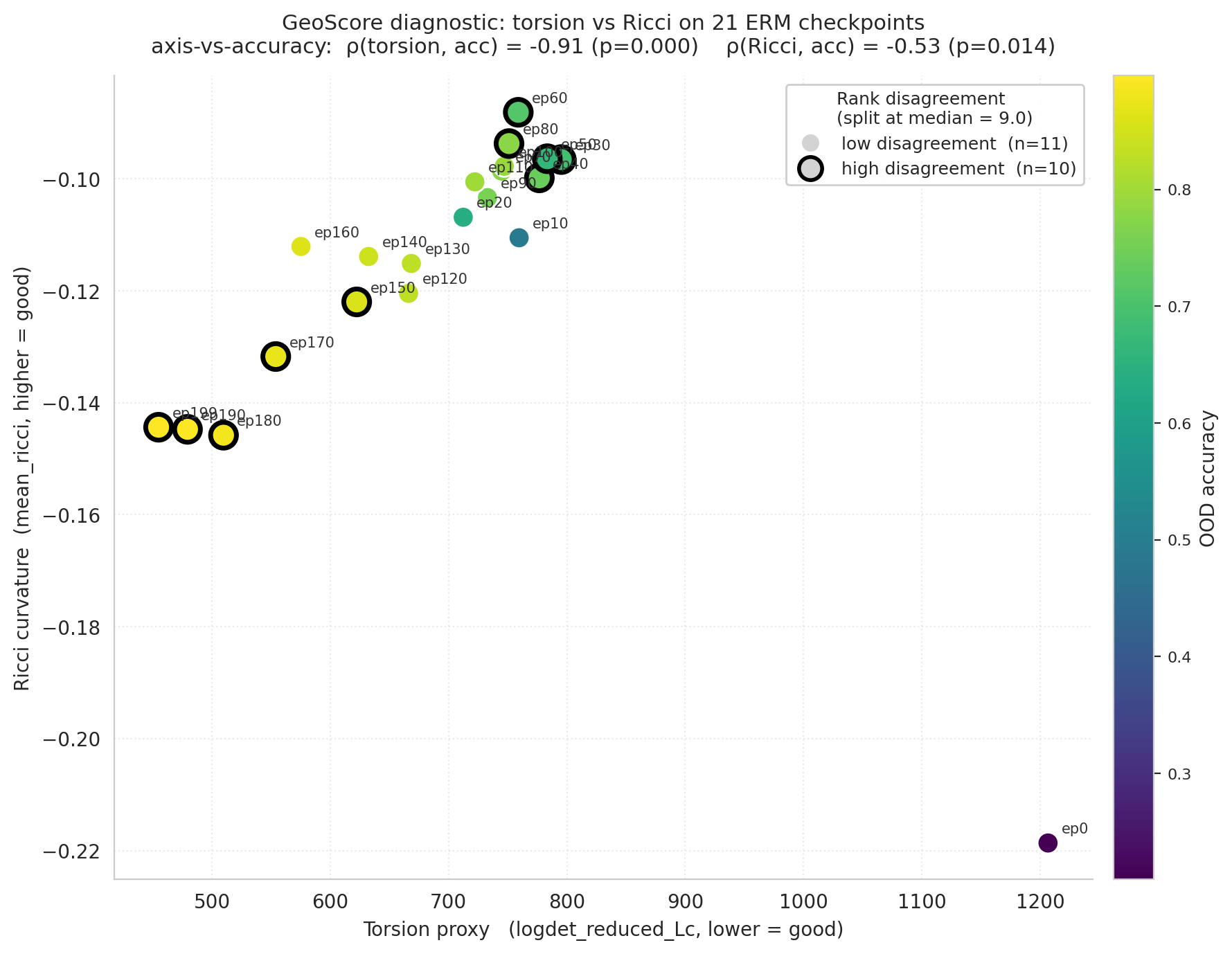}
	\caption{GeoScore diagnostic: disagreement between torsion and Ricci signals across checkpoints. Each point corresponds to a checkpoint (colored by OOD accuracy). While torsion (logdet proxy) exhibits a strong monotonic relationship with OOD accuracy ($\rho=-0.91$), Ricci curvature shows weaker and partially inconsistent alignment ($\rho=+0.53$). Points with high rank disagreement (black outlines) highlight regimes where local regularity (Ricci) and global complexity (torsion) provide conflicting signals. This discrepancy motivates combining geometric and topological descriptors within \textsc{TopoGeoScore} to resolve global--local inconsistencies.}
	\label{fig:plot_geoscore_diagnostic}
\end{figure}


\begin{figure}[h!]
    \centering
    \includegraphics[width=0.9\textwidth]{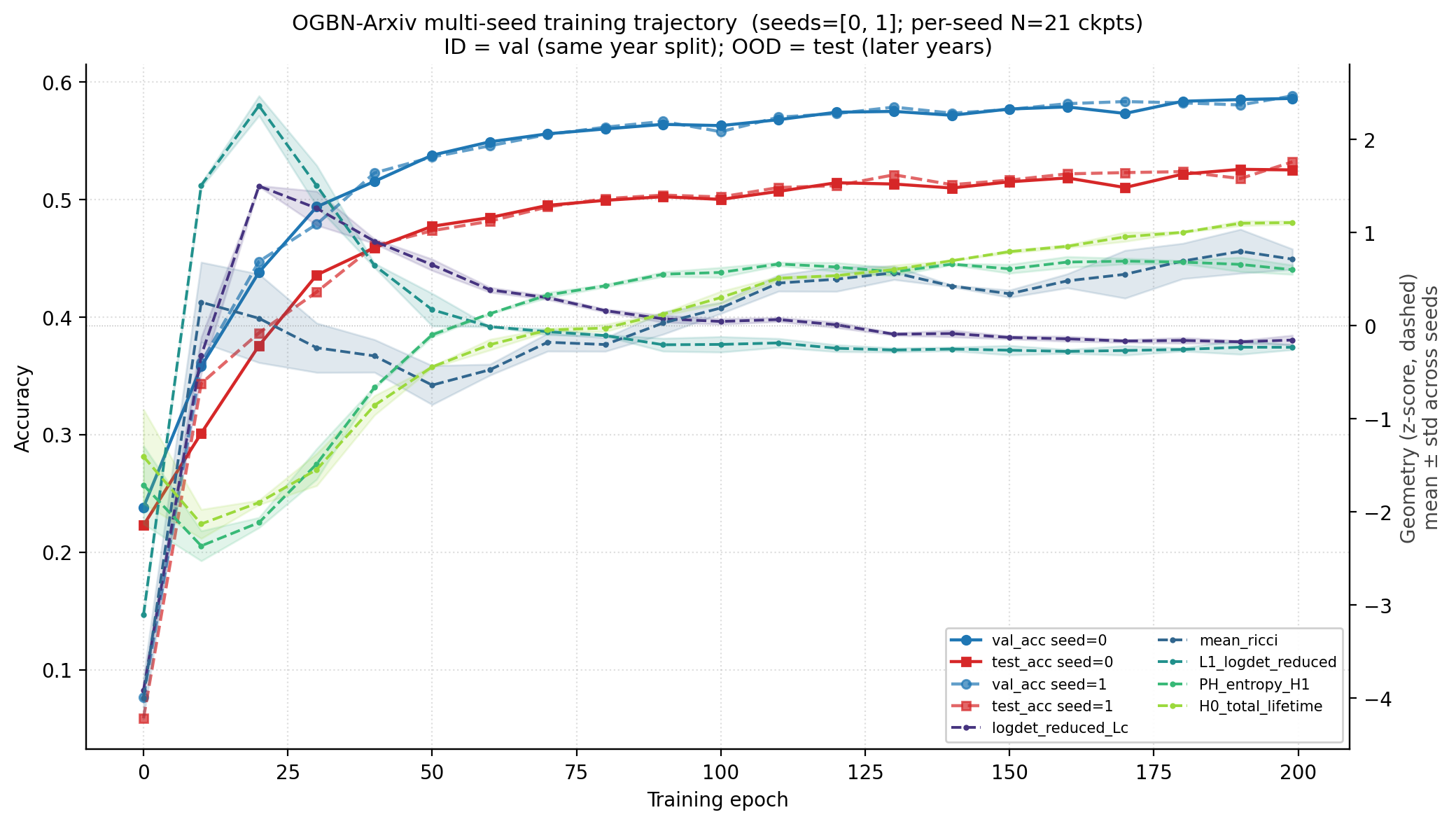}
    \caption{Multi-seed training trajectory. Two seeds are overlaid with different line styles; geometry curves are shown as mean $\pm$ standard-deviation bands across seeds. This tests whether the trajectory shape is stable across seeds.}
    \label{fig:plot_ogbn_v2_multiseed_trajectory}
\end{figure}

\begin{figure}[h!]
    \centering
    \includegraphics[width=0.9\textwidth]{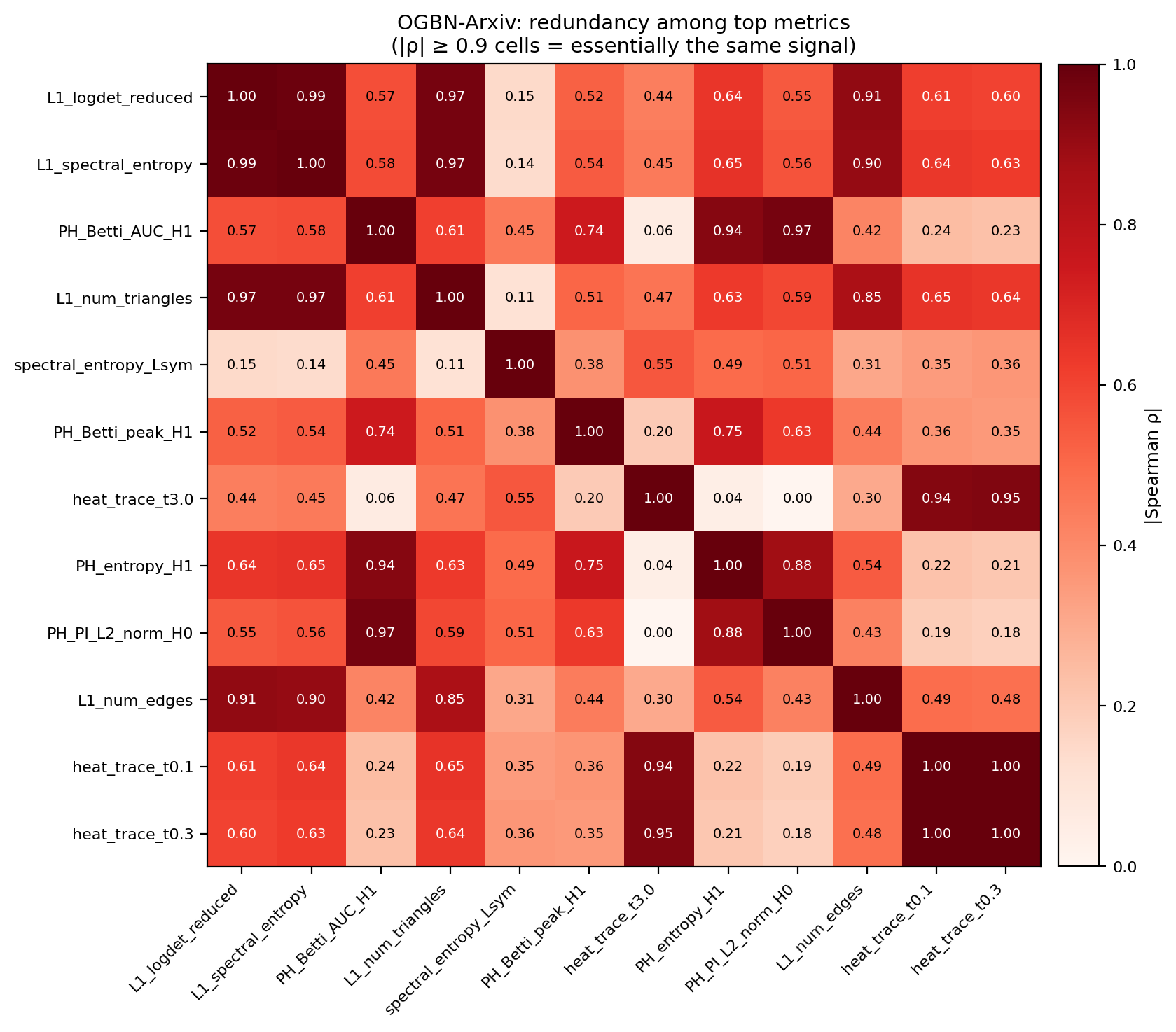}
    \caption{Matrix of absolute Spearman correlations among the top-12 metrics. Cells with $|\rho|\ge0.9$ indicate near-redundant metric signals.}
    \label{fig:plot_ogbn_v2_redundancy_heatmap}
\end{figure}

\begin{figure}[h!]
    \centering
    \includegraphics[width=0.9\textwidth]{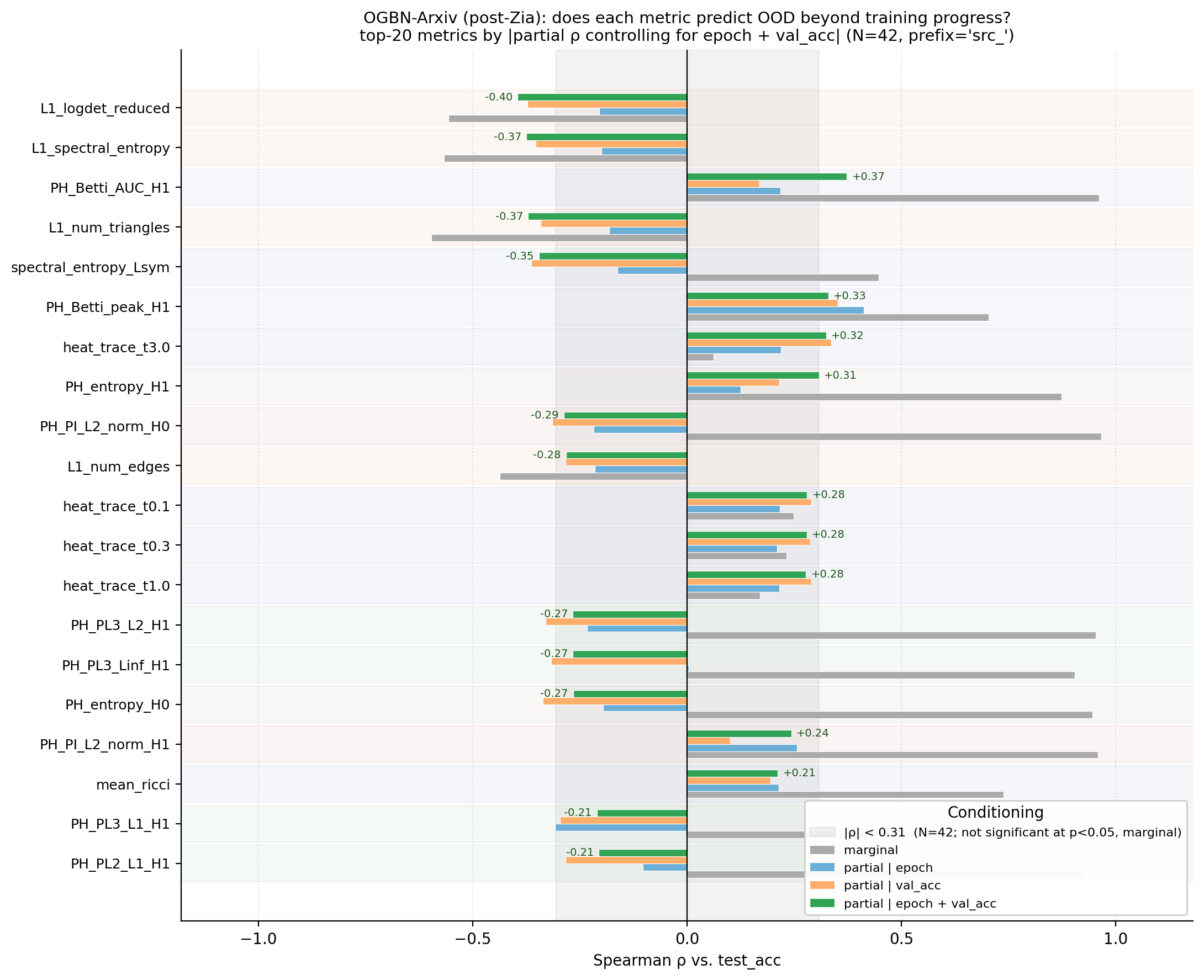}
    \caption{For each top-$20$ metric, four Spearman correlations are shown: marginal, partial controlling for epoch, partial controlling for validation accuracy, and partial controlling for both. Metrics are sorted by the strictest control condition.}
    \label{fig:plot_ogbn_v2_partial_correlations}
\end{figure}

\begin{figure}[h!]
    \centering
    \includegraphics[width=0.9\textwidth]{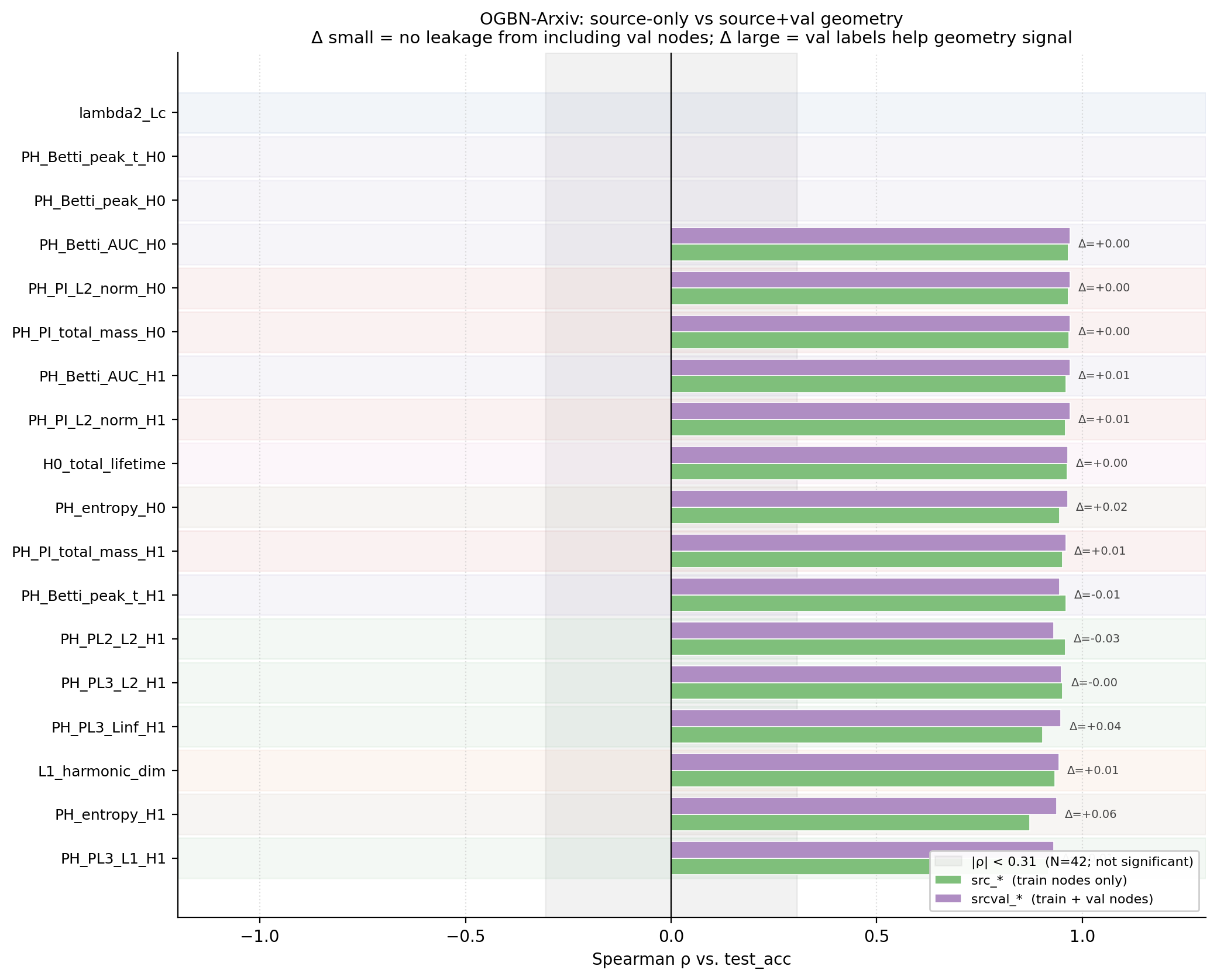}
    \caption{Per-metric paired bars: $\rho$ on \texttt{src\_*} (train nodes only) vs.\ $\rho$ on \texttt{srcval\_*} (train+val nodes), against \texttt{test\_acc}. $\Delta$ values are annotated. If $\Delta$ is small, including validation nodes does not materially change the source-only signal; if $\Delta$ is large, validation-node information affects the geometry signal.}
    \label{fig:plot_ogbn_v2_src_vs_srcval}
\end{figure}

\begin{figure}[h!]
    \centering
    \includegraphics[width=0.9\textwidth]{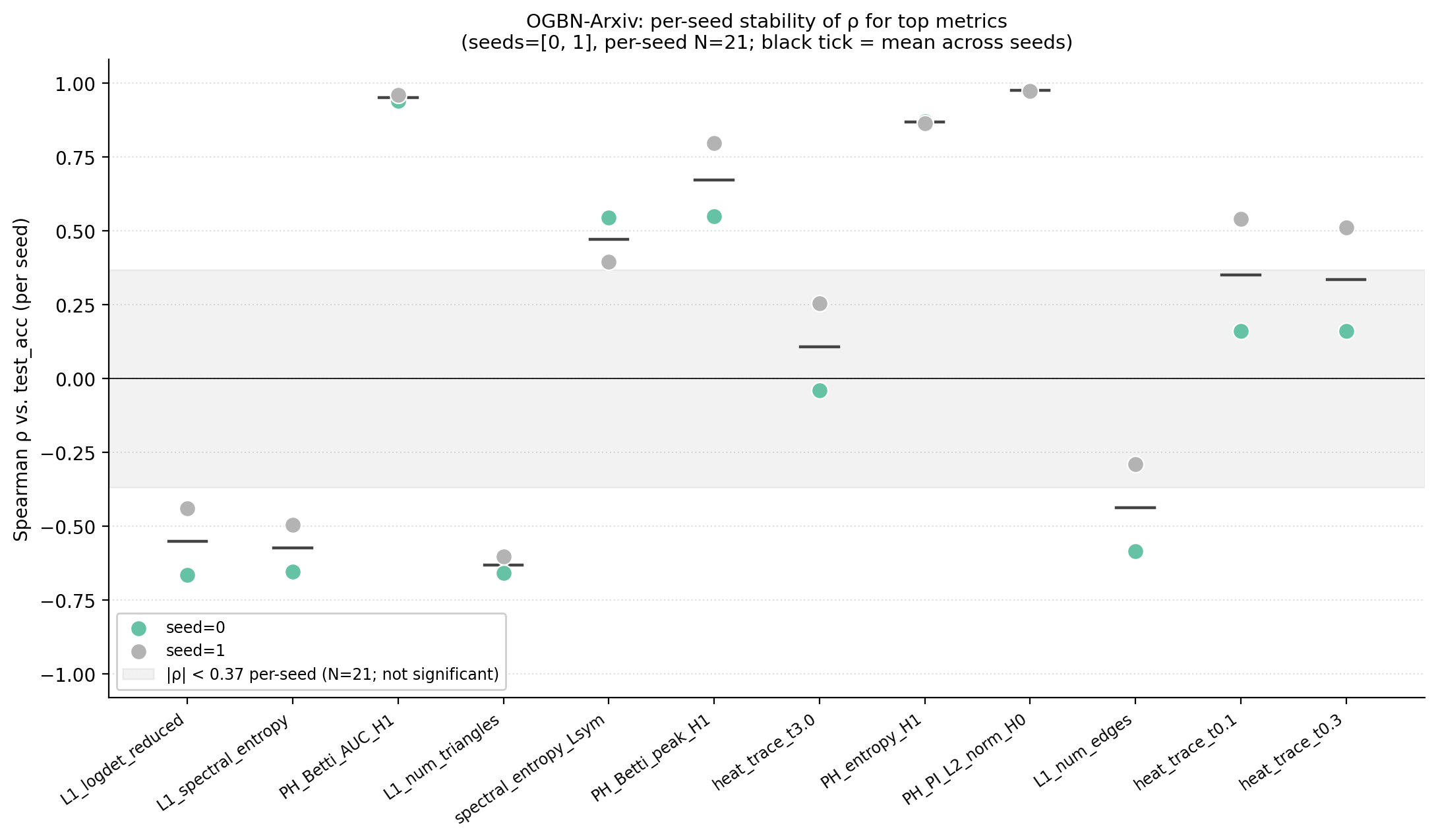}
    \caption{Per-metric scatter of per-seed Spearman correlations, with a mean line. The significance band is drawn for per-seed $N=21$ ($|\rho|>0.37$).}
    \label{fig:plot_ogbn_v2_seed_stability}
\end{figure}

\begin{figure}[h!]
	\centering
	\includegraphics[width=0.9\textwidth]{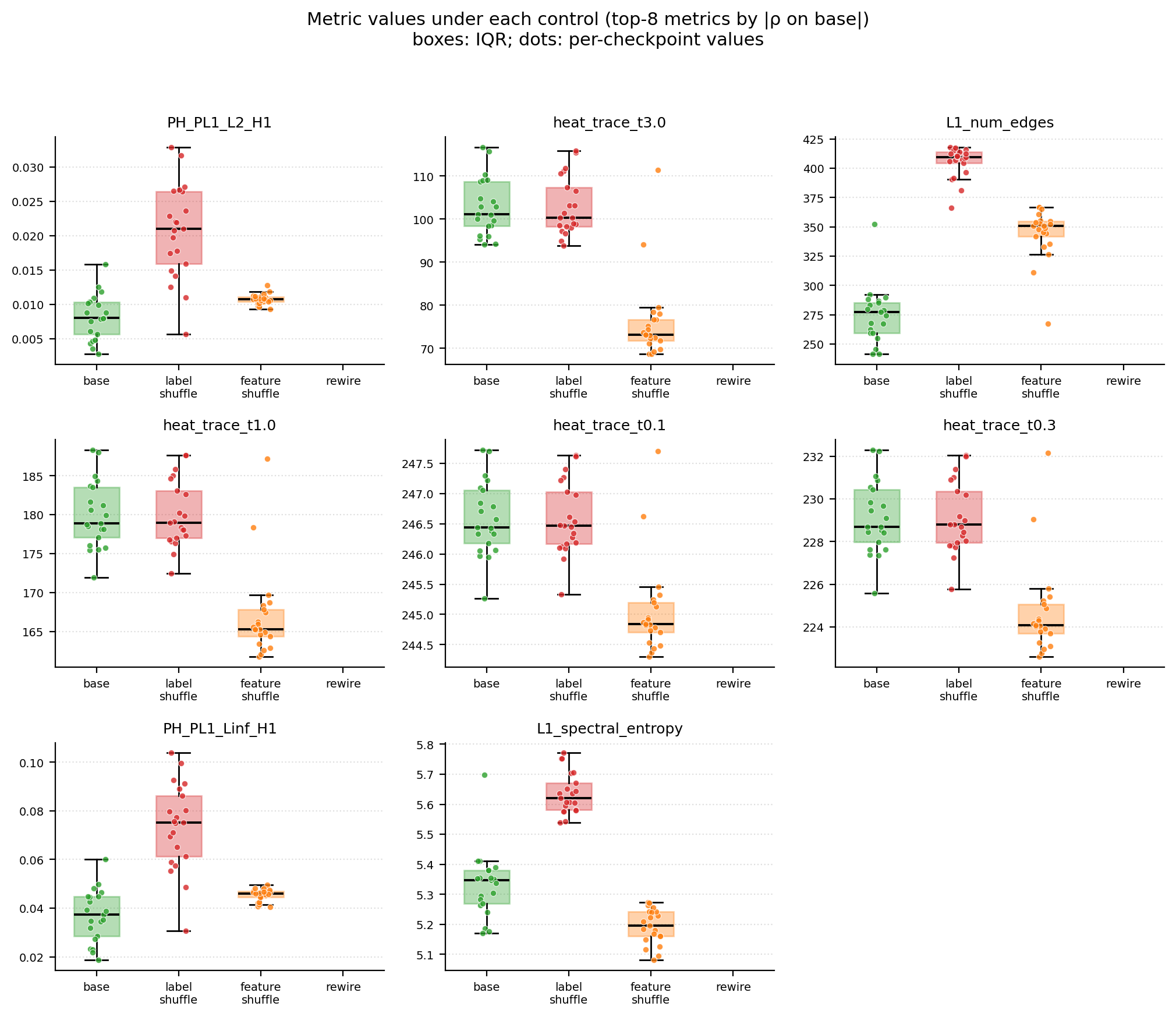}
	\caption{Metric values under structural controls for the top-$k$ metrics (ranked by $|\rho|$ on the base setting). Boxes show interquartile ranges and dots denote per-checkpoint values. Label shuffling perturbs metrics strongly, while feature shuffling generally reduces their magnitude, indicating sensitivity to label--structure and feature--structure alignment. Under full edge rewiring, several metrics become degenerate or undefined, leading to near-zero or missing values and highlighting their reliance on meaningful graph topology.}
	\label{fig:plot_controls_value_distributions}
\end{figure}

\begin{figure}[h!]
	\centering
	\includegraphics[width=0.9\textwidth]{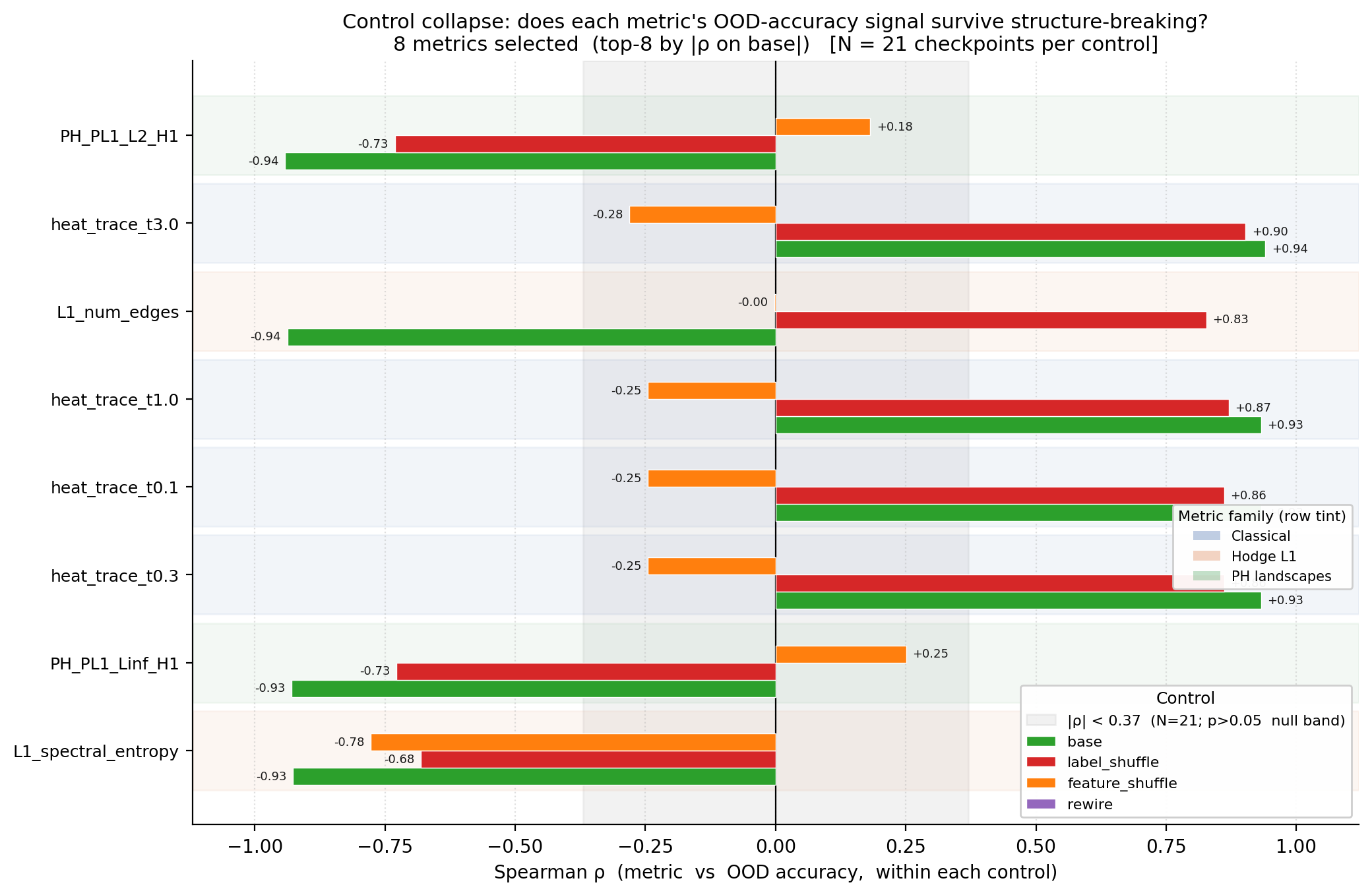}
	\caption{Control-collapse analysis of metric--OOD accuracy correlation (Spearman $\rho$) across structure-breaking interventions. Metrics are selected as the top-8 by $|\rho|$ in the base setting ($N=21$ checkpoints). While most metrics exhibit strong predictive power under the base condition, their behavior under controls reveals the source of this signal: correlations largely persist under label shuffling, indicating that metrics are not purely label-dependent, but collapse under feature shuffling and edge rewiring, demonstrating a strong reliance on feature geometry and graph topology. The shaded band indicates the non-significant region ($|\rho|<0.37$, $p>0.05$).}
	\label{fig:plot_controls_collapse}
\end{figure}

\begin{figure}[h!]
	\centering
	\includegraphics[width=0.9\textwidth]{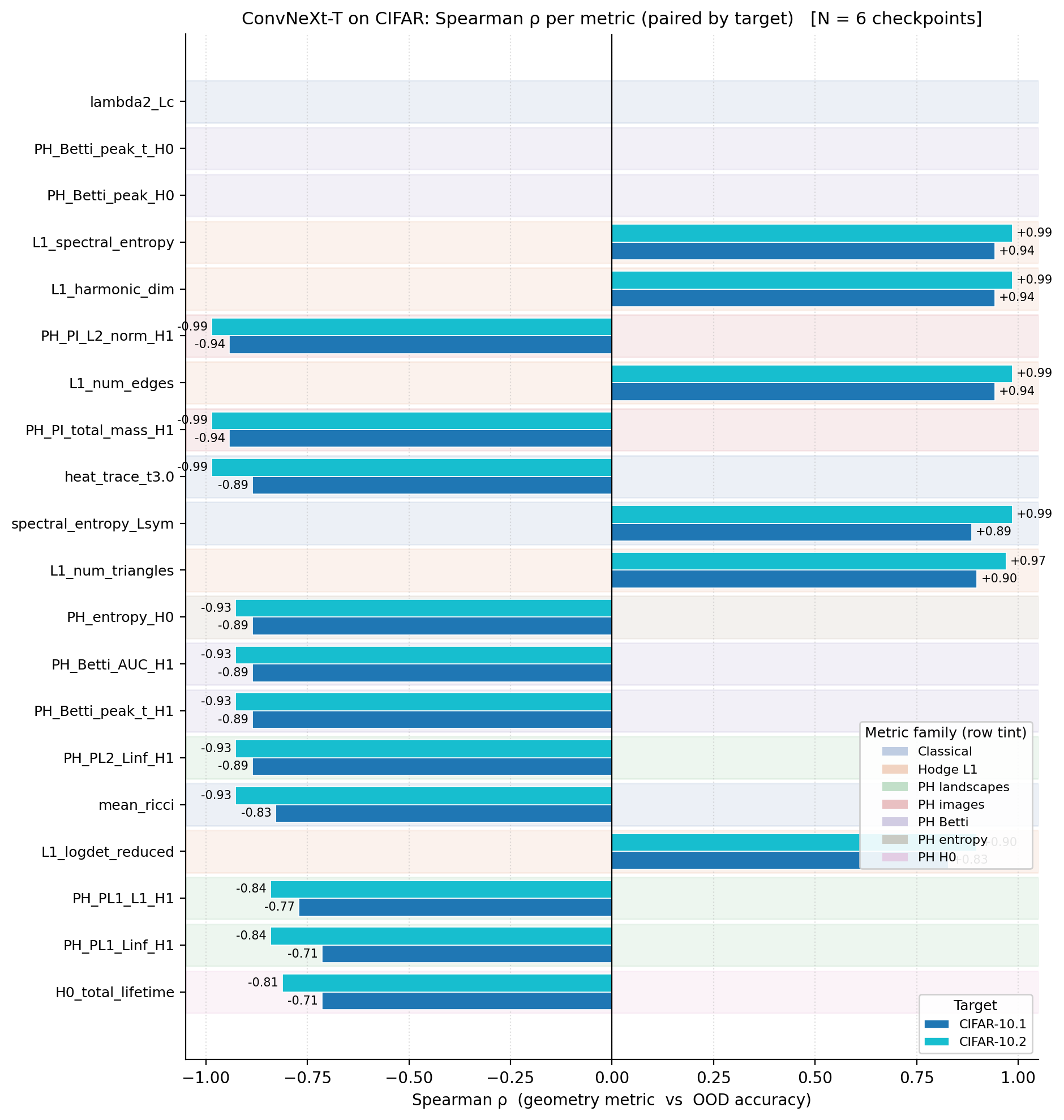}
	\caption{Spearman correlation ($\rho$) between geometry-based metrics and OOD accuracy for ConvNeXt-T on CIFAR, evaluated across two targets (CIFAR-10.1 and CIFAR-10.2) with $N=6$ checkpoints. Most metrics exhibit strong and consistent correlations across both targets, with stable sign and magnitude, indicating that the induced ranking signal is largely target-agnostic. Positive correlations are observed for spectral and structural metrics, while several PH-based metrics show strong negative correlations. These results support the view that geometry-based metrics capture intrinsic model quality rather than target-specific effects.}
	\label{fig:plot_convnext_rho_compare}
\end{figure}

\begin{figure}[h!]
	\centering
	\includegraphics[width=0.9\textwidth]{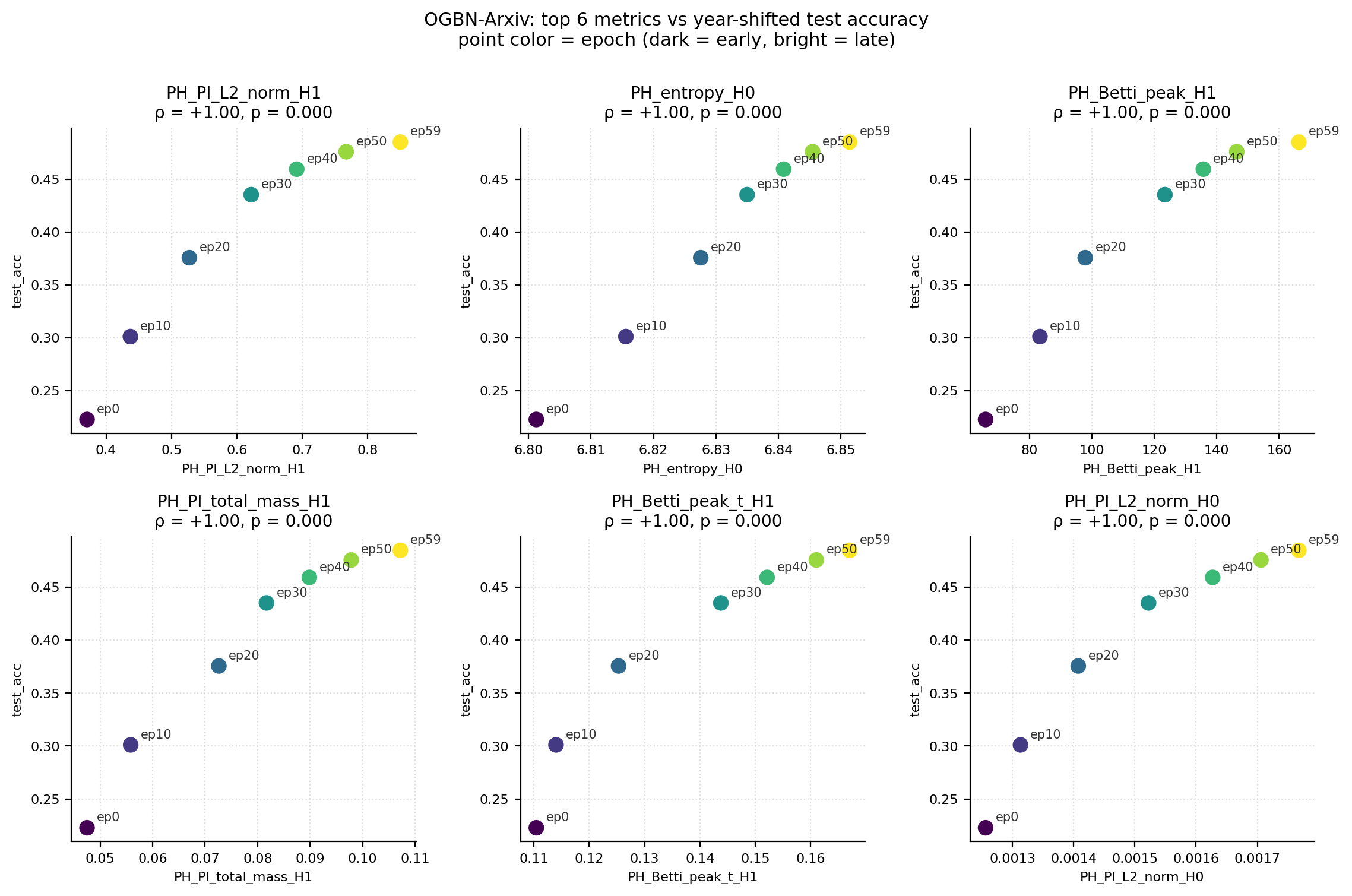}
	\caption{Per-epoch trajectory on OGBN-Arxiv (single seed). Each subplot shows a geometry-based metric versus year-shifted test accuracy, with points colored by training epoch (dark to bright). All metrics exhibit near-monotonic trends with very high Spearman correlation ($\rho \approx 0.97$--$0.98$), indicating that they track training progress and model quality within a run. As both metrics and accuracy co-evolve with epoch, these results primarily reflect shared temporal dynamics and serve as a consistency check rather than evidence for cross-model or cross-seed OOD selection.}
	\label{fig:plot_ogbn_scatter_grid}
\end{figure}

\begin{table}[h]
	\centering
	\small
	\caption{Selection performance across CIFAR-10 variants. $\rho$ denotes Spearman correlation between metric and OOD accuracy.}
	\label{tab:cifar_results}
	\begin{tabular}{llrrrrrr}
		\toprule
		Dataset & Score & $\rho$ & $p$ & $n$ & Sel. Acc & Oracle Acc & Gap \\
		\midrule
		\multicolumn{8}{l}{\textbf{CIFAR-10/10.1 (N=21 ckpts)}} \\
		& torsion\_only          & -0.913 & $7.85\mathrm{E}{-09}$ & 21 & 0.8975 & 0.8975 & 0.000 \\
		& ricci\_only            &  0.527 & $1.40\mathrm{E}{-02}$ & 21 & 0.7125 & 0.8975 & 0.185 \\
		& fixed\_GeoScore        & -0.588 & $5.28\mathrm{E}{-04}$ & 21 & 0.8625 & 0.8975 & 0.035 \\
		& topology\_aware\_fixed & -0.904 & $1.94\mathrm{E}{-08}$ & 21 & 0.8975 & 0.8975 & 0.000 \\
		& TopoGeoScore\_SSL      & -0.945 & $1.06\mathrm{E}{-10}$ & 21 & 0.8975 & 0.8975 & 0.000 \\
		\midrule
		\multicolumn{8}{l}{\textbf{CIFAR-10-C (N=105 ckpt-sev)}} \\
		& torsion\_only          & -0.629 & $6.57\mathrm{E}{-13}$ & 105 & 0.92536 & 0.92536 & 0.000 \\
		& ricci\_only            &  0.265 & $6.33\mathrm{E}{-05}$ & 105 & 0.78432 & 0.92536 & 0.141 \\
        & fixed\_GeoScore        & -0.415 & $1.10\mathrm{E}{-05}$ & 105 & 0.88842 & 0.92536 & 0.037 \\
		& topology\_aware\_fixed & -0.622 & $1.37\mathrm{E}{-12}$ & 105 & 0.92536 & 0.92536 & 0.000 \\
		& TopoGeoScore\_SSL      & -0.641 & $1.82\mathrm{E}{-13}$ & 105 & 0.92536 & 0.92536 & 0.000 \\
		\bottomrule
	\end{tabular}
\end{table}

\begin{table}[h]
	\centering
	\scriptsize
	\caption{Per-severity selection performance on CIFAR-10-C. $\rho$ denotes Spearman correlation between metric and OOD accuracy.}
	\label{tab:cifar_severity}
	\begin{tabular}{llrrrrrr}
		\toprule
		Severity & Score & $n$ & $\rho$ & $p$ & Sel. Acc & Oracle Acc & Gap \\
		\midrule
		
		\multicolumn{8}{l}{\textbf{Severity 1}} \\
		1 & torsion\_only          & 21 & -0.904 & $1.94\mathrm{E}{-08}$ & \textbf{0.92536} & 0.92536 & \textbf{0.000} \\
		1 & ricci\_only            & 21 &  0.516 & $1.67\mathrm{E}{-02}$ & 0.78432 & 0.92536 & 0.14104 \\
		1 & fixed\_GeoScore        & 21 & -0.571 & $6.81\mathrm{E}{-03}$ & 0.88842 & 0.92536 & 0.03694 \\
		1 & topology\_aware\_fixed & 21 & -0.914 & $6.84\mathrm{E}{-09}$ & \textbf{0.92536} & 0.92536 & \textbf{0.000} \\
		1 & TopoGeoScore\_SSL      & 21 & -0.956 & $1.48\mathrm{E}{-11}$ & \textbf{0.92536} & 0.92536 & \textbf{0.000} \\
		
		\midrule
		\multicolumn{8}{l}{\textbf{Severity 2}} \\
		2 & torsion\_only          & 21 & -0.912 & $8.99\mathrm{E}{-09}$ & 0.88442 & 0.88446 & \textbf{0.00004} \\
		2 & ricci\_only            & 21 &  0.540 & $1.15\mathrm{E}{-02}$ & 0.70798 & 0.88446 & 0.17648 \\
		2 & fixed\_GeoScore        & 21 & -0.534 & $1.27\mathrm{E}{-02}$ & 0.84116 & 0.88446 & 0.04330 \\
		2 & topology\_aware\_fixed & 21 & -0.890 & $6.15\mathrm{E}{-08}$ & 0.88442 & 0.88446 & \textbf{0.00004} \\
		2 & TopoGeoScore\_SSL      & 21 & -0.939 & $3.00\mathrm{E}{-10}$ & 0.88442 & 0.88446 & \textbf{0.00004} \\
		
		\midrule
		\multicolumn{8}{l}{\textbf{Severity 3}} \\
		3 & torsion\_only          & 21 & -0.921 & $3.32\mathrm{E}{-09}$ & 0.83946 & 0.84084 & \textbf{0.00138} \\
		3 & ricci\_only            & 21 &  0.553 & $9.28\mathrm{E}{-03}$ & 0.65912 & 0.84084 & 0.18172 \\
		3 & fixed\_GeoScore        & 21 & -0.522 & $1.52\mathrm{E}{-02}$ & 0.80098 & 0.84084 & 0.03986 \\
		3 & topology\_aware\_fixed & 21 & -0.894 & $4.95\mathrm{E}{-08}$ & 0.83946 & 0.84084 & \textbf{0.00138} \\
		3 & TopoGeoScore\_SSL      & 21 & -0.940 & $2.46\mathrm{E}{-10}$ & 0.83946 & 0.84084 & \textbf{0.00138} \\
		
		\midrule
		\multicolumn{8}{l}{\textbf{Severity 4}} \\
		4 & torsion\_only          & 21 & -0.934 & $6.38\mathrm{E}{-10}$ & 0.79974 & 0.80178 & \textbf{0.00204} \\
		4 & ricci\_only            & 21 &  0.558 & $8.51\mathrm{E}{-03}$ & 0.59806 & 0.80178 & 0.20372 \\
		4 & fixed\_GeoScore        & 21 & -0.529 & $1.38\mathrm{E}{-02}$ & 0.75978 & 0.80178 & 0.04200 \\
		4 & topology\_aware\_fixed & 21 & -0.881 & $1.40\mathrm{E}{-07}$ & 0.79974 & 0.80178 & \textbf{0.00204} \\
		4 & TopoGeoScore\_SSL      & 21 & -0.925 & $2.09\mathrm{E}{-09}$ & 0.79974 & 0.80178 & \textbf{0.00204} \\
		
		\midrule
		\multicolumn{8}{l}{\textbf{Severity 5}} \\
		5 & torsion\_only          & 21 & -0.905 & $1.72\mathrm{E}{-08}$ & 0.70546 & 0.70968 & \textbf{0.00422} \\
		5 & ricci\_only            & 21 &  0.518 & $1.61\mathrm{E}{-02}$ & 0.48626 & 0.70968 & 0.22342 \\
		5 & fixed\_GeoScore        & 21 & -0.505 & $1.95\mathrm{E}{-02}$ & 0.66994 & 0.70968 & 0.03974 \\
		5 & topology\_aware\_fixed & 21 & -0.834 & $2.67\mathrm{E}{-06}$ & 0.70546 & 0.70968 & \textbf{0.00422} \\
		5 & TopoGeoScore\_SSL      & 21 & -0.862 & $5.00\mathrm{E}{-07}$ & 0.70546 & 0.70968 & \textbf{0.00422} \\
		
		\bottomrule
	\end{tabular}
\end{table}

\subsection{Graph modality: OGBN-Arxiv}
\label{app:ogbn}

To assess whether the framework carries to a non-image, non-language 
modality with a natural temporal distribution shift, the same 
diagnostic pipeline is applied to a node-classification setting on 
OGBN-Arxiv \cite{hu2020ogb}. Two GCN seeds are trained, with 
checkpoints saved every $10$ epochs ($N\!=\!21$ per seed). For each 
checkpoint, geometry is computed under two label-free settings on the 
embedding cloud: \emph{source-only} (training nodes only) and 
\emph{source$+$val} (training and in-distribution validation nodes 
only); test nodes never enter the diagnostic. The OOD target is the 
year-shifted test accuracy. Because every metric tends to drift 
monotonically with training, marginal Spearman correlations are 
inflated by training progress alone.
Predictivity is therefore reported through \emph{partial} correlations 
that control for the training epoch and the in-distribution 
validation accuracy.

\begin{table}[h]
\centering
\small
\caption{\textbf{OGBN-Arxiv (graph modality), source-only geometry.} 
For each metric family, the strongest representative under the 
strictest partial-correlation control (epoch + val accuracy) is 
reported. Marginal correlations are large for nearly all metrics 
because of training-progress confounding; partial correlations 
isolate the residual source-only signal. Higher-order topological 
summaries retain a stable residual signal at the \emph{family} level; 
they should be interpreted as one persistent-homology family signal 
rather than as several independent metrics 
(per-metric redundancy heatmap in Fig.~\ref{fig:plot_ogbn_v2_redundancy_heatmap}).}
\label{tab:ogbn-main}
\begin{tabular}{@{}lcccccc@{}}
\toprule
& \multicolumn{4}{c}{\textbf{Source-only} (training nodes)} 
& \multicolumn{2}{c}{\textbf{Source$+$val} comparison} \\
\cmidrule(lr){2-5}\cmidrule(lr){6-7}
Metric family / representative 
& $\rho$ marg. 
& $\rho_{\mid \mathrm{epoch}}$ 
& $\rho_{\mid \mathrm{val}}$ 
& $\rho_{\mid \mathrm{epoch{+}val}}$ 
& $\rho_{\mid \mathrm{epoch{+}val}}$ 
& $\Delta$ \\
\midrule
Classical (heat\_trace\_t3.0)        & $+0.05$ & $+0.21$ & $+0.30$ & $+0.32$ & $+0.32$ & $+0.00$ \\
Hodge $L_1$ (L1\_logdet\_reduced)    & $-0.55$ & $-0.28$ & $-0.36$ & $-0.40$ & $-0.40$ & $+0.00$ \\
PH H$_0$ (PH\_PI\_L2\_norm\_H0)      & $+0.96$ & $-0.21$ & $-0.28$ & $-0.29$ & $-0.29$ & $+0.00$ \\
PH H$_1$ (PH\_Betti\_AUC\_H1)        & $+0.94$ & $+0.21$ & $+0.22$ & $+0.37$ & $+0.38$ & $+0.01$ \\
\bottomrule
\end{tabular}
\end{table}

\paragraph{Reading the table.}
Three observations follow. First, marginal correlations are uniformly 
large because all monotone-with-training metrics correlate with a 
test accuracy that itself rises with training. Second, after 
partialling out the training epoch and the in-distribution validation 
accuracy, the persistent-homology family retains the most stable 
residual source-only signal, while several classical Laplacian 
metrics show non-monotonic training trajectories on OGBN whose 
residual rank correlation is weaker. We do \emph{not} read this as evidence that PH is uniformly superior on graph data; rather, it indicates that the topology-mediator family retains the most stable residual signal under the strictest source-only control on this temporal shift. Individual within-family metrics should therefore be treated as redundant proxies for one signal rather than as independent discoveries. Third, the source-only and source$+$val columns agree closely ($\Delta$ small), supporting the target-label-free claim: including in-distribution validation nodes in the geometry does not materially change the ranking.

\paragraph{Scope of the OGBN claim.}
This setting reports two GCN seeds with $21$ checkpoints each. We 
present it as graph-modality \emph{evidence at the family level}, not 
as a stand-alone metric ranking. The full per-metric partial correlations, source-only versus source$+$val agreement, redundancy heatmap among top-ranked PH metrics, per-seed stability scatter, and multi-seed training trajectory are shown in Figs.~\ref{fig:plot_ogbn_v2_partial_correlations}--\ref{fig:plot_ogbn_v2_seed_stability} and Fig.~\ref{fig:plot_ogbn_v2_multiseed_trajectory}.


\clearpage

\section{Theoretical Results and Proofs}
\label{app:theory_proofs}

This appendix records the mathematical facts used and referenced throughout the text.

\subsection{Auxiliary results}

\begin{lemma}[weighted Matrix--Tree theorem]\label{lem:mtt} For
a connected weighted graph with combinatorial Laplacian $L$, every
principal cofactor of $L$ equals 
\[
\tau_{{\rm tree}}(G)=\sum_{T\in\mathsf{Tree}(G)}\prod_{e\in T}w_{e}.
\]
Moreover, 
\[
\det{}^{*}L=n\tau_{{\rm tree}}(G).
\]
\end{lemma}
\begin{proof}
The cofactor identity is the weighted Matrix--Tree theorem \citep{chung1997spectral,lyons2005determinantal}. Since
$L$ is symmetric positive semidefinite with nullspace $\operatorname{span}\{\mathbf{1}\}$,
its adjugate is 
\[
\operatorname{adj}(L)=\frac{\det{}^{*}L}{\|\mathbf{1}\|_{2}^{2}}\mathbf{1}\mathbf{1}^{\top}=\frac{\det{}^{*}L}{n}\mathbf{1}\mathbf{1}^{\top}.
\]
The diagonal entries of $\operatorname{adj}(L)$ are the principal
cofactors, each equal to $\tau_{{\rm tree}}(G)$. Hence $\det{}^{*}L/n=\tau_{{\rm tree}}(G)$. 
\end{proof}
\begin{lemma}[rank-one determinant update]\label{lem:det-update}
If $A$ is positive definite and $u$ is a vector, then 
\[
\det(A+uu^{\top})=\det(A)(1+u^{\top}A^{-1}u).
\]
The same identity applies to connected graph Laplacians after restricting
to $\mathbf{1}^{\perp}$, for update vectors in $\mathbf{1}^{\perp}$.
\end{lemma}
\begin{proof}
For positive definite $A$, 
\[
A+uu^{\top}=A^{1/2}(I+A^{-1/2}uu^{\top}A^{-1/2})A^{1/2}.
\]
The middle factor has determinant $1+\|A^{-1/2}u\|_{2}^{2}=1+u^{\top}A^{-1}u$.
For a connected Laplacian, the restriction to $\mathbf{1}^{\perp}$ is positive
definite and the pseudo-determinant is the determinant of this restriction. 
\end{proof}
\begin{lemma}[finite Hodge decomposition]\label{lem:hodge-decomp}
For a finite oriented simplicial complex with boundary maps $B_{q}$,
the $q$th Hodge Laplacian 
\[
L_{q}=B_{q}^{\top}B_{q}+B_{q+1}B_{q+1}^{\top}
\]
satisfies 
\[
C_{q}=\operatorname{im} B_{q}^{\top}\oplus\ker L_{q}\oplus\operatorname{im} B_{q+1},\qquad\ker L_{q}\cong H_{q}(K;\mathbb{R}).
\]
\end{lemma}
\begin{proof}
The proof is standard finite-dimensional linear algebra \citep{eckmann1944harmonische,schaub2021signal}. The identities
$B_{q}B_{q+1}=0$ imply $\operatorname{im} B_{q+1}\subseteq\ker B_{q}$ and $\operatorname{im} B_{q}^{\top}\perp\ker B_{q}$.
Also 
\[
x^{\top}L_{q}x=\|B_{q}x\|_{2}^{2}+\|B_{q+1}^{\top}x\|_{2}^{2},
\]
so $\ker L_{q}=\ker B_{q}\cap\ker B_{q+1}^{\top}$. Decomposing $\ker B_{q}$
orthogonally into $\operatorname{im} B_{q+1}$ and its orthogonal complement gives
the harmonic representative of each homology class. 
\end{proof}
\begin{lemma}[Weyl perturbation and log pseudo-determinant]\label{lem:logdet-stability}
Let $A$ and $A'$ be symmetric positive semidefinite matrices with
the same kernel dimension. Suppose the positive eigenvalues of $A$
are at least $a>0$ and 
\[
\|A'-A\|_{2}\le a/2.
\]
Then, with $p=\operatorname{rank}(A)$, 
\[
\left|\log\det{}^{*}A'-\log\det{}^{*}A\right|\le\frac{2p}{a}\|A'-A\|_{2}.
\]
\end{lemma}
\begin{proof}
Weyl's inequality \citep{bhatia1997matrix} gives $|\lambda_{i}(A')-\lambda_{i}(A)|\le\|A'-A\|_{2}$
after ordering eigenvalues. The positive eigenvalues of $A'$ are
at least $a/2$. Since $\log$ is $2/a$-Lipschitz on $[a/2,\infty)$,
\[
\sum_{i=1}^{p}|\log\lambda_{i}(A')-\log\lambda_{i}(A)|\le\frac{2p}{a}\|A'-A\|_{2}.
\]
\end{proof}
\begin{lemma}[Vietoris--Rips stability under distance-matrix perturbations]\label{lem:vr-stability}
Let two finite metric spaces on the same index set have distance matrices
$D$ and $D'$ satisfying 
\[
\|D-D'\|_{\infty}\le\eta.
\]
Then their Vietoris--Rips persistence modules are $\eta$-interleaved
when the filtration parameter is interpoint distance, and 
\[
d_{B}(\operatorname{Dgm}_{q}(D),\operatorname{Dgm}_{q}(D'))\le\eta
\]
for every homological degree $q$. \end{lemma}
\begin{proof}
For every scale $\alpha$, any simplex present in $\operatorname{VR}(D;\alpha)$
is present in $\operatorname{VR}(D';\alpha+\eta)$, and conversely. These inclusions
define an $\eta$-interleaving of the two filtrations. The algebraic
stability theorem for persistence modules \citep{cohensteiner2007stability,chazal2009proximity} then gives the bottleneck
bound. 
\end{proof}
\begin{lemma}[Rademacher uniform convergence template]\label{lem:rademacher-template}
Let $\mathcal{F}$ be a class of functions bounded in $[0,M]$. For
$N\ge1$ and $0<\delta<1$, with probability at least $1-\delta$ over
$N$ i.i.d. samples, 
\[
\sup_{f\in\mathcal{F}}\left(\mathrm{E} f-\frac{1}{N}\sum_{i=1}^{N}f(X_{i})\right)\le2\mathfrak{R}_{N}(\mathcal{F})+M\sqrt{\frac{\log(2/\delta)}{2N}}.
\]
For linear functions $x\mapsto w^{\top}x$ with $\|w\|_{2}\le R$
and $\|x\|_{2}\le B$, the Rademacher complexity is at most $RB/\sqrt{N}$.
\end{lemma}
\begin{proof}
This is the standard symmetrization and concentration bound for empirical
processes \citep{bartlett2002rademacher,shalev2014understanding}. The linear-class estimate follows from Cauchy--Schwarz:
\[
\mathrm{E}_{\sigma}\sup_{\|w\|\le R}\frac{1}{N}\sum_{i=1}^{N}\sigma_{i}w^{\top}x_{i}=\frac{R}{N}\mathrm{E}_{\sigma}\left\Vert \sum_{i=1}^{N}\sigma_{i}x_{i}\right\Vert _{2}\le\frac{R}{N}\left(\sum_{i=1}^{N}\|x_{i}\|_{2}^{2}\right)^{1/2}\le\frac{RB}{\sqrt{N}}.
\]
\end{proof}

\subsection{Distance-distortion assumptions}

For a point cloud $Z=\{z_{i}\}_{i=1}^{n}$, write 
\[
\delta_{ij}=\|z_{i}-z_{j}\|_{2}.
\]
Assume $1\le k<n-1$, so that each directed $k$NN list and its
non-self complement are nonempty. Let $N_{k}(i)$ be a fixed tie-broken
directed $k$NN list for $i$, computed among the points other than $i$.
Define the directed $k$NN margin 
\[
\Gamma_{k}(Z)=\min_{i}\left(\min_{\ell\notin N_{k}(i),\,\ell\neq i}\delta_{i\ell}-\max_{j\in N_{k}(i)}\delta_{ij}\right).
\]
The margin is positive exactly when every selected neighbor is separated
by a nonzero gap from every non-selected non-self point.

\begin{assumption}[small pairwise distortion]\label{ass:distortion}
A transformed cloud $Z'=\{z_{i}'\}_{i=1}^{n}$ satisfies 
\[
\max_{i,j}|\delta_{ij}'-\delta_{ij}|\le\eta,\qquad\delta_{ij}'=\|z_{i}'-z_{j}'\|_{2}.
\]
\end{assumption}

\begin{assumption}[self-tuning weights and nondegeneracy]\label{ass:nondegenerate}
The self-tuning bandwidths are 
\[
\sigma_{i}=\max_{j\in N_{k}(i)}\delta_{ij},
\]
and the edge weights are 
\[
w_{ij}=\exp\left(-\frac{\delta_{ij}^{2}}{\sigma_{i}\sigma_{j}}\right)\quad\text{on mutual \ensuremath{k}NN edges.}
\]
Assume 
\[
\sigma_{i}\ge\sigma_{\min}>0,\qquad0<d_{\min}\le d_{i}\le d_{\max}<\infty,
\]
for all vertices in the original graph. \end{assumption}

\subsection{Formal statements referenced in the main text}

\begin{proposition}[Normalized Matrix--Tree identity]\label{prop:normalized-mtt}
Let $G=(V,E,w)$ be a weighted graph whose vertices all have positive degrees, and let $\mathcal{L}$ be its normalized Laplacian. In the connected case, denote by
\[
\tau_{\mathrm{tree}}(G)=\sum_{T\in\mathsf{Tree}(G)}\prod_{e\in T}w_e
\]
the weighted spanning-tree partition function and let $\mathrm{vol}(G)=\sum_{v\in V}d_v$. Then
\[
\det{}^*(\mathcal{L})=\frac{\mathrm{vol}(G)\tau_{\mathrm{tree}}(G)}{\prod_{v\in V}d_v}.
\]
Equivalently,
\[
\log\det{}^*(\mathcal{L})=\log\tau_{\mathrm{tree}}(G)+\log\mathrm{vol}(G)-\sum_{v\in V}\log d_v .
\]
If $G$ has connected components $G_1,\ldots,G_r$, then the identity applies componentwise and
\[
\det{}^*(\mathcal{L}_G)=
\prod_{a=1}^{r}
\frac{\mathrm{vol}(G_a)\tau_{\mathrm{tree}}(G_a)}
{\prod_{v\in V(G_a)}d_v}.
\]
\end{proposition}

\begin{proposition}[Hodge kernel and first homology]\label{prop:hodge-cycles}
Let $K$ be a finite oriented simplicial complex, and let $L_1=B_1^\top B_1+B_2B_2^\top$ be its unweighted first Hodge Laplacian, where $B_1:C_1(K;\mathbb{R})\to C_0(K;\mathbb{R})$ and $B_2:C_2(K;\mathbb{R})\to C_1(K;\mathbb{R})$ are the oriented boundary matrices. Then
\[
\ker L_1=\ker B_1\cap\ker B_2^\top\cong H_1(K;\mathbb{R}).
\]
Consequently, $\dim\ker L_1=\beta_1(K)$, and, for a graph flag complex truncated at dimension two,
\[
\beta_1(K)=|E|-|V|+\beta_0(K)-\operatorname{rank}(B_2).
\]
\end{proposition}

\begin{theorem}[positive-view invariance and stability]\label{thm:view-stability}
Let $Z$ and $Z'$ satisfy Assumptions~\ref{ass:distortion}--\ref{ass:nondegenerate}. Suppose
\[
2\eta<\Gamma_k(Z),\qquad \eta\le\sigma_{\min}/2,\qquad kC_w\eta\le d_{\min}/2,
\]
where, with $D_\eta=\max_{(i,j)\in E}\delta_{ij}+\eta$,
\[
C_w=\frac{8D_\eta}{\sigma_{\min}^2}+\frac{16D_\eta^2}{\sigma_{\min}^3}.
\]
Then:
\begin{enumerate}[label=(\alph*)]
\item Every directed $k$NN list is unchanged, hence the mutual $k$NN edge set is unchanged.
\item The self-tuning weights satisfy $\max_{(i,j)\in E}|w_{ij}'-w_{ij}|\le C_w\eta$.
\item The normalized graph Laplacians satisfy
\[
\|\mathcal{L}(Z')-\mathcal{L}(Z)\|_2\le C_{\mathcal{L}}\eta,\qquad
C_{\mathcal{L}}=kC_w\left(\frac{2}{d_{\min}}+\frac{(2+\sqrt{2})d_{\max}}{d_{\min}^2}\right).
\]
\item If $a_0>0$ is the smallest positive eigenvalue of $\mathcal{L}(Z)$, $p_0=\operatorname{rank}\mathcal{L}(Z)$, and $C_{\mathcal{L}}\eta\le a_0/2$, then
\[
\bigl|\log\det{}^{*}\mathcal{L}(Z')-\log\det{}^{*}\mathcal{L}(Z)\bigr|
\le \frac{2p_0C_{\mathcal{L}}}{a_0}\eta .
\]
\item If the Hodge complex is the unweighted flag complex of the mutual $k$NN graph, then the flag complex, $L_1$, $\beta_1$, and the unweighted Hodge spectrum are exactly unchanged. More generally, if a weighted Hodge construction has the same kernel dimension for $Z$ and $Z'$ and satisfies $\|L_1(Z')-L_1(Z)\|_2\le C_1\eta$, and if the smallest positive eigenvalue of $L_1(Z)$ is $a_1>0$ with $p_1=\operatorname{rank}L_1(Z)$ and $C_1\eta\le a_1/2$, then
\[
\bigl|\log\det{}^{*}L_1(Z')-\log\det{}^{*}L_1(Z)\bigr|
\le \frac{2p_1C_1}{a_1}\eta .
\]
\item For Vietoris--Rips persistence built from the pairwise distance matrices,
\[
d_B\!\left(\operatorname{Dgm}_q(Z),\operatorname{Dgm}_q(Z')\right)\le \eta,\qquad q=0,1,
\]
when the filtration parameter is interpoint distance. If the filtration parameter is a radius rather than an interpoint-distance threshold, the bound is rescaled by the corresponding deterministic change of parameter. Any vectorization map $V$ that is $L_V$-Lipschitz with respect to bottleneck distance satisfies
\[
\|V(\operatorname{Dgm}_q(Z))-V(\operatorname{Dgm}_q(Z'))\|\le L_V\eta .
\]
\end{enumerate}
\end{theorem}

\begin{corollary}[exact and approximate positive-view invariance]\label{cor:positive-views}
If $Z'=ZQ$ for an orthogonal matrix $Q$, then all pairwise distances are identical. Hence the mutual $k$NN graph, self-tuning weights, normalized Laplacian spectrum, torsion proxy, Ollivier--Ricci curvatures, unweighted Hodge features, and Vietoris--Rips persistence diagrams are exactly invariant.

If a transformation, such as mild noise, feature masking followed by renormalization, or PCA-preserving projection, satisfies Assumption~\ref{ass:distortion} with $\eta$ below the margin and spectral-gap thresholds in Theorem~\ref{thm:view-stability}, then the quantities covered by Theorem~\ref{thm:view-stability} are stable with the displayed bounds.
\end{corollary}

For the self-supervised objective statements, let $G_i$ denote an anchor graph, $G_{ip}^+$ a positive view, and $G_{iq}^-$ a negative view. Write
\[
g_i=g(G_i),\qquad g_{ip}^+=g(G_{ip}^+),\qquad g_{iq}^-=g(G_{iq}^-),
\]
and define $a_{ip}=g_i-g_{ip}^+$ and $b_{iq}=g_{iq}^--g_i$. For $M_+$ positive anchor-view pairs and $M_-$ negative anchor-view pairs, define
\[
\widehat{\mathcal{J}}(w)
=
\frac{1}{M_+}\sum_{i,p}(w^\top a_{ip})^2
+
\frac{\lambda}{M_-}\sum_{i,q}[m-w^\top b_{iq}]_+
+
\mu\|w\|_2^2,
\qquad w\in\mathbb{R}_{\ge0}^{5}.
\]
The population objective $\mathcal{J}$ is defined analogously by replacing empirical averages over positive and negative pairs by expectations under $\mathcal{P}_+$ and $\mathcal{P}_-$.

\begin{theorem}[convex monotone separator]\label{thm:ssl-convex}
Assume $m>0$, $\lambda\ge0$, and $\mu>0$. Then:
\begin{enumerate}[label=(\alph*)]
\item $\widehat{\mathcal{J}}$ is convex on $\mathbb{R}_{\ge0}^{5}$ and is $2\mu$-strongly convex in the Euclidean norm. Hence it has a unique minimizer over $\mathbb{R}_{\ge0}^{5}$.
\item Every subgradient KKT point is globally optimal.
\item The score is coordinatewise monotone: if $w\ge0$ and $g'\ge g$ coordinatewise, then $s_w(g')\ge s_w(g)$.
\item Suppose there is a vector $u\in\mathbb{R}_{\ge0}^{5}$ with $\|u\|_2=1$ and constants $\varepsilon_+\ge0$, $\gamma_->0$ such that $|u^\top(g_{ip}^+-g_i)|\le\varepsilon_+$ for all positive views and $u^\top(g_{iq}^--g_i)\ge\gamma_-$ for all negative views. Then $w=(m/\gamma_-)u$ has zero empirical separation loss and invariance loss at most $m^2\varepsilon_+^2/\gamma_-^2$. Consequently,
\[
\widehat{\mathcal{J}}(w^\star)\le \frac{m^2\varepsilon_+^2}{\gamma_-^2}+\frac{\mu m^2}{\gamma_-^2}.
\]
\end{enumerate}
\end{theorem}

\begin{theorem}[uniform finite-sample control]\label{thm:ssl-generalization}
Assume that $m\ge0$, that the $M_+\ge1$ positive empirical pairs are i.i.d. samples from $\mathcal{P}_+$, that the $M_-\ge1$ negative empirical pairs are i.i.d. samples from $\mathcal{P}_-$, and that the two samples are independent. Assume also that $\|g(G)\|_2\le B$ for every anchor and view, and restrict weights to
\[
\mathcal{W}_R=\{w\in\mathbb{R}_{\ge0}^{5}:\|w\|_2\le R\}.
\]
Then, for any $0<\delta<1$, with probability at least $1-\delta$, simultaneously for all $w\in\mathcal{W}_R$,
\[
\mathcal{J}(w)\le\widehat{\mathcal{J}}(w)+\epsilon_+(M_+,B,R,\delta)+\lambda\epsilon_-(M_-,B,R,m,\delta),
\]
where one admissible choice is
\[
\epsilon_+=\frac{16B^2R^2}{\sqrt{M_+}}+4B^2R^2\sqrt{\frac{\log(4/\delta)}{2M_+}},
\]
and
\[
\epsilon_-=\frac{4BR}{\sqrt{M_-}}+(m+2BR)\sqrt{\frac{\log(4/\delta)}{2M_-}}.
\]
\end{theorem}

\subsection{Proofs of the main theoretical statements}
\begin{proof}[Proof of Proposition~\ref{prop:normalized-mtt}]
We prove the connected case first. The disconnected statement follows by block diagonalization over connected components.

By the weighted Matrix--Tree theorem, every principal cofactor of the combinatorial Laplacian $L=D-W$ equals
\[
\tau_{\mathrm{tree}}(G)=\sum_{T\in\mathsf{Tree}(G)}\prod_{e\in T}w_e .
\]
Since
\[
\mathcal{L}=D^{-1/2}LD^{-1/2},
\]
the corresponding normalized cofactor satisfies, for every vertex $i$,
\[
\det\mathcal{L}_{\widehat{i},\widehat{i}}
=
\det\!\left(D_{\widehat{i}}^{-1/2}L_{\widehat{i},\widehat{i}}D_{\widehat{i}}^{-1/2}\right)
=
\frac{\tau_{\mathrm{tree}}(G)}{\prod_{j\neq i}d_j}.
\]
Because $G$ is connected, $\mathcal{L}$ is symmetric positive semidefinite of rank $|V|-1$ and has nullspace spanned by $D^{1/2}\mathbf{1}$. For any symmetric rank-$(n-1)$ matrix $A$ with null vector $x$, the adjugate satisfies
\[
\operatorname{adj}(A)
=
\frac{\det{}^{*}(A)}{\|x\|_2^2}xx^\top .
\]
Applying this with $A=\mathcal{L}$ and $x=D^{1/2}\mathbf{1}$ gives
\[
\det\mathcal{L}_{\widehat{i},\widehat{i}}
=
\frac{\det{}^{*}(\mathcal{L})}{\|D^{1/2}\mathbf{1}\|_2^2}d_i
=
\frac{\det{}^{*}(\mathcal{L})}{\mathrm{vol}(G)}d_i .
\]
Equating the two expressions for the same cofactor,
\[
\frac{\det{}^{*}(\mathcal{L})}{\mathrm{vol}(G)}d_i
=
\frac{\tau_{\mathrm{tree}}(G)}{\prod_{j\neq i}d_j}
=
\frac{d_i\,\tau_{\mathrm{tree}}(G)}{\prod_{j}d_j}.
\]
Since $d_i>0$, cancellation yields
\[
\det{}^{*}(\mathcal{L})
=
\frac{\mathrm{vol}(G)\,\tau_{\mathrm{tree}}(G)}{\prod_{j}d_j}.
\]
Taking logarithms gives the stated logarithmic identity. If $G$ has several connected components, $\mathcal{L}$ is block diagonal after permuting vertices, and its pseudo-determinant is the product of the pseudo-determinants of the component blocks. Applying the connected result to each block proves the componentwise formula.
\end{proof}

\begin{proof}[Proof of Proposition~\ref{prop:hodge-cycles}]
For any edge signal $x\in C_1(K;\mathbb{R})$,
\[
x^\top L_1 x
=
x^\top B_1^\top B_1x+x^\top B_2B_2^\top x
=
\|B_1x\|_2^2+\|B_2^\top x\|_2^2 .
\]
Hence $x\in\ker L_1$ if and only if $B_1x=0$ and $B_2^\top x=0$, so
\[
\ker L_1=\ker B_1\cap\ker B_2^\top .
\]
The finite-dimensional Hodge decomposition gives the orthogonal direct sum
\[
C_1(K;\mathbb{R})
=
\operatorname{im}B_1^\top
\oplus
\ker L_1
\oplus
\operatorname{im}B_2 .
\]
Because $B_1B_2=0$, boundaries $\operatorname{im}B_2$ are contained in cycles $\ker B_1$. Every cycle $z\in\ker B_1$ decomposes uniquely as
\[
z=z_\partial+h,\qquad z_\partial\in\operatorname{im}B_2,\quad h\in\ker L_1 .
\]
Therefore the map $h\mapsto[h]$ is an isomorphism
\[
\ker L_1\cong \ker B_1/\operatorname{im}B_2 = H_1(K;\mathbb{R}),
\]
and $\dim\ker L_1=\beta_1(K)$.

For the dimension formula, rank-nullity gives
\[
\beta_1(K)
=
\dim\ker B_1-\operatorname{rank}(B_2)
=
|E|-\operatorname{rank}(B_1)-\operatorname{rank}(B_2).
\]
For a graph complex with $|V|$ vertices and $\beta_0(K)$ connected components,
\[
\operatorname{rank}(B_1)=|V|-\beta_0(K).
\]
Substitution yields
\[
\beta_1(K)
=
|E|-|V|+\beta_0(K)-\operatorname{rank}(B_2),
\]
as claimed.
\end{proof}

\begin{proof}[Proof of Theorem~\ref{thm:view-stability}]
(a) Let $j\in N_k(i)$ and let $\ell\notin N_k(i)$ with $\ell\neq i$. By definition of the directed $k$NN margin,
\[
\delta_{i\ell}-\delta_{ij}\ge \Gamma_k(Z).
\]
Assumption~\ref{ass:distortion} gives
\[
\delta'_{ij}\le \delta_{ij}+\eta,\qquad
\delta'_{i\ell}\ge \delta_{i\ell}-\eta .
\]
Therefore
\[
\delta'_{i\ell}-\delta'_{ij}
\ge
\Gamma_k(Z)-2\eta
>
0 .
\]
Every selected neighbor remains closer than every non-selected non-self point. Hence every directed $k$NN list is unchanged, and so the mutual $k$NN edge set is unchanged.

(b) Since the neighbor lists are unchanged,
\[
|\sigma'_i-\sigma_i|\le \eta .
\]
Moreover, $\eta\le \sigma_{\min}/2$ implies $\sigma_i,\sigma_i'\ge \sigma_{\min}/2$. Define
\[
F(x,a,b)=\exp\!\left(-\frac{x^2}{ab}\right).
\]
On the domain $0\le x\le D_\eta$ and $a,b\ge\sigma_{\min}/2$, the partial derivatives are bounded by
\[
\left|\frac{\partial F}{\partial x}\right|\le \frac{8D_\eta}{\sigma_{\min}^2},
\qquad
\left|\frac{\partial F}{\partial a}\right|\le \frac{8D_\eta^2}{\sigma_{\min}^3},
\qquad
\left|\frac{\partial F}{\partial b}\right|\le \frac{8D_\eta^2}{\sigma_{\min}^3}.
\]
The mean-value theorem therefore gives
\[
|w'_{ij}-w_{ij}|\le C_w\eta
\]
for every mutual $k$NN edge.

(c) Let $\Delta W=W'-W$. Since each vertex has at most $k$ mutual $k$NN edges,
\[
\|\Delta W\|_2\le \|\Delta W\|_\infty \le kC_w\eta,
\qquad
\|D'-D\|_2=\max_i |d'_i-d_i|\le kC_w\eta .
\]
The assumption $kC_w\eta\le d_{\min}/2$ implies $d'_i\ge d_{\min}/2$. Put $A=D^{-1/2}$ and $A'=(D')^{-1/2}$. Since $x\mapsto x^{-1/2}$ has derivative bounded by $\sqrt{2}/d_{\min}^{3/2}$ on $[d_{\min}/2,\infty)$,
\[
\|A'-A\|_2
\le
\frac{\sqrt{2}}{d_{\min}^{3/2}}\,kC_w\eta .
\]
Also $\|A\|_2\le d_{\min}^{-1/2}$, $\|A'\|_2\le (2/d_{\min})^{1/2}$, and $\|W\|_2\le d_{\max}$. Since the normalized adjacency is $S=AWA$ and $\mathcal{L}=I-S$,
\begin{align*}
\|\mathcal{L}(Z')-\mathcal{L}(Z)\|_2
&=
\|A'W'A'-AWA\|_2\\
&\le
\|A'\Delta W A'\|_2
+
\|(A'-A)WA'\|_2
+
\|AW(A'-A)\|_2\\
&\le
kC_w\eta
\left(
\frac{2}{d_{\min}}
+
\frac{(2+\sqrt{2})d_{\max}}{d_{\min}^2}
\right).
\end{align*}
This is the stated bound with constant $C_{\mathcal{L}}$.

(d) The edge set and connected components are unchanged, so the kernel dimension of $\mathcal{L}(Z)$ and $\mathcal{L}(Z')$ is the same. Weyl's inequality gives
\[
|\lambda_i(\mathcal{L}(Z'))-\lambda_i(\mathcal{L}(Z))|
\le
\|\mathcal{L}(Z')-\mathcal{L}(Z)\|_2
\le
C_{\mathcal{L}}\eta .
\]
Every positive eigenvalue of $\mathcal{L}(Z')$ is therefore at least $a_0/2$. Since $\log x$ is $2/a_0$-Lipschitz on $[a_0/2,\infty)$,
\begin{align*}
\left|
\log\det{}^{*}\mathcal{L}(Z')-
\log\det{}^{*}\mathcal{L}(Z)
\right|
&\le
\sum_{r=1}^{p_0}
\left|
\log\lambda_r(\mathcal{L}(Z'))-
\log\lambda_r(\mathcal{L}(Z))
\right|\\
&\le
\frac{2p_0}{a_0}
\|\mathcal{L}(Z')-\mathcal{L}(Z)\|_2\\
&\le
\frac{2p_0C_{\mathcal{L}}}{a_0}\eta .
\end{align*}

(e) If the Hodge construction is the unweighted flag complex of the mutual $k$NN graph, then part (a) implies that the flag complex and all boundary matrices are identical. Hence $L_1$, $\beta_1$, and the unweighted Hodge spectrum are exactly unchanged. In the weighted case, the stated assumptions give the same kernel dimension and the perturbation bound
\[
\|L_1(Z')-L_1(Z)\|_2\le C_1\eta .
\]
Applying the same Weyl/log-Lipschitz argument as in part (d), with $a_1$ and $p_1$ in place of $a_0$ and $p_0$, gives
\[
\left|
\log\det{}^{*}L_1(Z')-
\log\det{}^{*}L_1(Z)
\right|
\le
\frac{2p_1C_1}{a_1}\eta .
\]

(f) Assumption~\ref{ass:distortion} says that the pairwise distance matrices differ by at most $\eta$ in sup norm. Thus the corresponding Vietoris--Rips filtrations are $\eta$-interleaved when the filtration parameter is interpoint distance. The stability theorem for persistence diagrams gives
\[
d_B\!\left(\operatorname{Dgm}_q(Z),\operatorname{Dgm}_q(Z')\right)
\le
\eta,
\qquad q=0,1 .
\]
If the filtration uses a radius convention rather than an interpoint-distance threshold, the same argument applies after the deterministic rescaling of the parameter. Finally, if $V$ is $L_V$-Lipschitz with respect to bottleneck distance, then
\[
\|V(\operatorname{Dgm}_q(Z))-V(\operatorname{Dgm}_q(Z'))\|
\le
L_V\eta .
\]
\end{proof}

\begin{proof}[Proof of Corollary~\ref{cor:positive-views}]
If $Z'=ZQ$ with $Q$ orthogonal, then
\[
\|z'_i-z'_j\|_2=\|(z_i-z_j)Q\|_2=\|z_i-z_j\|_2
\]
for all $i,j$. Hence all pairwise distances are identical, so the $k$NN lists, mutual $k$NN graph, self-tuning weights, normalized Laplacian spectrum, and unweighted flag complex are identical. The torsion proxy, unweighted Hodge features, and Vietoris--Rips persistence diagrams are therefore exactly invariant. Ollivier--Ricci curvatures are also exactly invariant because they are computed from the same weighted graph metric and the same random-walk distributions.

The approximate statement is Theorem~\ref{thm:view-stability} applied to the realized transformed cloud whenever the transformation satisfies Assumption~\ref{ass:distortion} below the required margin and spectral-gap thresholds.
\end{proof}

\begin{remark}[Bootstrap views]
Bootstrap subsampling changes the vertex set and is therefore not covered by the fixed-correspondence stability result in Theorem~\ref{thm:view-stability}. Its use as a positive view should be understood as a sampling-stability heuristic, or formalized separately through a bounded-difference assumption for the class-level graph functional being estimated.
\end{remark}

\begin{proof}[Proof of Theorem~\ref{thm:ssl-convex}]
(a) For each positive pair, the map
\[
w\mapsto (w^\top a_{ip})^2
\]
is a convex quadratic with Hessian $2a_{ip}a_{ip}^\top\succeq 0$. For each negative pair,
\[
w\mapsto [m-w^\top b_{iq}]_+
\]
is the maximum of the two affine functions $0$ and $m-w^\top b_{iq}$, hence is convex. The regularizer $\mu\|w\|_2^2$ is $2\mu$-strongly convex. Since $\mathbb{R}_{\ge0}^5$ is closed and convex, and the regularizer makes the objective coercive, $\widehat{\mathcal{J}}$ has a unique minimizer over $\mathbb{R}_{\ge0}^5$.

(b) A subgradient KKT point satisfies
\[
0\in \partial\widehat{\mathcal{J}}(w)+N_{\mathbb{R}_{\ge0}^5}(w),
\]
where $N_{\mathbb{R}_{\ge0}^5}(w)$ is the normal cone to the feasible orthant. This is precisely the first-order optimality condition for minimizing a convex function over a closed convex set. Hence every such point is globally optimal.

(c) If $g'\ge g$ coordinatewise and $w\ge0$, then
\[
s_w(g')-s_w(g)=w^\top(g'-g)\ge0.
\]
Thus the score is coordinatewise monotone.

(d) Let $w=(m/\gamma_-)u$. This vector is feasible because $u\in\mathbb{R}_{\ge0}^5$. For every positive view,
\[
|w^\top a_{ip}|
=
\frac{m}{\gamma_-}
|u^\top(g_i-g_{ip}^{+})|
\le
\frac{m\varepsilon_+}{\gamma_-}.
\]
Averaging squares gives invariance loss at most
\[
\frac{m^2\varepsilon_+^2}{\gamma_-^2}.
\]
For every negative view,
\[
w^\top b_{iq}
=
\frac{m}{\gamma_-}
u^\top(g_{iq}^{-}-g_i)
\ge
m,
\]
so
\[
[m-w^\top b_{iq}]_+=0.
\]
The regularizer of this candidate is
\[
\mu\|w\|_2^2=\frac{\mu m^2}{\gamma_-^2}.
\]
Since $w^\star$ minimizes $\widehat{\mathcal{J}}$, its objective value is no larger than that of this feasible candidate:
\[
\widehat{\mathcal{J}}(w^\star)
\le
\frac{m^2\varepsilon_+^2}{\gamma_-^2}
+
\frac{\mu m^2}{\gamma_-^2}.
\]
\end{proof}

\begin{proof}[Proof of Theorem~\ref{thm:ssl-generalization}]
For positive pairs, $a=g-g^+$ satisfies $\|a\|_2\le2B$. The linear class
\[
\mathcal{F}_+
=
\{a\mapsto w^\top a:\ w\in\mathcal{W}_R\}
\]
has empirical Rademacher complexity at most
\[
\widehat{\mathfrak{R}}_{M_+}(\mathcal{F}_+)
\le
\frac{R}{M_+}
\left(\sum_{r=1}^{M_+}\|a_r\|_2^2\right)^{1/2}
\le
\frac{2BR}{\sqrt{M_+}} .
\]
For $w\in\mathcal{W}_R$, $|w^\top a|\le2BR$. The map $x\mapsto x^2$ is $4BR$-Lipschitz on $[-2BR,2BR]$ and vanishes at zero. By the contraction inequality,
\[
\widehat{\mathfrak{R}}_{M_+}
\left(
\{a\mapsto (w^\top a)^2:\ w\in\mathcal{W}_R\}
\right)
\le
\frac{8B^2R^2}{\sqrt{M_+}} .
\]
The positive loss is bounded by $4B^2R^2$. The standard Rademacher high-probability bound gives, with probability at least $1-\delta/2$, simultaneously for all $w\in\mathcal{W}_R$,
\[
\mathbb{E}_{\mathcal{P}_+}(w^\top a)^2
\le
\frac{1}{M_+}\sum_{r=1}^{M_+}(w^\top a_r)^2
+
\frac{16B^2R^2}{\sqrt{M_+}}
+
4B^2R^2\sqrt{\frac{\log(4/\delta)}{2M_+}} .
\]

For negative pairs, $b=g^- - g$ satisfies $\|b\|_2\le2B$. The linear class $b\mapsto w^\top b$ has Rademacher complexity at most $2BR/\sqrt{M_-}$. The hinge map $x\mapsto [m-x]_+$ is $1$-Lipschitz, and subtracting its value at zero does not change the Rademacher complexity. The negative loss is bounded by $m+2BR$. Another Rademacher bound gives, with probability at least $1-\delta/2$, simultaneously for all $w\in\mathcal{W}_R$,
\[
\mathbb{E}_{\mathcal{P}_-}[m-w^\top b]_+
\le
\frac{1}{M_-}\sum_{r=1}^{M_-}[m-w^\top b_r]_+
+
\frac{4BR}{\sqrt{M_-}}
+
(m+2BR)\sqrt{\frac{\log(4/\delta)}{2M_-}} .
\]
Taking a union bound over the positive and negative events and adding the deterministic regularizer $\mu\|w\|_2^2$ gives
\[
\mathcal{J}(w)
\le
\widehat{\mathcal{J}}(w)
+
\epsilon_+(M_+,B,R,\delta)
+
\lambda\epsilon_-(M_-,B,R,m,\delta)
\]
simultaneously for all $w\in\mathcal{W}_R$, with $\epsilon_+$ and $\epsilon_-$ as stated.
\end{proof}

\section{Implementation Details and Runtime}
\label{app:runtime}
\paragraph{Embedding preprocessing:}
All experiments are implemented in PyTorch and PyTorch Geometric. For each checkpoint, penultimate-layer representations are extracted from fixed network layers using forward hooks and $\ell_2$-normalised before graph construction. To reduce class-imbalance effects and stabilise spectral statistics, we use class-balanced sampling before graph generation.

\paragraph{Graph and topological feature computation:}
Class-conditional mutual $k$-nearest-neighbour graphs are constructed using FAISS-accelerated nearest-neighbour search with adaptive self-tuning kernel weights (Eq.~\ref{eq:self-tuning}). Ollivier--Ricci curvature is computed using entropically regularised Wasserstein transport, while higher-order topological descriptors are extracted through Hodge Laplacian analysis and Vietoris--Rips persistent homology.

\paragraph{Runtime and compute resources:}
Experiments were conducted on NVIDIA A100 GPUs with 40GB of memory. Depending on dataset size and checkpoint count, graph construction, eigenspectrum estimation, curvature computation, and topological feature extraction required approximately 5--15 minutes per checkpoint. The implementation uses checkpoint caching, crash-safe serialisation, and resumable metric computation to support long-running geometric pipelines. 

\section{Existing assets and licenses}
We use standard public datasets and model families only for evaluation. CIFAR-10, CIFAR-10.1, CIFAR-10-C, ImageNet-C, MNLI, HANS, OGBN-Arxiv, ConvNeXt-T checkpoints, and GCN implementations are credited through their original publications in the main paper. We follow the corresponding dataset/model terms of use and do not redistribute the raw datasets. Dataset download and preparation instructions are provided in the supplementary README.

\section{Broader impact}
This work may have a positive impact by improving checkpoint selection before deployment under distribution shift, reducing reliance on target labels or target samples when such data are unavailable. A potential negative impact is that users may overinterpret a source-only diagnostic score as a guarantee of target-domain robustness. To mitigate this risk, we explicitly frame \textsc{TopoGeoScore} as a diagnostic selector rather than a robustness training method, and we state that source-only geometry cannot capture target-specific shifts that leave no measurable trace in the source representation.

\newpage

\section*{NeurIPS Paper Checklist}

\begin{enumerate}

\item {\bf Claims}
    \item[] Question: Do the main claims made in the abstract and introduction accurately reflect the paper's contributions and scope?
    \item[] Answer: \answerYes{} 
    \item[] Justification: The abstract and introduction clearly state the source-only checkpoint-selection setting, the proposed global--local--topological scoring framework, and the source-conditioned self-supervised learning objective. The claims are scoped to robustness diagnosis and checkpoint ranking, rather than robustness training or guaranteed prediction of arbitrary OOD accuracy.
    \item[] Guidelines:
    \begin{itemize}
        \item The answer \answerNA{} means that the abstract and introduction do not include the claims made in the paper.
        \item The abstract and/or introduction should clearly state the claims made, including the contributions made in the paper and important assumptions and limitations. A \answerNo{} or \answerNA{} answer to this question will not be perceived well by the reviewers. 
        \item The claims made should match theoretical and experimental results, and reflect how much the results can be expected to generalize to other settings. 
        \item It is fine to include aspirational goals as motivation as long as it is clear that these goals are not attained by the paper. 
    \end{itemize}

\item {\bf Limitations}
    \item[] Question: Does the paper discuss the limitations of the work performed by the authors?
    \item[] Answer: \answerYes{} 
    \item[] Justification: Section~5 discusses the main limitations, including that \textsc{TopoGeoScore} is a diagnostic selector rather than a robustness training method, cannot capture target-specific shifts invisible from source representations, is evaluated on small checkpoint/model families in some supporting settings, and requires graph and spectral/topological computation.
    \item[] Guidelines:
    \begin{itemize}
        \item The answer \answerNA{} means that the paper has no limitation while the answer \answerNo{} means that the paper has limitations, but those are not discussed in the paper. 
        \item The authors are encouraged to create a separate ``Limitations'' section in their paper.
        \item The paper should point out any strong assumptions and how robust the results are to violations of these assumptions (e.g., independence assumptions, noiseless settings, model well-specification, asymptotic approximations only holding locally). The authors should reflect on how these assumptions might be violated in practice and what the implications would be.
        \item The authors should reflect on the scope of the claims made, e.g., if the approach was only tested on a few datasets or with a few runs. In general, empirical results often depend on implicit assumptions, which should be articulated.
        \item The authors should reflect on the factors that influence the performance of the approach. For example, a facial recognition algorithm may perform poorly when image resolution is low or images are taken in low lighting. Or a speech-to-text system might not be used reliably to provide closed captions for online lectures because it fails to handle technical jargon.
        \item The authors should discuss the computational efficiency of the proposed algorithms and how they scale with dataset size.
        \item If applicable, the authors should discuss possible limitations of their approach to address problems of privacy and fairness.
        \item While the authors might fear that complete honesty about limitations might be used by reviewers as grounds for rejection, a worse outcome might be that reviewers discover limitations that aren't acknowledged in the paper. The authors should use their best judgment and recognize that individual actions in favor of transparency play an important role in developing norms that preserve the integrity of the community. Reviewers will be specifically instructed to not penalize honesty concerning limitations.
    \end{itemize}

\item {\bf Theory assumptions and proofs}
    \item[] Question: For each theoretical result, does the paper provide the full set of assumptions and a complete (and correct) proof?
    \item[] Answer: \answerYes{} 
    \item[] Justification: Section~3.7 summarizes the assumptions and scope of the theoretical results, including fixed-vertex perturbations, kNN-margin and spectral-gap conditions, and the distinction between feature stability and guarantees about learning the score. Formal statements and proofs are provided in Appendix~B.
    \item[] Guidelines:
    \begin{itemize}
        \item The answer \answerNA{} means that the paper does not include theoretical results. 
        \item All the theorems, formulas, and proofs in the paper should be numbered and cross-referenced.
        \item All assumptions should be clearly stated or referenced in the statement of any theorems.
        \item The proofs can either appear in the main paper or the supplemental material, but if they appear in the supplemental material, the authors are encouraged to provide a short proof sketch to provide intuition. 
        \item Inversely, any informal proof provided in the core of the paper should be complemented by formal proofs provided in appendix or supplemental material.
        \item Theorems and Lemmas that the proof relies upon should be properly referenced. 
    \end{itemize}

    \item {\bf Experimental result reproducibility}
    \item[] Question: Does the paper fully disclose all the information needed to reproduce the main experimental results of the paper to the extent that it affects the main claims and/or conclusions of the paper (regardless of whether the code and data are provided or not)?
    \item[] Answer: \answerYes{} 
    \item[] Justification: Section~4.1 specifies the datasets, checkpoint families, source/target settings, embedding extraction, graph construction, feature sampling, key hyperparameters, optimization objective, and evaluation metrics. Appendix~A provides additional per-setting and per-metric results.
    \item[] Guidelines:
    \begin{itemize}
        \item The answer \answerNA{} means that the paper does not include experiments.
        \item If the paper includes experiments, a \answerNo{} answer to this question will not be perceived well by the reviewers: Making the paper reproducible is important, regardless of whether the code and data are provided or not.
        \item If the contribution is a dataset and\slash or model, the authors should describe the steps taken to make their results reproducible or verifiable. 
        \item Depending on the contribution, reproducibility can be accomplished in various ways. For example, if the contribution is a novel architecture, describing the architecture fully might suffice, or if the contribution is a specific model and empirical evaluation, it may be necessary to either make it possible for others to replicate the model with the same dataset, or provide access to the model. In general. releasing code and data is often one good way to accomplish this, but reproducibility can also be provided via detailed instructions for how to replicate the results, access to a hosted model (e.g., in the case of a large language model), releasing of a model checkpoint, or other means that are appropriate to the research performed.
        \item While NeurIPS does not require releasing code, the conference does require all submissions to provide some reasonable avenue for reproducibility, which may depend on the nature of the contribution. For example
        \begin{enumerate}
            \item If the contribution is primarily a new algorithm, the paper should make it clear how to reproduce that algorithm.
            \item If the contribution is primarily a new model architecture, the paper should describe the architecture clearly and fully.
            \item If the contribution is a new model (e.g., a large language model), then there should either be a way to access this model for reproducing the results or a way to reproduce the model (e.g., with an open-source dataset or instructions for how to construct the dataset).
            \item We recognize that reproducibility may be tricky in some cases, in which case authors are welcome to describe the particular way they provide for reproducibility. In the case of closed-source models, it may be that access to the model is limited in some way (e.g., to registered users), but it should be possible for other researchers to have some path to reproducing or verifying the results.
        \end{enumerate}
    \end{itemize}

\item {\bf Open access to data and code}
    \item[] Question: Does the paper provide open access to the data and code, with sufficient instructions to faithfully reproduce the main experimental results, as described in supplemental material?
    \item[] Answer: \answerYes{}.
\item[] Justification: All datasets used in the experiments are public benchmarks, including CIFAR-10, CIFAR-10.1, CIFAR-10-C, ImageNet-C, MNLI, HANS, and OGBN-Arxiv. The supplementary material includes an anonymised Python implementation and README for the OGBN-Arxiv graph-data setting as a compact reproducibility example of the proposed graph/scoring pipeline; a complete public repository with full reproduction scripts will be released after de-anonymisation and acceptance.
    \item[] Guidelines:
    \begin{itemize}
        \item The answer \answerNA{} means that paper does not include experiments requiring code.
        \item Please see the NeurIPS code and data submission guidelines (\url{https://neurips.cc/public/guides/CodeSubmissionPolicy}) for more details.
        \item While we encourage the release of code and data, we understand that this might not be possible, so \answerNo{} is an acceptable answer. Papers cannot be rejected simply for not including code, unless this is central to the contribution (e.g., for a new open-source benchmark).
        \item The instructions should contain the exact command and environment needed to run to reproduce the results. See the NeurIPS code and data submission guidelines (\url{https://neurips.cc/public/guides/CodeSubmissionPolicy}) for more details.
        \item The authors should provide instructions on data access and preparation, including how to access the raw data, preprocessed data, intermediate data, and generated data, etc.
        \item The authors should provide scripts to reproduce all experimental results for the new proposed method and baselines. If only a subset of experiments are reproducible, they should state which ones are omitted from the script and why.
        \item At submission time, to preserve anonymity, the authors should release anonymized versions (if applicable).
        \item Providing as much information as possible in supplemental material (appended to the paper) is recommended, but including URLs to data and code is permitted.
    \end{itemize}

\item {\bf Experimental setting/details}
    \item[] Question: Does the paper specify all the training and test details (e.g., data splits, hyperparameters, how they were chosen, type of optimizer) necessary to understand the results?
    \item[] Answer: \answerYes{} 
    \item[] Justification: Section~4.1 describes the experimental settings, checkpoint counts, data sources, target shifts, graph construction, feature sampling, $k$ value, and self-supervised training hyperparameters. Additional results, severity sweeps, and per-setting analyses are provided in Appendix~A.
    \item[] Guidelines:
    \begin{itemize}
        \item The answer \answerNA{} means that the paper does not include experiments.
        \item The experimental setting should be presented in the core of the paper to a level of detail that is necessary to appreciate the results and make sense of them.
        \item The full details can be provided either with the code, in appendix, or as supplemental material.
    \end{itemize}

\item {\bf Experiment statistical significance}
    \item[] Question: Does the paper report error bars suitably and correctly defined or other appropriate information about the statistical significance of the experiments?
    \item[] Answer: \answerYes{} 
    \item[] Justification: The main and appendix tables report Spearman correlations, $p$-values, sample sizes, selected accuracy, oracle accuracy, and gap to oracle. The paper also states that statistical significance should be interpreted in the context of the finite checkpoint/model family sizes.
    \item[] Guidelines:
    \begin{itemize}
        \item The answer \answerNA{} means that the paper does not include experiments.
        \item The authors should answer \answerYes{} if the results are accompanied by error bars, confidence intervals, or statistical significance tests, at least for the experiments that support the main claims of the paper.
        \item The factors of variability that the error bars are capturing should be clearly stated (for example, train/test split, initialization, random drawing of some parameter, or overall run with given experimental conditions).
        \item The method for calculating the error bars should be explained (closed form formula, call to a library function, bootstrap, etc.)
        \item The assumptions made should be given (e.g., Normally distributed errors).
        \item It should be clear whether the error bar is the standard deviation or the standard error of the mean.
        \item It is OK to report 1-sigma error bars, but one should state it. The authors should preferably report a 2-sigma error bar than state that they have a 96\% CI, if the hypothesis of Normality of errors is not verified.
        \item For asymmetric distributions, the authors should be careful not to show in tables or figures symmetric error bars that would yield results that are out of range (e.g., negative error rates).
        \item If error bars are reported in tables or plots, the authors should explain in the text how they were calculated and reference the corresponding figures or tables in the text.
    \end{itemize}

\item {\bf Experiments compute resources}
    \item[] Question: For each experiment, does the paper provide sufficient information on the computer resources (type of compute workers, memory, time of execution) needed to reproduce the experiments?
    \item[] Answer: \answerYes{} 
    \item[] Justification: The experiments were conducted on NVIDIA A100 GPUs with 40GB memory. Appendix~C reports the main computational stages, including embedding extraction, mutual $k$-NN graph construction, eigenspectrum estimation, curvature computation, and topological feature extraction; the supplementary README also provides the software dependencies for the released OGBN-Arxiv implementation.
    \item[] Guidelines:
    \begin{itemize}
        \item The answer \answerNA{} means that the paper does not include experiments.
        \item The paper should indicate the type of compute workers CPU or GPU, internal cluster, or cloud provider, including relevant memory and storage.
        \item The paper should provide the amount of compute required for each of the individual experimental runs as well as estimate the total compute. 
        \item The paper should disclose whether the full research project required more compute than the experiments reported in the paper (e.g., preliminary or failed experiments that didn't make it into the paper). 
    \end{itemize}
    
\item {\bf Code of ethics}
    \item[] Question: Does the research conducted in the paper conform, in every respect, with the NeurIPS Code of Ethics \url{https://neurips.cc/public/EthicsGuidelines}?
    \item[] Answer: \answerYes{} 
    \item[] Justification: The work proposes a diagnostic method for source-only checkpoint selection using standard public benchmarks. It does not involve human-subject studies, private data collection, unsafe data release, or high-risk deployed systems, and the submission is anonymized.
    \item[] Guidelines:
    \begin{itemize}
        \item The answer \answerNA{} means that the authors have not reviewed the NeurIPS Code of Ethics.
        \item If the authors answer \answerNo, they should explain the special circumstances that require a deviation from the Code of Ethics.
        \item The authors should make sure to preserve anonymity (e.g., if there is a special consideration due to laws or regulations in their jurisdiction).
    \end{itemize}

\item {\bf Broader impacts}
    \item[] Question: Does the paper discuss both potential positive societal impacts and negative societal impacts of the work performed?
    \item[] Answer: \answerYes{} 
    \item[] Justification: The work may have positive impact by improving checkpoint selection before deployment under distribution shift, especially when target labels or target samples are unavailable. A potential negative impact is overreliance on a source-only diagnostic as a guarantee of target-domain robustness; the paper mitigates this by explicitly framing \textsc{TopoGeoScore} as a diagnostic selector and by stating that source-only geometry cannot capture target-specific shifts invisible from the source representation.
    \item[] Guidelines:
    \begin{itemize}
        \item The answer \answerNA{} means that there is no societal impact of the work performed.
        \item If the authors answer \answerNA{} or \answerNo, they should explain why their work has no societal impact or why the paper does not address societal impact.
        \item Examples of negative societal impacts include potential malicious or unintended uses (e.g., disinformation, generating fake profiles, surveillance), fairness considerations (e.g., deployment of technologies that could make decisions that unfairly impact specific groups), privacy considerations, and security considerations.
        \item The conference expects that many papers will be foundational research and not tied to particular applications, let alone deployments. However, if there is a direct path to any negative applications, the authors should point it out. For example, it is legitimate to point out that an improvement in the quality of generative models could be used to generate Deepfakes for disinformation. On the other hand, it is not needed to point out that a generic algorithm for optimizing neural networks could enable people to train models that generate Deepfakes faster.
        \item The authors should consider possible harms that could arise when the technology is being used as intended and functioning correctly, harms that could arise when the technology is being used as intended but gives incorrect results, and harms following from (intentional or unintentional) misuse of the technology.
        \item If there are negative societal impacts, the authors could also discuss possible mitigation strategies (e.g., gated release of models, providing defenses in addition to attacks, mechanisms for monitoring misuse, mechanisms to monitor how a system learns from feedback over time, improving the efficiency and accessibility of ML).
    \end{itemize}
    
\item {\bf Safeguards}
    \item[] Question: Does the paper describe safeguards that have been put in place for responsible release of data or models that have a high risk for misuse (e.g., pre-trained language models, image generators, or scraped datasets)?
    \item[] Answer: \answerNA{} 
    \item[] Justification: The paper does not release high-risk generative models, scraped datasets, or dual-use pretrained models. The released material is a reference implementation for a post-hoc scoring and checkpoint-selection method evaluated on standard public benchmarks.
    \item[] Guidelines:
    \begin{itemize}
        \item The answer \answerNA{} means that the paper poses no such risks.
        \item Released models that have a high risk for misuse or dual-use should be released with necessary safeguards to allow for controlled use of the model, for example by requiring that users adhere to usage guidelines or restrictions to access the model or implementing safety filters. 
        \item Datasets that have been scraped from the Internet could pose safety risks. The authors should describe how they avoided releasing unsafe images.
        \item We recognize that providing effective safeguards is challenging, and many papers do not require this, but we encourage authors to take this into account and make a best faith effort.
    \end{itemize}

\item {\bf Licenses for existing assets}
    \item[] Question: Are the creators or original owners of assets (e.g., code, data, models), used in the paper, properly credited and are the license and terms of use explicitly mentioned and properly respected?
    \item[] Answer: \answerYes{} 
    \item[] Justification: The paper cites the original sources for the datasets and model families used, including CIFAR-10.1, CIFAR-10-C, ImageNet-C, MNLI, HANS, ConvNeXt, and OGBN-Arxiv. The supplementary README lists the public assets used and states that raw datasets are not redistributed and should be obtained from their official sources under the corresponding terms of use.
    \item[] Guidelines:
    \begin{itemize}
        \item The answer \answerNA{} means that the paper does not use existing assets.
        \item The authors should cite the original paper that produced the code package or dataset.
        \item The authors should state which version of the asset is used and, if possible, include a URL.
        \item The name of the license (e.g., CC-BY 4.0) should be included for each asset.
        \item For scraped data from a particular source (e.g., website), the copyright and terms of service of that source should be provided.
        \item If assets are released, the license, copyright information, and terms of use in the package should be provided. For popular datasets, \url{paperswithcode.com/datasets} has curated licenses for some datasets. Their licensing guide can help determine the license of a dataset.
        \item For existing datasets that are re-packaged, both the original license and the license of the derived asset (if it has changed) should be provided.
        \item If this information is not available online, the authors are encouraged to reach out to the asset's creators.
    \end{itemize}

\item {\bf New assets}
    \item[] Question: Are new assets introduced in the paper well documented and is the documentation provided alongside the assets?
    \item[] Answer: \answerYes{} 
    \item[] Justification: The new released asset is the anonymized supplementary implementation for the graph/scoring pipeline, together with a README describing usage and reproduction steps. The paper does not introduce a new dataset, benchmark, or pretrained model.
    \item[] Guidelines:
    \begin{itemize}
        \item The answer \answerNA{} means that the paper does not release new assets.
        \item Researchers should communicate the details of the dataset\slash code\slash model as part of their submissions via structured templates. This includes details about training, license, limitations, etc. 
        \item The paper should discuss whether and how consent was obtained from people whose asset is used.
        \item At submission time, remember to anonymize your assets (if applicable). You can either create an anonymized URL or include an anonymized zip file.
    \end{itemize}

\item {\bf Crowdsourcing and research with human subjects}
    \item[] Question: For crowdsourcing experiments and research with human subjects, does the paper include the full text of instructions given to participants and screenshots, if applicable, as well as details about compensation (if any)? 
    \item[] Answer: \answerNA{} 
    \item[] Justification: The paper does not involve crowdsourcing experiments or research with human subjects.
    \item[] Guidelines:
    \begin{itemize}
        \item The answer \answerNA{} means that the paper does not involve crowdsourcing nor research with human subjects.
        \item Including this information in the supplemental material is fine, but if the main contribution of the paper involves human subjects, then as much detail as possible should be included in the main paper. 
        \item According to the NeurIPS Code of Ethics, workers involved in data collection, curation, or other labor should be paid at least the minimum wage in the country of the data collector. 
    \end{itemize}

\item {\bf Institutional review board (IRB) approvals or equivalent for research with human subjects}
    \item[] Question: Does the paper describe potential risks incurred by study participants, whether such risks were disclosed to the subjects, and whether Institutional Review Board (IRB) approvals (or an equivalent approval/review based on the requirements of your country or institution) were obtained?
    \item[] Answer: \answerNA{} 
    \item[] Justification: The paper does not involve human-subject research, participant recruitment, or collection of new human-subject data.
    \item[] Guidelines:
    \begin{itemize}
        \item The answer \answerNA{} means that the paper does not involve crowdsourcing nor research with human subjects.
        \item Depending on the country in which research is conducted, IRB approval (or equivalent) may be required for any human subjects research. If you obtained IRB approval, you should clearly state this in the paper. 
        \item We recognize that the procedures for this may vary significantly between institutions and locations, and we expect authors to adhere to the NeurIPS Code of Ethics and the guidelines for their institution. 
        \item For initial submissions, do not include any information that would break anonymity (if applicable), such as the institution conducting the review.
    \end{itemize}

\item {\bf Declaration of LLM usage}
    \item[] Question: Does the paper describe the usage of LLMs if it is an important, original, or non-standard component of the core methods in this research? Note that if the LLM is used only for writing, editing, or formatting purposes and does \emph{not} impact the core methodology, scientific rigor, or originality of the research, declaration is not required.
    \item[] Answer: \answerNA{} 
    \item[] Justification: LLMs are not used as an important, original, or non-standard component of the proposed method, experiments, or scientific contribution. Any use for writing, editing, or formatting does not affect the core methodology, scientific claims, or results.
    \item[] Guidelines:
    \begin{itemize}
        \item The answer \answerNA{} means that the core method development in this research does not involve LLMs as any important, original, or non-standard components.
        \item Please refer to our LLM policy in the NeurIPS handbook for what should or should not be described.
    \end{itemize}

\end{enumerate}

\end{document}